\documentclass{article}

\PassOptionsToPackage{numbers, compress}{natbib}

\usepackage[preprint]{neurips_2023}

\usepackage[utf8]{inputenc} 
\usepackage[T1]{fontenc}    
\usepackage{hyperref}       
\usepackage{url}            
\usepackage{booktabs}       
\usepackage{amsfonts}       
\usepackage{nicefrac}       
\usepackage{microtype}      
\usepackage{xcolor}         
\usepackage{graphicx}
\usepackage{amsmath}
\usepackage{amssymb}
\usepackage{mathtools}
\usepackage{amsthm}
\usepackage{algorithm}
\usepackage{algpseudocode}
\usepackage{changepage}
\usepackage{sidecap}
\usepackage{bbm}
\usepackage{tikz}

\DeclareMathOperator*{\argmax}{arg\,max}
\DeclareMathOperator*{\argmin}{arg\,min}
\DeclareMathOperator{\sign}{sign}
\newcommand*\diff{\mathop{}\!\mathrm{d}}

\newlength\savewidth\newcommand\shline{\noalign{\global\savewidth\arrayrulewidth
  \global\arrayrulewidth 1pt}\hline\noalign{\global\arrayrulewidth\savewidth}}

\theoremstyle{plain}
\newtheorem{theorem}{Theorem}[section]

\theoremstyle{definition}

\theoremstyle{remark}

\definecolor{mydarkblue}{rgb}{0,0.08,0.45}
\hypersetup{ %
    colorlinks=true,
    linkcolor=mydarkblue,
    citecolor=mydarkblue,
    filecolor=mydarkblue,
    urlcolor=mydarkblue,
}

\title{IDQL: Implicit Q-Learning as an Actor-Critic Method with Diffusion Policies}

\author{
    Philippe Hansen-Estruch ~
    Ilya Kostrikov ~
    Michael Janner \\ ~ 
    \textbf{Jakub Grudzien Kuba} ~
    \textbf{Sergey Levine} \\
    UC Berkeley\\
    \texttt{\{hansenpmeche, kostrikov, janner, kuba\}@berkeley.edu} \\ \texttt{svlevine@eecs.berkeley.edu}
}
\begin{document}
\maketitle
\begin{abstract}
Effective offline RL methods require properly handling out-of-distribution actions. Implicit Q-learning (IQL) addresses this by training a Q-function using only dataset actions through a modified Bellman backup. However, it is unclear which policy actually attains the values represented by this implicitly trained Q-function. In this paper, we reinterpret IQL as an actor-critic method by generalizing the critic objective and connecting it to a behavior-regularized implicit actor. This generalization shows how the induced actor balances reward maximization and divergence from the behavior policy, with the specific loss choice determining the nature of this tradeoff. Notably, this actor can exhibit complex and multimodal characteristics, suggesting issues with the conditional Gaussian actor fit with advantage weighted regression (AWR) used in prior methods. Instead, we propose using samples from a diffusion parameterized behavior policy and weights computed from the critic to then importance sampled our intended policy. We introduce Implicit Diffusion Q-learning (IDQL), combining our general IQL critic with the policy extraction method. IDQL maintains the ease of implementation of IQL while outperforming prior offline RL methods and demonstrating robustness to hyperparameters. Code is available at \href{https://github.com/philippe-eecs/IDQL}{\texttt{github.com/philippe-eecs/IDQL}}.
\end{abstract}

\section{Introduction}

Offline RL holds promise in enabling policies to be learned from static datasets. Value function estimation provides a basic building block for many modern offline RL methods, but such methods need to handle the out-of-distribution actions that arise when evaluating the learned policy
. While a variety of methods based
on constraints and regularization have been proposed to address this, in this work we will study implicit Q-learning (IQL) \citep{kostrikov2021offline}, which avoids querying the value for unseen actions entirely. In IQL the Q-function is trained with Bellman backups using the state value function as the target value. The state value function is trained via expectile regression onto the Q-values implicitly performing a maximization over actions. While IQL provides an appealing alternative to other offline RL methods, it remains unclear what policy the learned value function corresponds to, and in turn, makes it difficult to understand the bottlenecks in IQL-style algorithms.

In this paper, we derive a variant of IQL that significantly improves over the performance of the original approach, is relatively insensitive to hyperparameters, and outperforms even more recent offline RL methods. Our key observation is based on a new perspective that reinterprets IQL as an actor-critic method. This is achieved through the generalization of the value optimization problem in IQL to use an arbitrary convex loss, which we then link to an implicit behavior-regularized actor. This generalization is significant, as it demonstrates how different choices of critic objectives give rise to distinct implicit actor distributions that diverge from the behavior policy to varying extents. Consequently, examining the characteristics of this implicit actor reveals not only the trade-offs made by IQL in critic learning, but also potential avenues for improving the method. 

\begin{figure*}
    \centering
    \includegraphics[height=3.3cm]{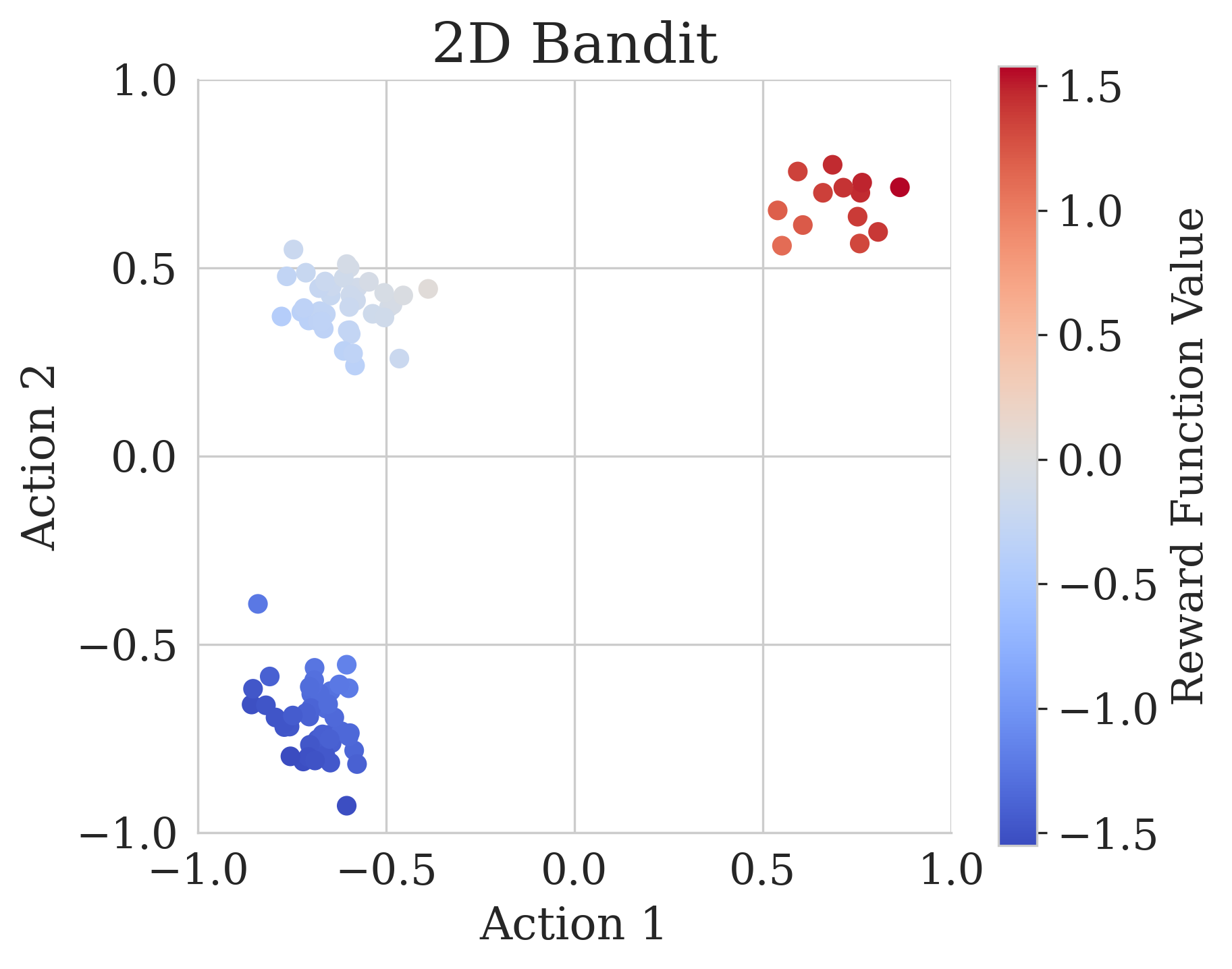} \includegraphics[height=3.3cm]{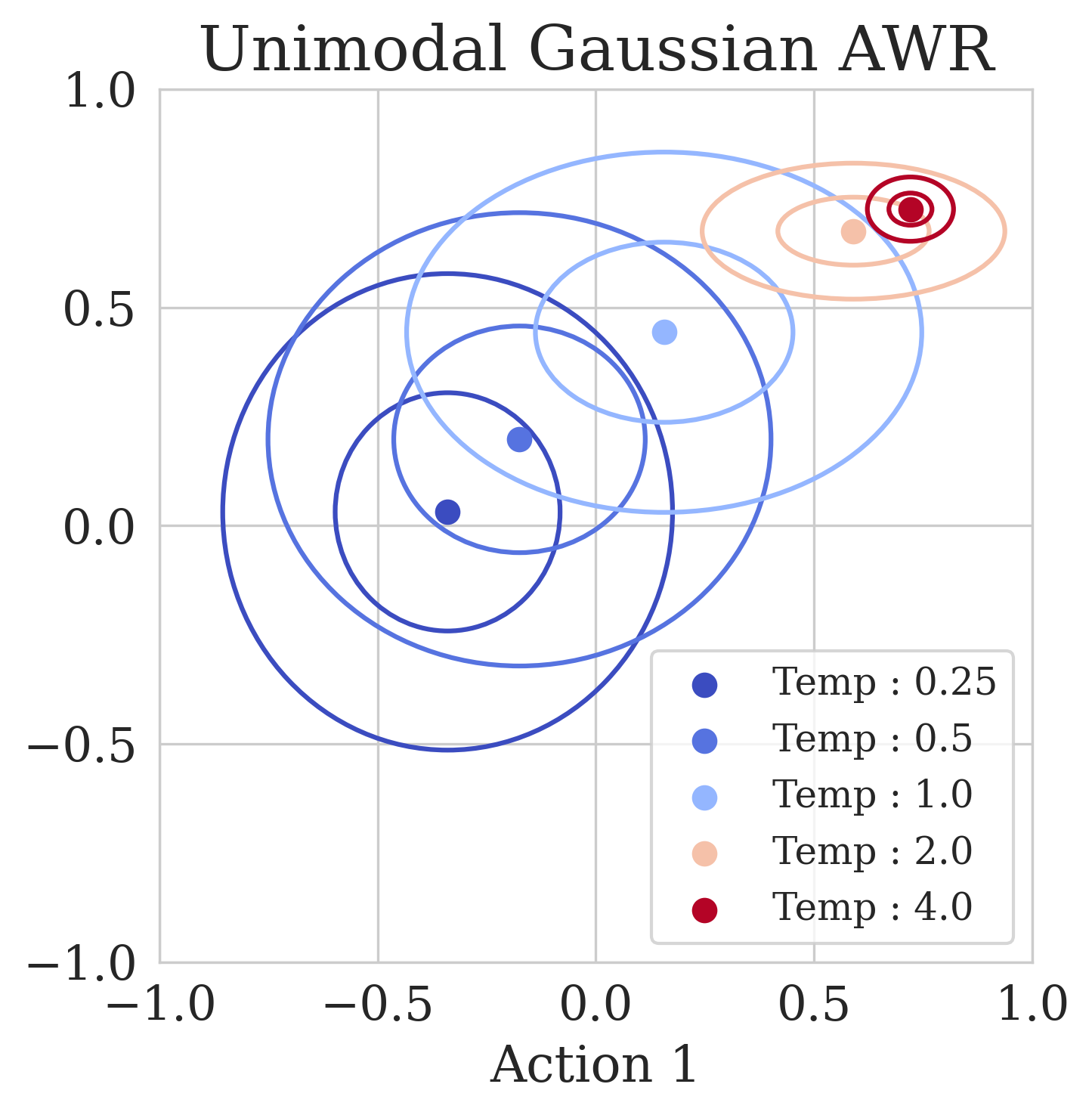}
    \includegraphics[height=3.3cm]{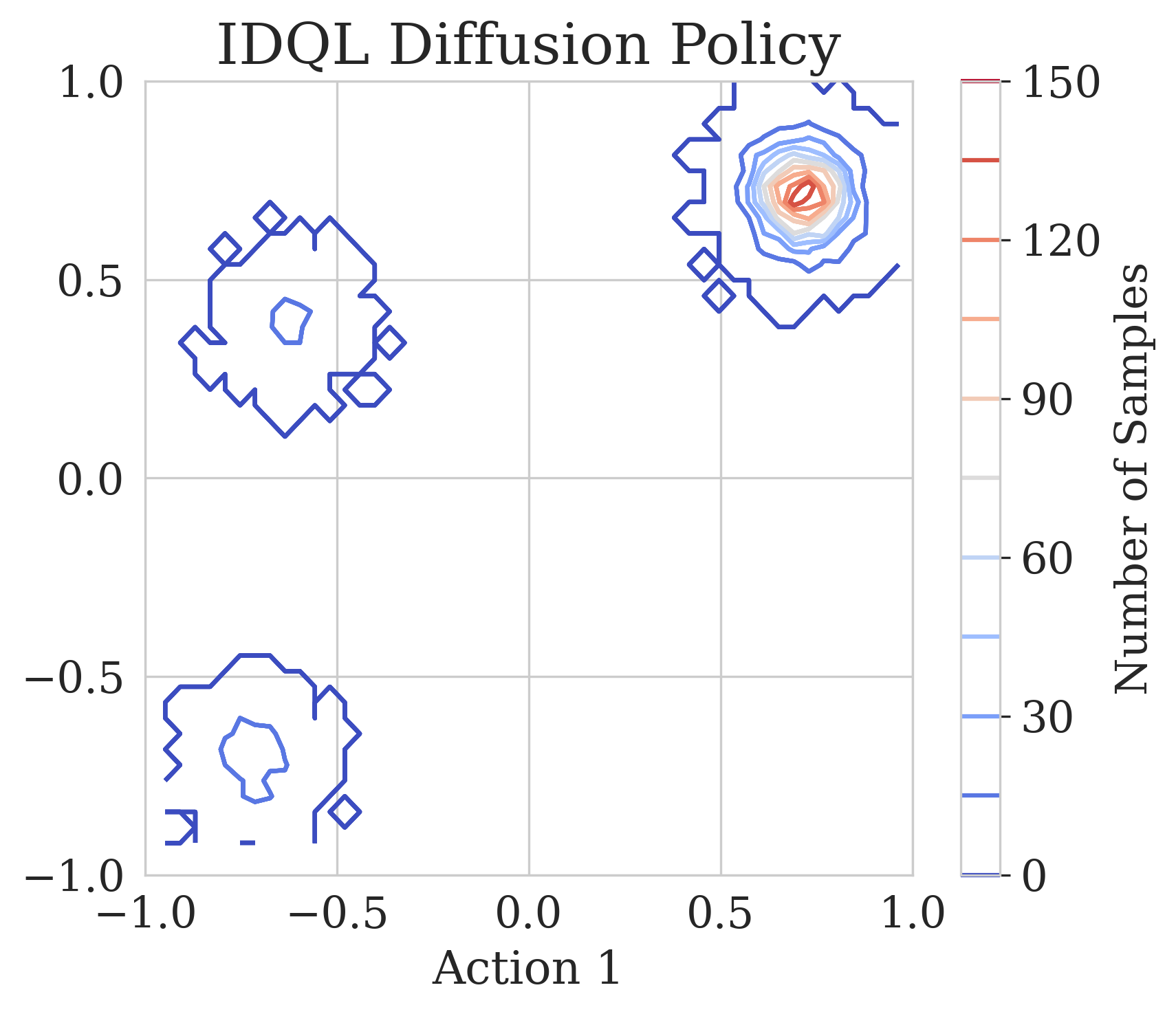}
    \caption{\textbf{Comparing unimodal AWR as used by IQL to our IDQL diffusion policy.} (Left) 2D bandit with three actions and reward function $r(a_1, a_2) = a_1 + a_2$. For the 90th expectile of the reward distribution, we compute contour plots for the corresponding unimodal Gaussian policy fitted with AWR (as used by IQL) for various inverse temperatures (center) as well as samples from our proposed extraction method (right). The unimodal actor incorrectly approximates the implicit policy for all temperatures, while the diffusion extraction method captures the implicit policy accurately. }
    \label{figure:teaser}
\end{figure*}

Our analysis reveals that the implicit actor is complex and potentially multimodal which suggests an issue with the policy extraction scheme used with IQL~\citep{kostrikov2021offline}, where a unimodal Gaussian actor is trained with advantage weighted regression (AWR)~\citep{peters2007, peng2019advantage, nair2020awac}. Such a parameterization is unlikely to accurately approximate the implicit actor (Figure~\ref{figure:teaser}). In contrast to conventional actor-critic algorithms, where the critic adapts to the \textit{explicit} actor parameterization, IQL's strength lies in its fully decoupled critic training from the explicit actor, thereby avoiding instability and hyperparameter sensitivity. However, this requires that the policy extraction approximates the implicit actor accurately. 
One helpful observation to address this challenge is that the implicit actor induced through critic learning can be expressed as a reweighting of the behavior policy, which can be approximated via importance weights. Thus, for policy extraction, we propose reweighting samples from a diffusion-parameterized behavior model with critic-computed weights, which can then be re-sampled to express both the implicit and greedy policy. Diffusion models \citep{sohl2015deep,ho2020denoising, song2020score} are a crucial detail for enabling our policy extraction method to represent multimodal distributions accurately.

We introduce Implicit Diffusion Q-learning (IDQL). Our main contributions are (1) the generalization of IQL to an actor-critic method where the choice of critic loss induces an implicit actor distribution that can be expressed as a reweighting of the behavior policy; (2) a diffusion based policy extraction algorithm that combines samples from an expressive diffusion model with a reweighting scheme for recovering an accurate approximation to the implicit actor, or maximizing the critic value. IDQL outperforms prior methods on the D4RL offline RL benchmarks~\citep{fu2020d4rl}, including methods based on IQL and other methods that use diffusion models. Furthermore, IDQL is very portable, requiring limited tuning of hyperparameters across domains (e.g., antmaze, locomotion) to perform well for both offline RL and online finetuning with offline pretraining.
\vspace{-0.3cm}
\section{Related Work}
\label{relatedwork}
While offline RL has been approached with a wide range of algorithmic frameworks~\citep{lange2012batch,levine2020offline}, recent methods often use value-based algorithms derived from Q-learning and off-policy actor-critic methods. To adapt approaches to the offline setting, it is necessary to avoid overestimated values out-of-distribution actions~\citep{kumar2019stabilizing}. To this end, recent methods use implicit divergence constraints \citep{peters2007, peng2019advantage, nair2020awac, wang2020critic}, explicit density models \citep{wu2019behavior, fujimoto2019off, kumar2019stabilizing, ghasemipour2021emaq}, and supervised learning terms \citep{fujimoto2021minimalist}. Some works also directly penalize out-of-distribution action Q-values \citep{kostrikovfisher2021offline, kumar2020conservative}. Recently, \emph{implicit} TD backups have been developed that avoid the use of out-of-sample actions in Q-function by using asymmetric loss functions with SARSA-style on-policy backups. This approach was proposed in IQL~\citep{kostrikov2021offline}, which trains a Q-function with an expectile loss and then extracts a policy with AWR~\citep{peng2019advantage}. 
However, the extraction procedure itself imposes a KL-divergence constraint, which is not consistent with the policy the critic captures. Extreme Q-learning (EQL)~\citep{garg2023extreme} modifies the Q-function objective in IQL to estimate a soft value consistent with the entropy regularized AWR policy. Concurrently, \citet{xu2023offline} provides another generalization of IQL in behavior-regularized MDPs, arriving at two new implicit Q-learning objectives. Our generalized IQL derivation is related, but features a different generalization form. 
Contrary to these papers, which focus on Q-learning, our primary practical contributions are in the policy extraction step motivated by this generalization. 

Our method leverages expressive generative models to capture the policy. To that end, we review works that use expressive generative models for offline RL. EMaQ \citep{ghasemipour2021emaq} defines a policy using an auto-regressive behavioral cloning model,
using the argmax action from this policy for backups in critic learning. The Decision Transformer and Trajectory Transformer \citep{chen2021decision, janner2021offline} clone the entire trajectory space using transformers and apply reward conditioning or beam search, respectively, to bias sampled trajectories towards higher rewards. Diffusion models have also been used in behavioral cloning and offline RL.
\citet{florence2021implicit} and  \citet{pearce2023imitating} use
energy-based models and diffusion models, respectively, for behavioral cloning. \citet{janner2022planning} and \citet{ajay2022conditional} use diffusion to directly model and sample the trajectory space; samples are guided with gradient guidance or reward conditioning. \citet{reuss2023goal} uses diffusion policies for goal-conditioned imitation learning. 

Closest to our work are prior methods that represent the actor with a diffusion model in offline RL. Diffusion Q-learning (DQL) \citep{wang2022diffusion}
incorporates diffusion to parameterize the actor in a TD3+BC-style algorithm \citep{fujimoto2021minimalist}. Select from Behavior Candidates (SfBC)~\citep{chen2022offline} uses importance reweighting from a diffusion behavior model to define the policy and then trains a critic with value iteration using samples from the policy. Action-Restricted Q-learning (ARQ)~\citep{goo2022know} performs a critic update with a diffusion behavior model and then uses AWR to extract the policy. Our method differs from these methods as the diffusion behavior model is kept separate from the critic learning process, only interacting with the critic at evaluation time. This makes it comparatively efficient to tune hyperparameters
. Furthermore, to our knowledge, we are the first method to apply diffusion to online finetuning.
\vspace{-0.2cm}
\section{Preliminaries}
RL is formulated in the context of a Markov decision process (MDP), which is defined as a tuple $( \mathcal{S}, \mathcal{A}, p_0(s), p_M(s' |s, a), r(s,a), \gamma )$ with state space $\mathcal{S}$, action space $\mathcal{A}$, initial state distribution $p_0(s)$, transition dynamics $p_M(s' | s, a)$, reward function $r(s, a)$, and discount $\gamma$. The goal is to recover a policy $\pi(a|s)$ that maximizes the discounted sum of rewards or return in the MDP. In offline RL, the agent only has access to a fixed dataset of transitions $\mathcal{D} = \{s, a, r, s' \}$ collected from a behavior policy distribution $\mu(a|s)$.

\paragraph{Implicit Q-learning.} Instead of constraining the policy or regularizing the critic, \citet{kostrikov2021offline}
proposed to approximate the expectile $\tau$ over the distribution of actions. For a parameterized critic $Q_\theta(s,a)$, target critic $Q_{\hat \theta}(s,a)$, and value network $V_{\psi}(s)$ the value objective is 
\begin{align}
    \begin{split}
        \label{eqn:fit_v_expectiles}
        \mathcal{L}_V(\psi) &= \mathbb{E}_{(s,a)~\sim \mathcal{D}}[L_2^\tau(Q_{\hat{\theta}}(s,a) - V_\psi(s))] \\ \mbox{ where } L_2^\tau(u) &= |\tau-\mathbbm{1}(u < 0)|u^2.
    \end{split}
\end{align}
Expectile regression uses a simple asymmetric squared error and requires only sampling actions from the datasets without any explicit policy. 
Then, this value function is used to update the Q-function:
\begin{align}
    \label{eqn:fit_q}
    \mathcal{L}_Q(\theta) &= \mathbb{E}_{(s,a,s')~\sim \mathcal{D}}[(r(s,a) + \gamma V_\psi(s') - Q_{\theta}(s,a))^2].
\end{align}

The Q-function is induced under the implicit policy distribution defined by the expectile.
For policy extraction, IQL uses AWR \citep{peters2007, peng2019advantage, nair2020awac}, which trains the policy via weighted regression by minimizing
 \begin{align}
 \label{AWR}
  \mathcal{L}_\pi(\phi) &= \mathbb{E}_{(s,a)\sim \mathcal{D}}[\exp(\alpha(Q_{\hat \theta}(s,a) - V_{\psi}(s)))\log \pi_{\phi}(a|s)].
\end{align}
The temperature parameter, $\alpha \in [0, \infty]$, serves to balance critic exploitation with behavior cloning.
\paragraph{Diffusion models.} 
We briefly review diffusion for behavior cloning as it is the basis of our policy extraction algorithm. Diffusion models~\citep{sohl2015deep, ho2020denoising, song2020score} are latent variable models that use a Markovian noising and denoising process that can be used to model a parameterized behavior distribution $\mu_\phi(a_0|s) = \int \mu_{\phi}(a_{0:T}|s) \diff a_{1:T}$ for the latent variables $a_1, \cdots, a_T$.
The forward noising process follows a fixed variance schedule  $\beta_1, \cdots, \beta_T$ that follows the distribution 
$$q(a_t | a_{t-1}) = \mathcal{N}(\sqrt{1 - \beta_t} a_{t-1}, \beta_t I).$$
Following DDPMs \citep{ho2020denoising}, our practical implementation involves parameterizing the score network directly $\mu_{\phi}(a_{t-1}|a_t, s, t)$ to recover the behavior cloning objective
\begin{equation}
\label{BCDiffusionLoss}
\mathcal{L}_{\mu}(\phi) = \mathbb{E}_{t \sim \mathcal{U}(1, T), \epsilon \sim \mathcal{N}(0, I), s,a \sim \mathcal{D}} [||\epsilon -  \mu_{\phi}(\sqrt{\hat \alpha_t}a + \sqrt{ 1 - \hat \alpha_t}\epsilon , s, t)||].
\end{equation}
 To sample from $\mu_{\phi}(a_0|s)$, we use Langevin sampling or reverse diffusion where $a_T \sim \mathcal{N}(0, I)$ and $\epsilon \sim \mathcal{N}(0, I)$ gets resampled every step
 \begin{equation}
 \label{Langevin_Dynamics}
     a_{t-1} \leftarrow \frac{1}{\sqrt{\alpha_t}}\bigg(a_t - \frac{\beta_t}{\sqrt{1 - \hat \alpha_t}} \mu_{\phi}(a_t |s, t)\bigg) + \sqrt{\beta_t} \epsilon, \ \text{for} \ t = \{T, \cdots, 1\}
 \end{equation}
 
\section{Implicit Q-Learning as an Actor-Critic Method}
\label{acmethod}

In this section, we will present a generalization of IQL that will not only provide a more complete conceptual understanding of implicit Q-learning, but also help us to understand how this method could be improved. \citet{kostrikov2021offline} show that IQL recovers Q-learning in the limit as $\tau$ approaches 1 in Equation~\ref{eqn:fit_v_expectiles}, but this does not describe what policy is captured by the Q-function in practice, when $0.5 < \tau < 1$. We can better understand the real behavior of IQL by reinterpreting it as an actor-critic method, where critic learning induces an implicit behavioral regularized actor $\pi_{\text{imp}}(a|s)$. This generalization will help us to understand how IQL can be improved, and the tradeoff that is captured by IQL's hyperparameters and the form of its loss function.

\subsection{Generalized Implicit Q-Learning}

To rederive IQL as an actor-critic method, we first generalize the value loss in Equation~\ref{eqn:fit_v_expectiles} to use an arbitrary convex loss $f$ on the difference $Q(s,a) - V(s)$. For a given $Q(s,a)$, the general IQL critic update can be defined as
\begin{equation}
\label{eqn:general_IQL}
    V^*(s) = \argmin_{V(s)} \mathbb{E}_{a \sim \mu(a|s)}[f(Q(s,a) - V(s))] = \argmin_{V(s)} \mathcal{L}_{V}^f(V(s)).
\end{equation}
We recover IQL when $f(u) = L^\tau_2(u)$, as in Equation~\ref{eqn:fit_v_expectiles}, but we will also consider other asymmetric convex losses. To define the implicit actor, we follow the conventional definition of value functions in actor-critic methods, where
\begin{equation}
\label{eqn:implicit_actor_definition}
    V(s) = \mathbb{E}_{a \sim \pi_{\text{imp}}(a|s)}[Q(s,a)].
\end{equation}
We use $f' = \frac{\partial f}{\partial V(s)}$ as shorthand for the derivative of $f$ with respect to $V(s)$.
\begin{theorem}
\label{thm:Implicit_Policy}For every state $s$ and convex loss function $f$ where $f'(0) = 0$, 
the solution to the optimization problem defined in Equation~\ref{eqn:general_IQL} is also a solution to the optimization problem
$$
\argmin_{V(s)} \mathbb{E}_{a \sim \pi_{\text{imp}}(a|s)}[(Q(s,a) - V(s))^2],
$$
where $\pi_{\text{imp}}(a|s) \propto \frac{\mu(a|s) |f'(Q(s,a) - V^*(s))|}{|Q(s,a) - V^*(s)|}$.
\end{theorem}
\begin{proof} See Appendix~\ref{proofs}. \end{proof}

Theorem \ref{thm:Implicit_Policy} provides us with a relationship between the (generalized) IQL loss function $f$ and the corresponding implicit actor $\pi_{\text{imp}}(a|s)$, thus indicating that IQL is an actor-critic method. To make it clear how the implicit actor relates to the behavior policy $\mu(a|s)$, we can define an importance weight
\begin{equation}
\label{eqn:importance_weight}
    w(s,a) = \frac{|f'(Q(s,a) - V^*(s))|}{|Q(s,a) - V^*(s)|},
\end{equation}
which yields an expression for the implicit actor as $\pi_{\text{imp}}(a|s) \propto \mu(a|s) w(s,a)$. 
The form of $f$ affects how $\pi_{\text{imp}}(a|s)$ deviates from $\mu(a|s)$.
Since IQL can recover the value function $V^*(s)$ without constructing the policy explicitly, the implicit actor only needs to be extracted at the end of critic training. 
Though, the complexity of the weight in Equation~\ref{eqn:importance_weight} presents a challenge, as policy extraction using a less expressive policy class is likely to result in a poor approximation of the implicit actor. We will return to this point in the next section, but first, we will provide three examples of potential functions $f$ and derive their corresponding implicit policies.

\paragraph{Expectiles.}
The expectile of a distribution corresponds to the conditional mean if points above the expectile are sampled more frequently than in the standard distribution. 
For a given $\tau$, the expectile objective corresponds to $f(u) = L_2^\tau(u)$ (from Equation~\ref{eqn:fit_v_expectiles}) and solution $V_\tau^2(s)$. From Theorem~\ref{thm:Implicit_Policy},
\begin{equation}
\label{eq:Expectile_Weights}
    w_2^{\tau}(s, a) = |\tau - \mathbbm{1}(Q(s,a) < V_\tau^2(s))|.
\end{equation}
Increasing $\tau$ for expectiles directly increases the deviation from the behavior policy.

\paragraph{Quantiles.}
The quantile statistic measures the top $\tau\%$ of the behavior policy performance. This is analogous to the mean behavior performance (SARSA) and Q-learning trade-off that occurs with expectiles, but instead balances median behavior performance with Q-learning. For a given $\tau$, the quantile objective corresponds to $f(u) = |\tau-\mathbbm{1}(u < 0)||u|$ and solution $V_\tau^1(s)$.
From Theorem~\ref{thm:Implicit_Policy},
\begin{equation}
\label{eq:Quantile_Weights}
    w_1^{\tau}(s, a) = \frac{|\tau - \mathbbm{1}(Q(s,a) < V_1^\tau(s))|}{|Q(s,a) - V_\tau^1(s)|}.
\end{equation}
As with expectiles, increasing $\tau$ directly influences the level of extrapolation from the behavior policy. Points should cluster around the quantile for this implicit policy.

\paragraph{Exponential.}
Another interesting choice for $f$ is the linex function $f(u) = \alpha \exp(u) - \alpha u$, where the temperature $\alpha \in [0, \infty]$ serves to balance matching the behavior policy and optimizing the critic. The solution for this objective, derived in Appendix~\ref{additionalderivations}, is 
\begin{equation}
\label{eqn:V_exp}
V_{\exp}(s) = \frac{1}{\alpha} \log \sum_a \exp(\alpha Q(s,a) + \log \mu(a|s)).
\end{equation}
The derivation of $V_{\exp}(s)$ indicates that it is a normalizer for the AWR policy (Equation~\ref{AWR}), where $\pi_{\text{awr}}(a|s) \propto \exp(\alpha Q(s,a) + \log \mu(a|s))$. Effectively, this loss applies a KL-divergence constraint. This objective matches the critic objectives used by \citet{garg2023extreme} and \citet{xu2023offline}, and it is also similar to the soft policy definition used by \citet{haarnoja2018soft}. From Theorem~\ref{thm:Implicit_Policy},
\begin{equation}
\label{eq:Exponential_Weights}
    w_{\exp}(s, a) =  \frac{\alpha|\exp\big(\alpha(Q(s,a) - V_{\exp}(s))\big) - 1|}{|Q(s,a) - V_{\exp}(s)|}.
\end{equation}
This gives most of the weight to the actions with the highest Q-value. If IQL is used with the AWR policy extraction, then this choice of $f$ would provide the correct corresponding critic, though we show later that in practice this loss does not lead to the best performance.

\paragraph{Comparing different loss functions.}
\begin{figure}
    
    \centering
    \includegraphics[height=3.0cm]{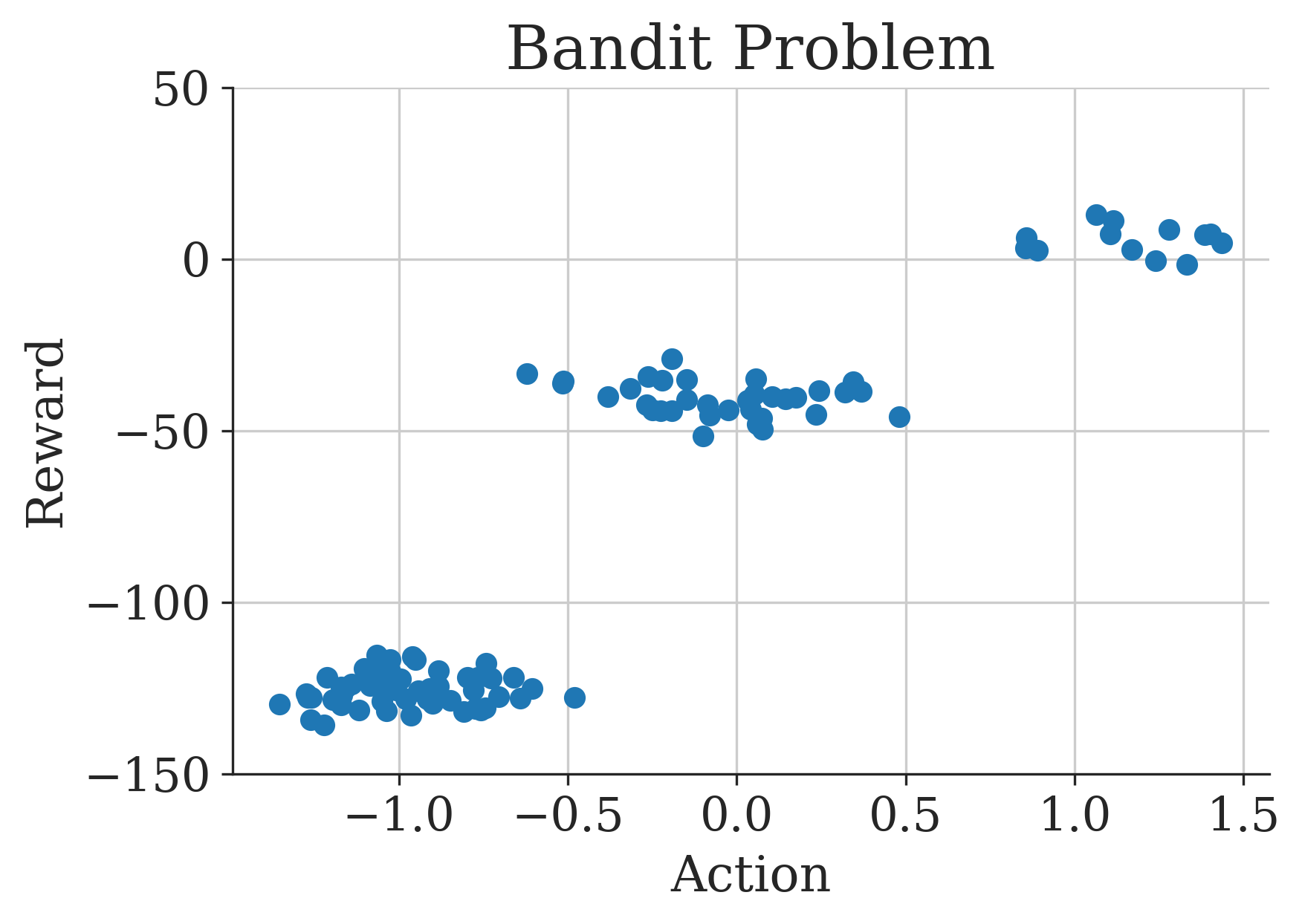} \includegraphics[height=3.0cm]{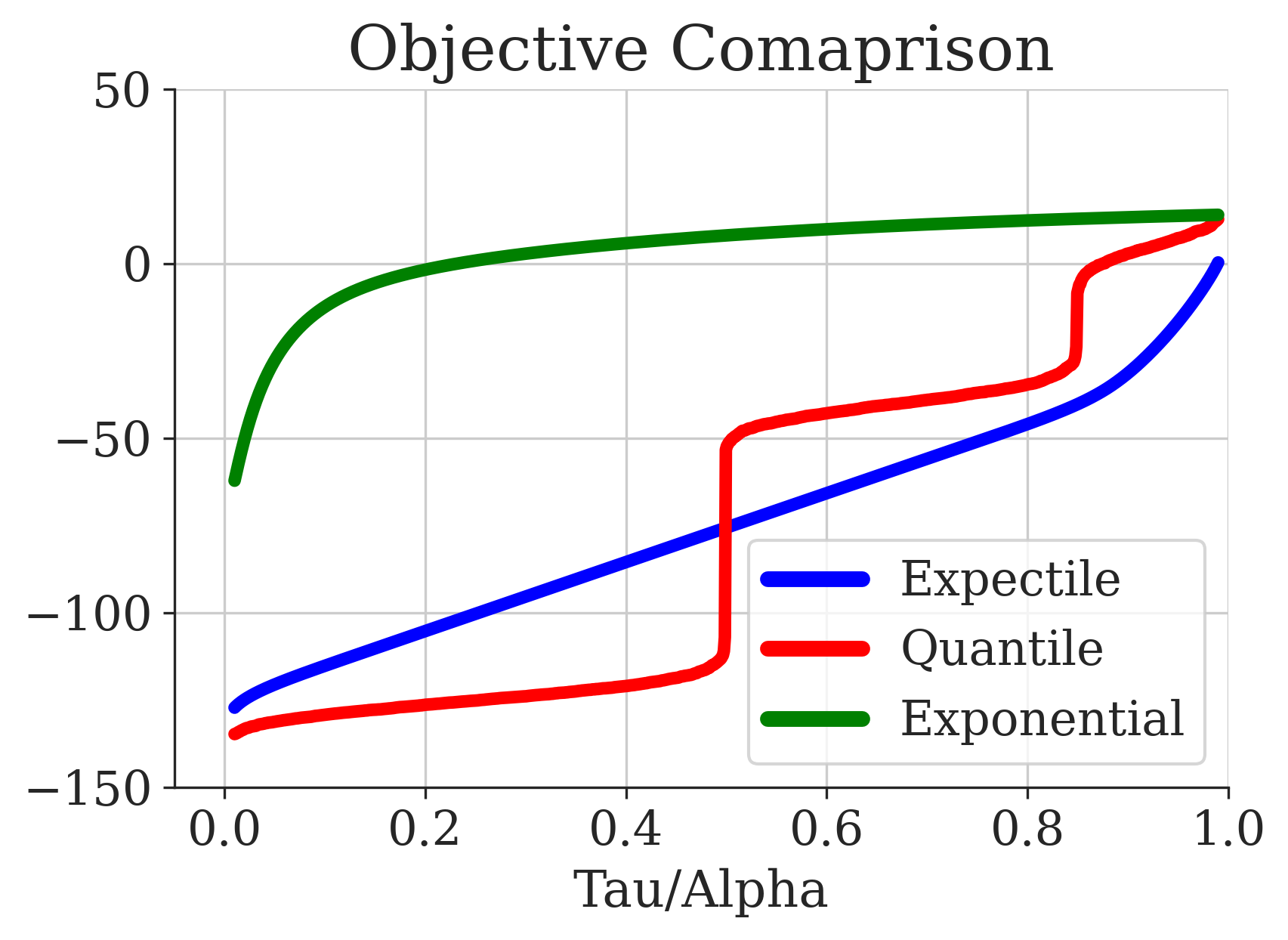} \includegraphics[height=3.0cm]{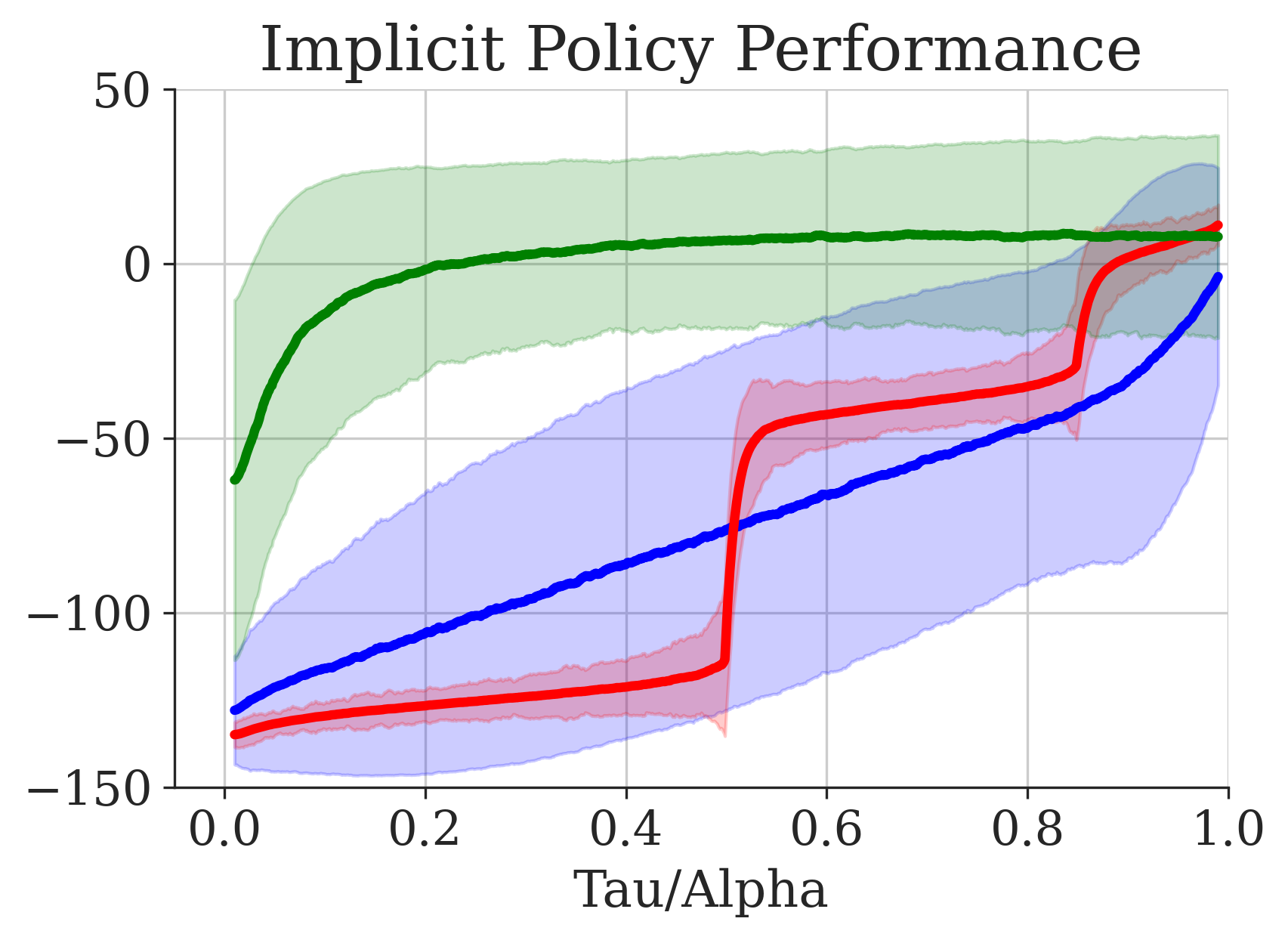}
    \caption{\textbf{Comparison of different loss functions for generalized IQL on a simple bandit.} The bandit's continuous action space has three clusters corresponding to low, medium, and high reward, with additive noise (left). We compute values from the different objectives in Section~\ref{acmethod} over choices of hyperparameter on the bandit distribution (center). Finally, we show the mean and standard deviation performance of samples from the implicit actor (from Theorem~\ref{thm:Implicit_Policy}) induced by each objective (right). Each objective captures different characteristics from the distribution of rewards.}
    \label{figure:banditproblem}
\end{figure}
Theorem~\ref{thm:Implicit_Policy} demonstrates that IQL can be generalized to an actor-critic method with the choice of loss function $f$ influencing the implicit policy. In each case, there is a hyperparameter that controls how much the implicit actor deviates from the behavior policy. In a simple bandit problem (Figure~\ref{figure:banditproblem}), the expectile increases smoothly with $\tau$, the quantile corresponds to the cumulative distribution function, and the exponential aligns with the maximum. The implicit policy distributions also differ for each objective: the expectile actor covers a broad range of outcomes, the quantile actor is tightly focused around the mean, and the exponential actor has some coverage around the maximum. Although it seems that the exponential critic learns the most optimal policy, in practice we find that it can be unstable, as we will show in Section~\ref{objectiveablation}.

\subsection{Policy Extraction and General Algorithm}

\label{policyextraction}
\begin{figure}[htb]
    \centering
    \begin{minipage}[t]{0.45\textwidth}
        \begin{algorithm}[H]
        \caption{General IQL Training}
        \label{alg:GeneralCriticLearning}
        \begin{algorithmic}
        \State \textbf{Hyperparameters:} LR $\lambda$, EMA $\eta$
\State \textbf{Initialize:} $\theta$, $\hat\theta$, $\psi$, and $\phi$
\\
\While {training not converged}
\State $\psi \leftarrow \psi - \lambda \nabla_{\psi} \mathcal{L}_V^f(\psi)$ (Equation~\ref{eqn:general_IQL})
\State $\theta \leftarrow \theta - \lambda \nabla_{\theta} \mathcal{L}_Q(\theta)$ (Equation~\ref{eqn:fit_q})
\State $\hat \theta \leftarrow (1 - \eta) \hat \theta + \eta \theta$
\State $\phi \leftarrow \phi - \lambda \nabla_{\phi} \mathcal{L}_{\mu}(\phi)$ (Equation~\ref{BCDiffusionLoss}) 
\EndWhile
\vspace{0.3855cm}
        \end{algorithmic}
        \end{algorithm}
    \end{minipage}
    \hfill
    \begin{minipage}[t]{0.54\textwidth}
        \begin{algorithm}[H]
        \caption{General IQL Policy Extraction}
        \label{alg:GeneralPolicyExtraction}
        \begin{algorithmic}
        \State \textbf{Hyperparameters:} Samples per state $N$
\State \textbf{Pretraining:} $Q_{\hat\theta}(s,a)$, $V_{\psi}(s)$, and $\mu_{\phi}(a|s)$
\\
\While {not done with episode}
\State Observe current state $s$ 
\State Sample $a_i \sim \mu_{\phi}(a|s)$, $i = 1, \ldots, N$
\State Compute $w(s,a_i)$ using Eqns.~\ref{eq:Expectile_Weights}, \ref{eq:Quantile_Weights}, \ref{eq:Exponential_Weights}, or \ref{eqn:argmax}
\State Normalize:  $p_i = \frac{w(s, a_i)}{\sum_j w(s, a_j)}$
\State Select $a_{\text{taken}}$ as a categorical from $p_i$
\EndWhile
        \end{algorithmic}
        \end{algorithm}
    \end{minipage}
\end{figure}
Theorem~\ref{thm:Implicit_Policy} shows that the implicit actors corresponding to critics trained with IQL can be complex and multimodal. However, standard IQL approximates these complex implicit actors with a unimodal conditional Gaussian policy~\citep{kostrikov2021offline}, and since the critic is unaware of this approximation, it does not adapt to it (in contrast to a standard actor-critic method where the critic adjusts to the limitations of the actor). While decoupling the actor from the critic should make the method less sensitive to hyperparameters, we hypothesize that in practice this benefit can only be realized if the final explicit actor is powerful enough to capture the complex implicit actor distribution.

One approach to train a more expressive actor would be to use a more powerful conditional generative model (e.g., diffusion model, normalizing flow) with an AWR-style importance weighted objective. However, it is known in the literature that using highly expressive models with importance weighted objectives can be problematic, as such models can increase the likelihood of all training points regardless of their weight~\citep{byrd2019effect, xu2021understanding}. We find using AWR in the DDPM objective to not help performance (Appendix~\ref{additionalablations}). In order to capture these policies without requiring extensive tuning, we train a highly expressive policy to represent the behavior policy and then reweight the samples from this model.

Denoting our learned behavior policy model as $\mu_{\phi}(a|s)$, we can generate samples from this model, and then use the critic to reweight these actions, ultimately forming the intended policy when resampled. This approach is summarized in Algorithm~\ref{alg:GeneralPolicyExtraction}, and provides samples from the correct implicit actor distribution. In practice, we also found that simply taking the action with the highest Q-value tends to yield better performance at evaluation time, which corresponds to selecting $w(s,a)$ to be a one-hot. This approach is analogous to how stochastic actor methods typically use a stochastic actor for critic learning and a deterministic actor at evaluation time, ~\citep{brandfonbrener2021offline, haarnoja2018soft}. This yields the deterministic policy
\begin{equation}
\label{eqn:argmax}
\pi(s) = \argmax_{{ a_i \sim \mu(a|s), \ i=1\ldots N}}Q(s,a_i).
\end{equation}

\paragraph{General algorithm summary.}
For choice of loss $f$ and generative model parameterization, our complete algorithm is summarized in the training phase (Algorithm~\ref{alg:GeneralCriticLearning}) and inference phase (Algorithm~\ref{alg:GeneralPolicyExtraction}). Theoretically, any parameterization of generative model could be used for behavior cloning in our method. However, we find that diffusion models yield the best results for our practical approach, as we discuss in the subsequent section.

\paragraph{Online finetuning procedure.}
\label{practicalfinetuning}

Since general IQL extracts policies only during evaluation, finetuning the behavioral distribution may not be necessary for situations where exploration is not a significant requirement. We outline two methods for fine-tuning using general IQL: (1) freezing the behavior policy $\mu_{\phi}(a|s)$, sampling the argmax action for exploration, and only finetuning $Q_{\theta}(s,a)$ and $V_{\psi}(s)$ and, (2) sampling $\pi_{\text{imp}}(a|s)$ for exploration, and finetuning $Q_{\theta}(s,a)$, $V_{\psi}(s)$, and $\mu_{\phi}(a|s)$.

\section{Implicit Diffusion Q-Learning}
\label{section:issuesdiff}
As argued in Section~\ref{policyextraction}, the policy extraction algorithm requires an expressive behavior distribution to model the implicit actor accurately. Diffusion models are a good fit here as they have been used to model complex distributions in images~\citep{ho2020denoising} as well as in continuous action spaces~\citep{wang2022diffusion, pearce2023imitating} prior. Therefore, for the practical implementation of our general IQL algorithm depicted in Algorithm~\ref{alg:GeneralCriticLearning} and Algorithm~\ref{alg:GeneralPolicyExtraction}, we use a diffusion model parameterization for $\mu_{\phi}(a|s)$ and the DDPM objective (Equation~\ref{BCDiffusionLoss}). We dub this method Implicit Diffusion Q-learning (IDQL).

\paragraph{Issues in continuous space expression with diffusion.}

A naïve DDPM implementation on continuous spaces can have issues with outputting outliers and expressing the distribution accurately. As an example, we test a simple implementation of DDPMs on 2D continuous datasets. In Figure~\ref{fig:2dsamples}, the simple MLP architecture fails to capture the data distributions and produces many out-of-distribution samples. The outliers are particularly problematic because they might be out-of-distribution for a trained Q-function and as a result, might receive erroneously high Q values.

    \begin{SCfigure}
        \centering
        \includegraphics[width=0.6\textwidth]{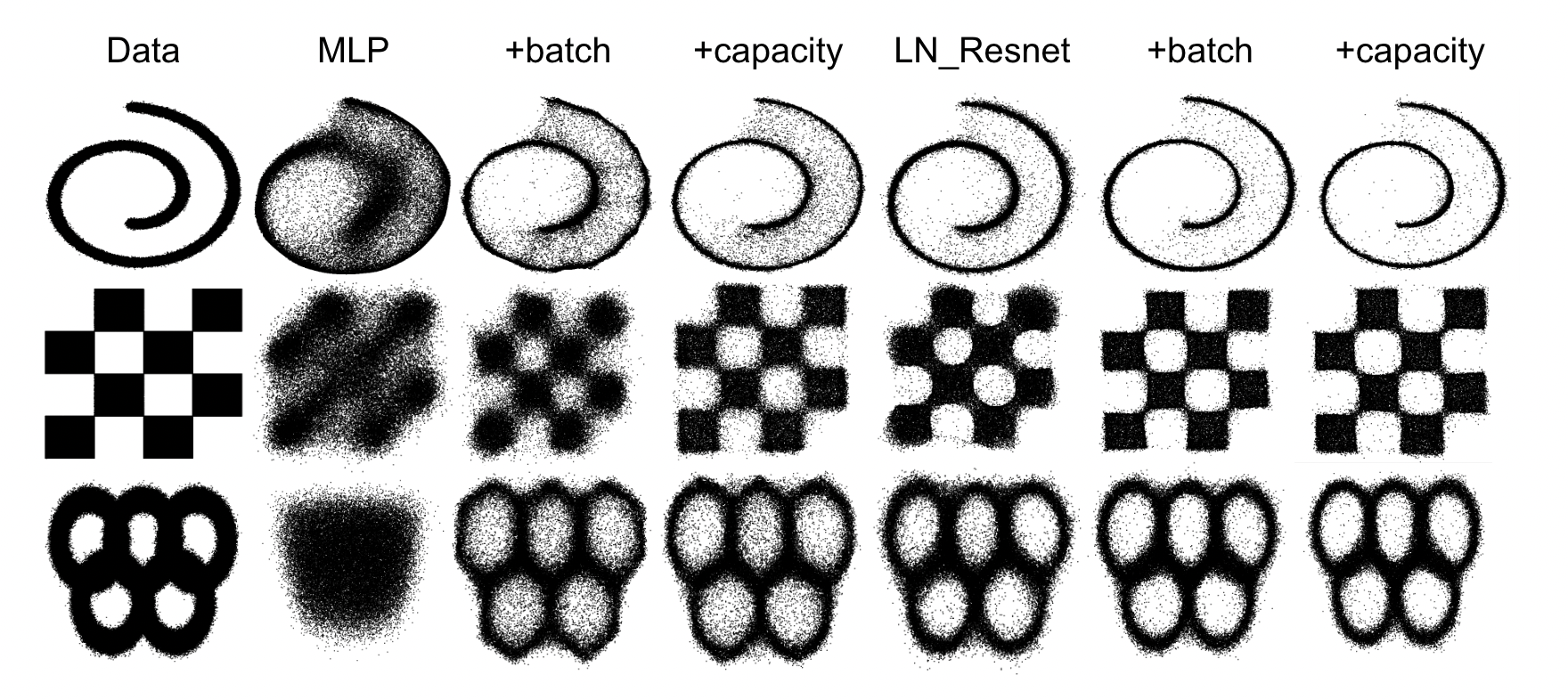}
        \caption{Samples from configurations of DDPMs for toy 2D continuous datasets. MLP has batch size 256 and two hidden layers size 256. +Batch increases the batch size to 4096 and +capacity uses batch size 4096 and 3 hidden layers. LN\_Resnet follows the same pattern, but uses 2 hidden blocks originally and then 3 hidden blocks for +capacity.}
        \label{fig:2dsamples}
    \end{SCfigure}

As demonstrated with diffusion on images~\citep{nichol2021improved}, we find that increasing the batch size and capacity of the MLP network fits the distribution better, but many outliers remain. Inspired by architectures that work well for transformers \citep{radford2019language}, we find that a more ideal network for diffusion should have high capacity while being well-regularized. We use a residual network \citep{he2016deep} with layer normalization \citep{ba2016layer} (LN\_resnet) as our score network parameterization (depicted in Appendix~\ref{archdepict}). Figure~\ref{fig:2dsamples} shows our architecture choice producing higher quality samples with fewer outliers compared to a standard MLP architecture. We demonstrate in Section~\ref{designchoices} that this architecture choice is crucial for strong performance and reduced sensitivity to $N$. This section provides contributions on how batch size, capacity, and architecture choices affect the modeling of continuous spaces with diffusion. Though, similar architectures for action modeling with diffusion have been proposed before~\citep{chen2022offline, pearce2023imitating, reuss2023goal}.
\vspace{-0.3cm}
\section{Experimental Evaluation}
\vspace{-0.2cm}
In order to assess the performance, scalability, and robustness of our approach, we compare it to prior works using two protocols: (1) a direct comparison on D4RL benchmarks~\citep{fu2020d4rl} with reported numbers, and (2) a ``one hyperparameter evaluation'' where we rerun top-performing methods with one hyperparameter setting tuned per domain. We also analyze our method's performance over the critic objectives mentioned in Section~\ref{acmethod} and in online finetuning. Additional experiment details and results can be found in Appendix~\ref{experimentdetails} and Appendix~\ref{additionalablations} respectively.

\subsection{Offline RL Results}
\label{mainresults}

A major appeal of offline RL methods is their ability to produce effective policies without any online interaction. The less tuning that is required for a given method, the easier it will be to use in the real world. Therefore, in this section, we evaluate select methods both with any number of hyperparameter allowed (or reported results) and in a regime where only one hyperparameter can be tuned per domain. We select Conservative Q-learning (CQL) \citep{kumar2020conservative}, Implicit Q-learning (IQL) \citep{kostrikov2021offline}, and Diffusion Q-learning (DQL) \citep{wang2022diffusion} to focus on because of their strong performance in the standard offline RL setting. We refer "-A" as reported results that allow any amount of tuning and "-1" as results that only allow one hyperparameter to be tuned between domains (for IDQL, $\tau = 0.7$ for all locomotion tasks and $\tau = 0.9$ for all antmaze tasks). Results are in Table~\ref{table:d4rl}. We also include comparisons to \%BC, Decision Transformers (DT) \citep{chen2021decision}, TD3+BC (TD3) \citep{fujimoto2021minimalist}, Extreme Q-learning (EQL) \citep{garg2023extreme}, and Selecting from Behavior Candidates (SfBC) \citep{chen2022offline} in Table~\ref{table:totalfulld4rl} (full table in Appendix~\ref{additionalablations}). 

\begin{table*}
\small
\begin{tabular*}{\textwidth}{@{\extracolsep{\fill}}lrrr|r|rrr|r}
\multicolumn{1}{c}{\bf Dataset} & \multicolumn{1}{c}{\bf CQL-A} & \multicolumn{1}{c}{\bf IQL-A} & \multicolumn{1}{c}{\bf DQL-A} & \multicolumn{1}{c}{\bf IDQL-A} & \multicolumn{1}{c}{\bf CQL-1} & \multicolumn{1}{c}{\bf IQL-1} & \multicolumn{1}{c}{\bf DQL-1}  & \multicolumn{1}{c}{\bf IDQL-1} \\
\shline
halfcheetah-med &44.0  & 47.4 & \textbf{51.1}  & \textbf{51.0}  &46.4  &47.6 & \textbf{50.6} & \textbf{49.7} \\
hopper-med &58.5  &66.3  & \textbf{90.5}  & 65.4 & 64.4 &63.7 & \textbf{75.2} & 63.1 \\
walker2d-med &72.5 &78.3  & \textbf{87.0}  & \textbf{82.5} & \textbf{81.6} &\textbf{81.9} & \textbf{83.4} & \textbf{80.2}\\
halfcheetah-med-rep & \textbf{45.5}  & 44.2  & \textbf{47.8}  & \textbf{45.9}  &\textbf{45.4} &43.1 & \textbf{45.8} & \textbf{45.1} \\
hopper-med-rep   &\textbf{95.0}  &\textbf{94.7} & \textbf{101.3} & \textbf{92.1}  & 88.5 & 42.5 &\textbf{94.5} & 82.4  \\
walker2d-med-rep  &77.2  &73.9  & \textbf{95.5}  & 85.1  &74.5 &\textbf{78.4} & \textbf{86.7} & \textbf{79.8}\\
halfcheetah-med-exp &\textbf{91.6}  &86.7  &\textbf{96.8}  & \textbf{95.9}  &64.6 &88.1 & \textbf{93.3} & \textbf{94.4}\\
hopper-med-exp &\textbf{105.4}  &91.5  & \textbf{111.1} & \textbf{108.6}  &99.3 &73.7 & \textbf{102.1} & \textbf{105.3}  \\
walker2d-med-exp &\textbf{108.8}  &\textbf{109.6}  & \textbf{110.1}  & \textbf{112.7}  &\textbf{109.6} &\textbf{110.5} & \textbf{109.6} & \textbf{111.6}\\ \shline
locomotion-v2 total & 698.5  & 692.4  & \textbf{791.2}  & \textbf{739.2}  &674.3 & 629.5 & \textbf{741.2} & \textbf{711.6}\\ \shline
antmaze-umaze &74.0  & 87.5  & \textbf{93.4}  & \textbf{94.0}  &65.0 &86.4 &47.6 & \textbf{93.8}\\
antmaze-umaze-div &\textbf{84.0}  &62.2  & 66.2  & \textbf{80.2}  &41.3 &\textbf{62.4} &35.8 & \textbf{62.0}\\
antmaze-med-play &61.2  & 71.2  & 76.6  & \textbf{84.5}  &31.4 & 76.0 &42.5 & \textbf{86.6}\\
antmaze-med-div &53.7  & 70.0 & 78.6  & \textbf{84.8}  &25.8 & 74.8 &46.3 & \textbf{83.5} \\
antmaze-large-play &15.8  & 39.6  & 46.4  & \textbf{63.5}  &8.5 &31.6 &19.0 & \textbf{57.0}\\
antmaze-large-div &14.9  & 47.5  & 57.3  & \textbf{67.9}  &7.0 &36.4 & 25.2 & \textbf{56.4} \\ 
\shline
antmaze-v0 total &303.6  & 378.0  & 418.5  & \textbf{474.6}  &180.0  & 368.4 & 216.4 & \textbf{439.3}\\ 
\shline
total &1002.1  & 1070.4  & \textbf{1209.7}  & \textbf{1213.8}  &854.3 & 997.9 & 957.4 & \textbf{1150.9}\\
\shline
training time &80m &20m & 240m & 60m & 80m & 20m & 240m & 40m
\end{tabular*}
\vspace{.1cm}
\caption{
\textbf{Focused offline RL comparison.} IDQL performs on par or better than other SOTA offline RL methods. "-A" refers to reported results and "-1" allows only one hyperparameter.
}
\label{table:d4rl}
\end{table*}

\begin{table*}
\small
\begin{tabular*}{\textwidth}{@{\extracolsep{\fill}}lrrrrrrrr|r}
\multicolumn{1}{c}{\bf Dataset} & \multicolumn{1}{c}{\bf \%BC} & \multicolumn{1}{c}{\bf DT} & \multicolumn{1}{c}{\bf TD3} & \multicolumn{1}{c}{\bf CQL} & \multicolumn{1}{c}{\bf IQL} & \multicolumn{1}{c}{\bf EQL} & \multicolumn{1}{c}{\bf SfBC} &  \multicolumn{1}{c}{\bf DQL} & \multicolumn{1}{c}{\bf IDQL}  \\ 
\shline
locomotion-v2 total & 666.2 & 672.6 & 677.4 & 698.5 & 692.4 &  \textbf{725.3} & 680.4 & \textbf{791.2} & \textbf{739.2}\\
\shline
antmaze-v0 total &134.2 & 112.2 &163.8 &303.6 & 378.0 & 386 & \textbf{445.2} & 418.5 & \textbf{474.6} \\ 
\shline
total &800.4 & 784.8 &841.2 &1002.1 & 1070.4 & 1111.3 & 1125.6 & \textbf{1209.7} & \textbf{1213.8} \\
\shline
training time &10m &960m & 20m & 80m & 20m & 20m & 785m & 240m & 60m
\end{tabular*}
\vspace{.1cm}
\caption{
\textbf{Full comparison for offline RL.} We compare IDQL-A to other prior offline RL methods. IDQL outperforms all methods in total score and receives the strongest antmaze results. 
}
\label{table:totalfulld4rl}
\end{table*}

In the standard evaluation protocol, IDQL performs competitively to the best prior methods on the locomotion tasks while outperforming prior methods on the antmaze tasks. In the one hyperparameter regime, the performance of IDQL degrades only slightly from the results in Table~\ref{table:d4rl}, while the other prior methods suffer considerably more, particularly on the more challenging antmaze domain. Thus, with limited tuning, IDQL outperforms the prior methods by a very significant margin. This is specifically apparent in the antmaze domain, where our one hyperparameter results outperfrom the best method (IQL) by +70 points. The details for this experiment can be found in Appendix~\ref{experimentdetails}. 

Furthermore, we compare training time between the different methods in Table~\ref{table:fulld4rl}. IDQL remains computationally efficient like IQL. In particular, IDQL is much faster than the other two diffusion methods \citet{chen2022offline} and \citet{wang2022diffusion}. Although IDQL uses a more expressive policy model, the critic training is completely separated from this model, leading to IDQL's computational efficiency.
\vspace{-0.4cm}
\subsection{Online Finetuning}

After offline training, policies can be improved with online interactions. We test the procedure of freezing the behavior policy and finetuning the value networks only, as well as finetuning all networks, as described in Section~\ref{practicalfinetuning}. We compare to current state-of-the-art finetuning methods: Cal-QL~\citep{nakamoto2023cal}, RLPD~\citep{ball2023efficient}, and IQL~\citep{kostrikov2021offline}. Results are presented in Table~\ref{table:fullfinetune}. We see large improvement in both pre-training and final fine-tuning performance compared to IQL. IDQL also remains competitive with RLPD and Cal-QL in finetuning, while having stronger pre-training results. Most of the gains come from improvements in the hardest antmaze-large environments.

\begin{table*}
\centering
\small
\begin{tabular*}{\textwidth}{@{\extracolsep{\fill}}lrrr|rr}
\multicolumn{1}{c}{\bf Dataset}  & \multicolumn{1}{c}{\bf Cal-QL} & \multicolumn{1}{c}{\bf RLPD}  & \multicolumn{1}{c}{\bf IQL}  & \multicolumn{1}{c}{\bf IDQL-Max} & \multicolumn{1}{c}{\bf IDQL-Imp}\\ 
\shline
    antmaze-umaze & $- \rightarrow -$ & $0.0 \rightarrow \textbf{99.0}$  & $\textbf{86.7} \rightarrow \textbf{96.0}$ & $\textbf{92.0} \rightarrow \textbf{99.0}$ & $\ \textbf{93.5} \rightarrow \textbf{99.5}$\\
    antmaze-umaze-diverse  & $- \rightarrow  -$ & $0.0 \rightarrow \textbf{99.0}$ & $\textbf{75.0} \rightarrow 84.0$ & $\textbf{78.7} \rightarrow 85.6$ & $\textbf{78.4} \rightarrow 73.0$ \\
    antmaze-medium-play & $54.0 \rightarrow \textbf{98.0}$ & $0.0 \rightarrow \textbf{99.5}$ & $72.0 \rightarrow \textbf{95.0}$ & $\textbf{83.3} \rightarrow \textbf{97.8}$ & $\textbf{84.4} \rightarrow \textbf{94.0}$\\
    antmaze-medium-diverse  & $73.0 \rightarrow \textbf{98.0}$ & $0.0 \rightarrow \textbf{98.0}$ & $\textbf{68.3} \rightarrow \textbf{92.0}$ & $\textbf{84.7} \rightarrow \textbf{98.0}$ & $\textbf{84.0} \rightarrow \textbf{98.7}$\\
    antmaze-large-play & $28.0 \rightarrow \textbf{90.0}$  & $0.0 \rightarrow \textbf{88.0}$  & $25.5 \rightarrow 46.0$ & $\textbf{60.1} \rightarrow \textbf{88.0}$ & $\textbf{60.3} \rightarrow \textbf{90.0}$ \\
    antmaze-large-diverse  & $32.0 \rightarrow \textbf{94.0}$ & $0.0 \rightarrow \textbf{87.5}$ & $42.6 \rightarrow 60.7$ & $\textbf{61.1} \rightarrow \textbf{90.7}$ & $\textbf{61.4} \rightarrow \textbf{93.0}$  \\ \shline
    total  & $- \rightarrow -$ & $0.0 \rightarrow \textbf{571.0}$ & $408.2 \rightarrow 473.7$ & $\textbf{459.9} \rightarrow \textbf{559.1}$ & $\textbf{462} \rightarrow \textbf{548.2}$ \\
\end{tabular*}
\vspace{.1cm}
\caption{
\textbf{Online finetuning results.} We compare finetuning performance with IDQL using two approaches: "-Max" refers to freezing the behavior policy and finetuning the value networks and "-Imp" refers to finetuning all networks as described in Section~\ref{practicalfinetuning}. 
}
\label{table:fullfinetune}
\end{table*}

\subsection{IQL Objective Ablations}
\label{objectiveablation}
In Section \ref{acmethod}, we discussed different objectives and the different implicit policies they induce. We compare the critic losses on the D4RL benchmarks in Figure \ref{figure:objective_ablation}. Both expectiles and quantiles perform well with argmax extraction, but the exponential loss is unstable and performs worse. This indicates that the exponential critic proposed by \citet{garg2023extreme} and \citet{xu2023offline} require more tuning and tricks to perform well, though these works also report instability issues with the objective. The implicit policy distributions also do not perform as well as the argmax extraction, but this is most likely because these policies have a non-zero probability of taking a poor action. This follows similar findings that stochastic policies work best for critic learning, but deterministic policies work best for evaluation \citep{haarnoja2018soft, brandfonbrener2021offline}. Overall, the expectile objective performs the strongest with greedy extraction.
\begin{figure}
    \includegraphics[width=0.24\textwidth]{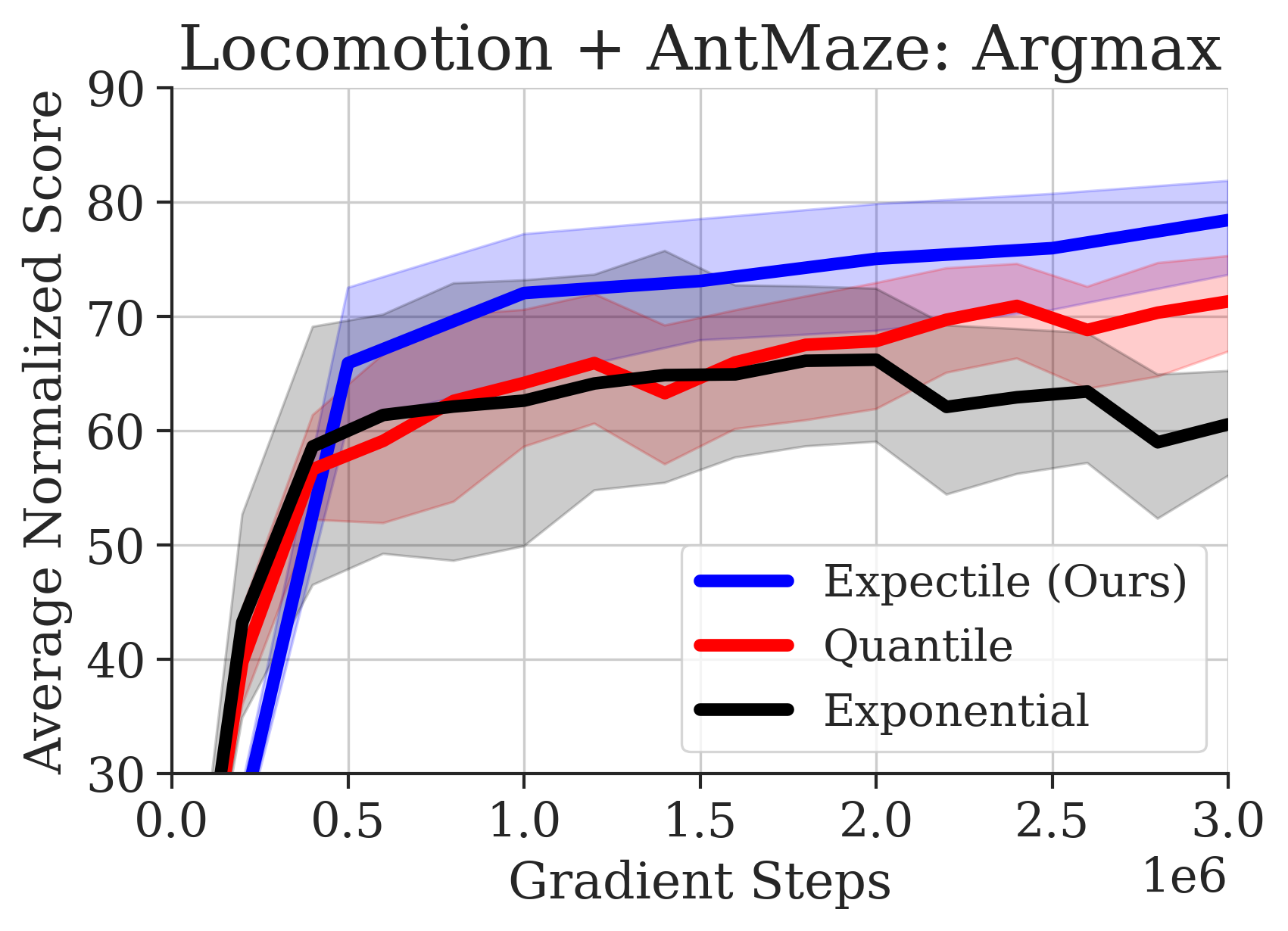} \includegraphics[width=0.23\textwidth]{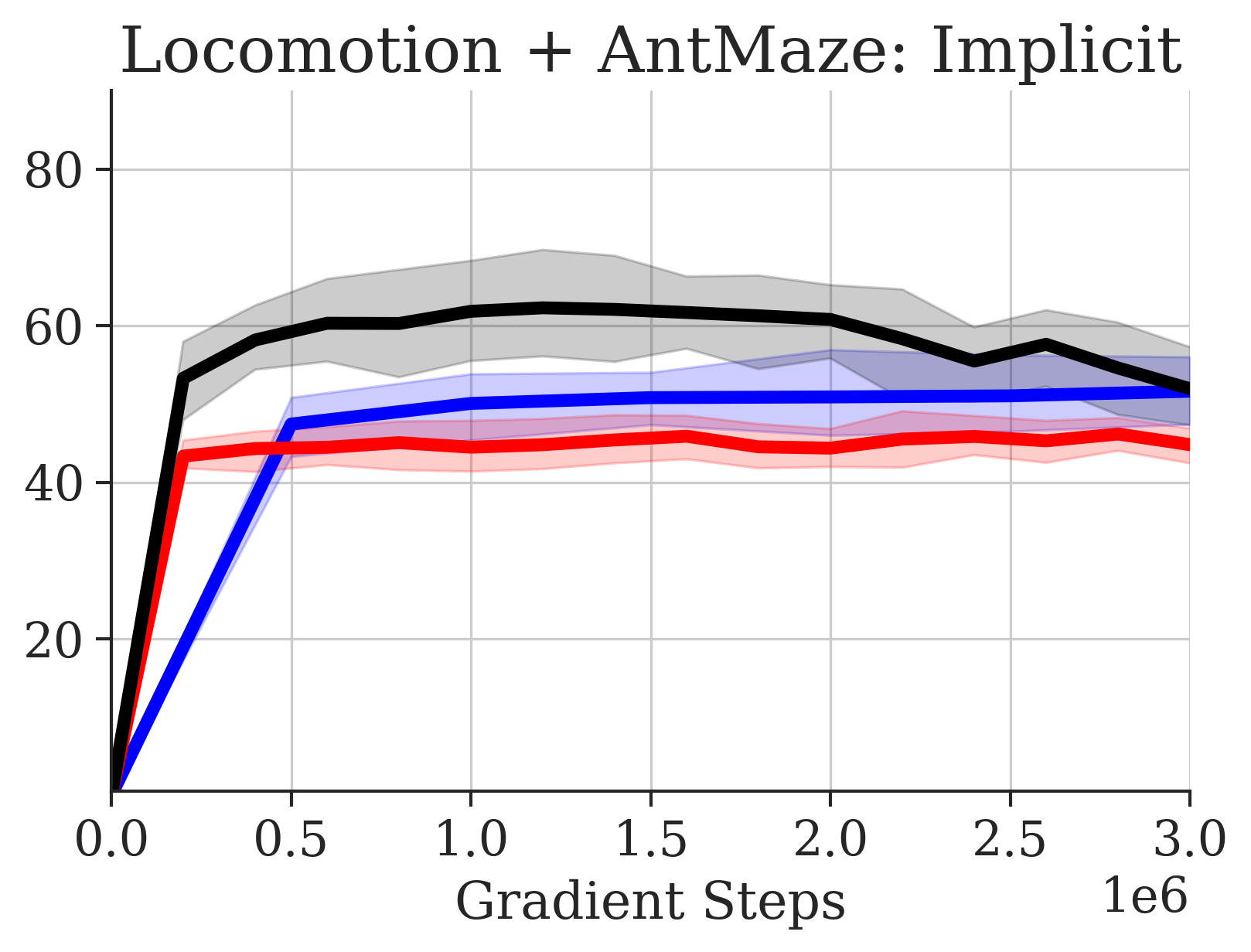} \includegraphics[width=0.24\textwidth]{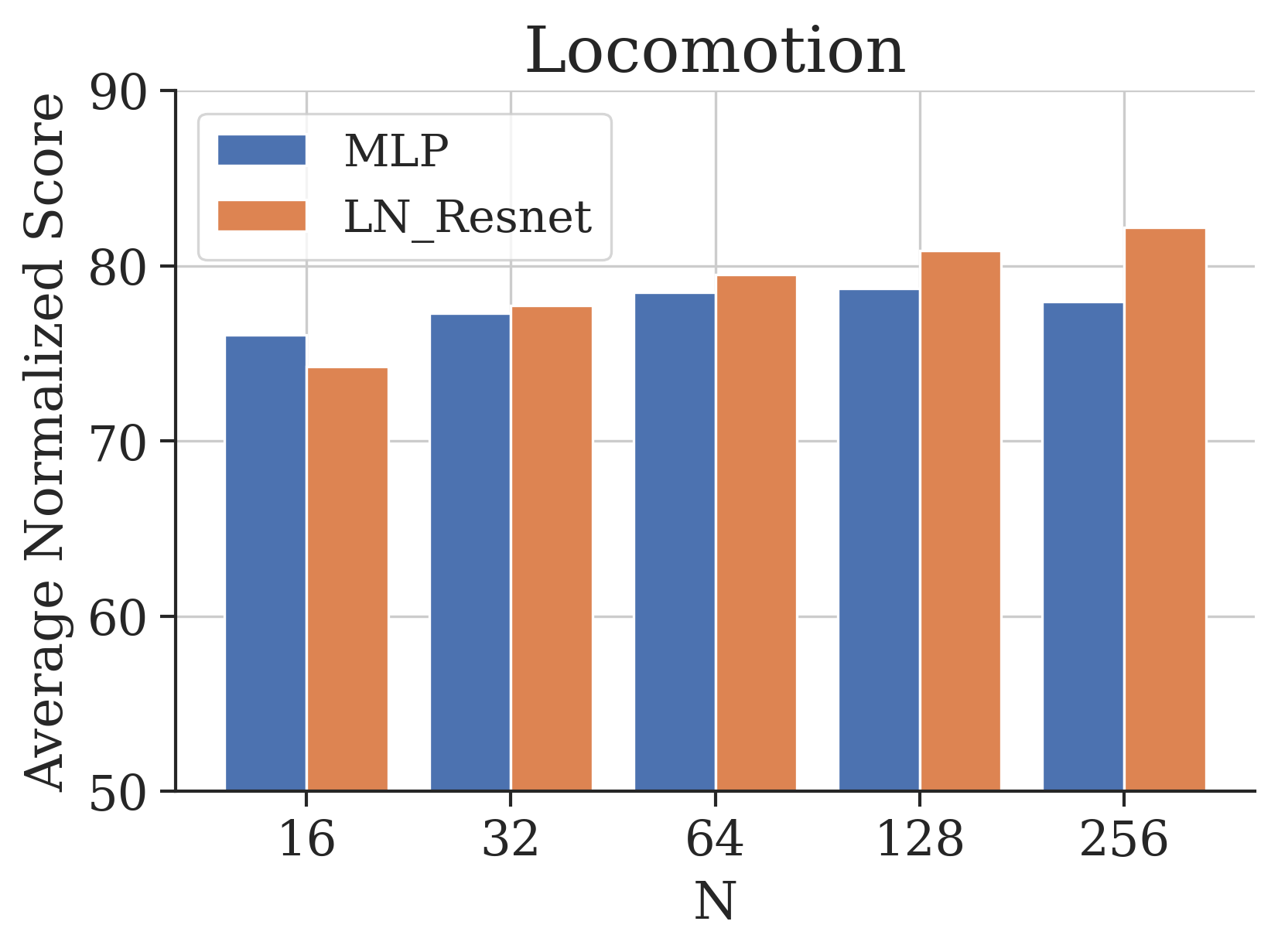} \includegraphics[width=0.23\textwidth]{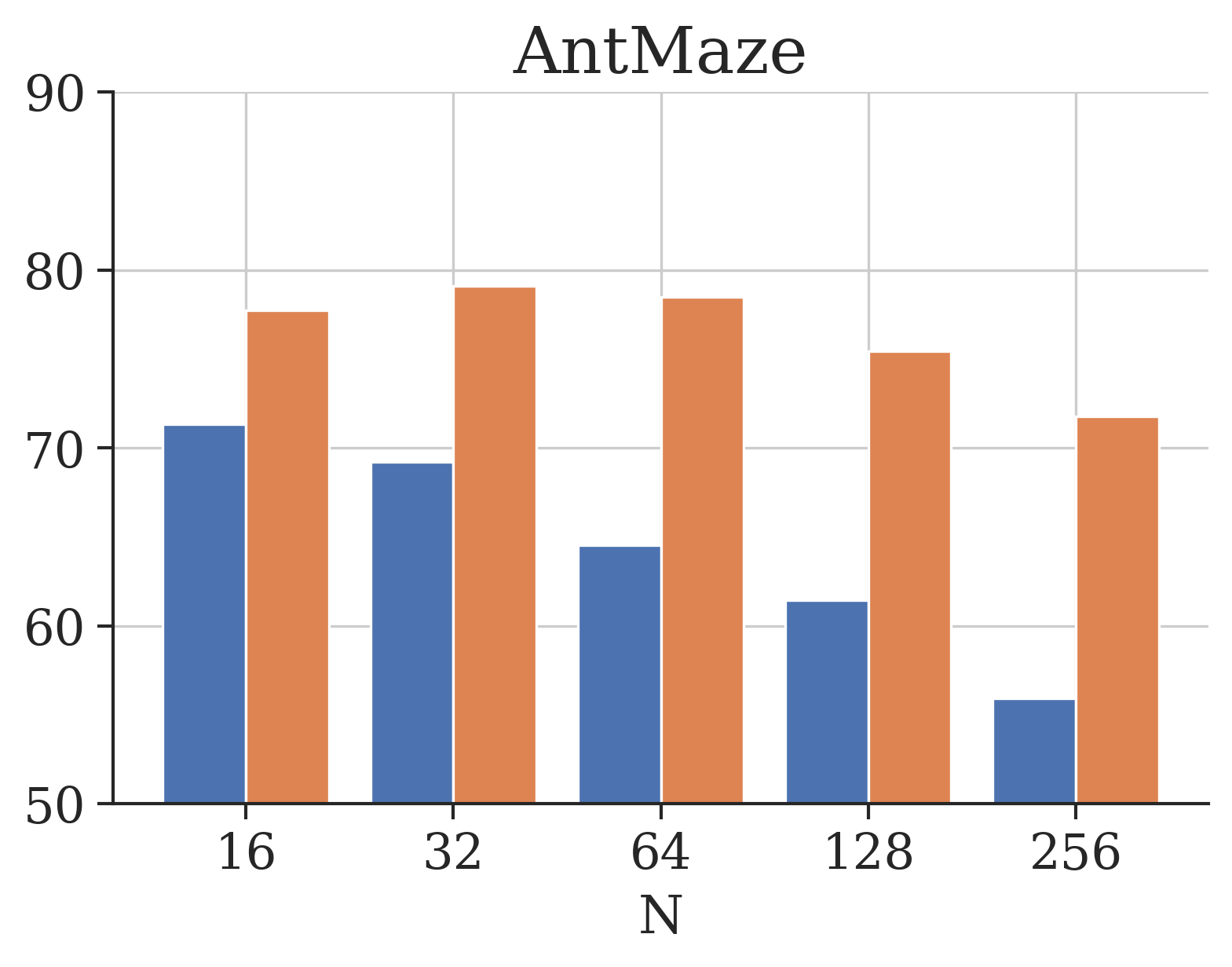} 
    \caption{\textbf{Objective and architecture ablation.} We compare the IQL objectives on all D4RL tasks using argmax extraction (first) and implicit policy extraction (second). We also include ablations on architecture choice (MLP vs LN\_Resnet) on locomotion (third) and antmaze (fourth) tasks.}
    \label{figure:objective_ablation}
\end{figure}

\subsection{Evaluating Diffusion Architecture Choices}
\label{designchoices}
In Section \ref{section:issuesdiff}, we introduced design choices that were important for reducing outliers and increasing the expressivity of diffusion models. To confirm that our architecture aids in performance and sensitivity to $N$, we compare our LN\_Resnet architecture to a three-layer MLP which contains a similar number of activation layers. We sweep over $N$ and measure the total final score. Results are in Figure~\ref{figure:objective_ablation}. Using an LN\_resnet is crucial for a reduction in sensitivity to $N$: for locomotion results, a higher $N$ leads to stronger performance, and for antmaze results increasing $N$ has a small effect on results. For the MLP architecture, increasing $N$ decreases performance, especially in antmaze.

\section{Discussion and Limitations}
In this work, we generalize IQL into an actor-critic method, where choices of convex asymmetric critic loss induce a behavior regularized implicit actor (Theorem~\ref{thm:Implicit_Policy}). The implicit policy is shown to be a complex importance weighted behavior distribution. This suggests that policy extraction methods based on simple Gaussian policies used in prior work might not perform well in IQL. We confirm this hypothesis by proposing a new policy extraction approach based on expressive diffusion models, describe a number of architecture design decisions that make such policies work well in practice, and present state-of-the-art results across offline RL benchmarks. Our method performs particularly well when the amount of hyperparameter tuning is restricted, which is important for practical applications of offline RL where tuning is often difficult or impossible. While IDQL performed well in most settings, there were some issues with overfitting on environments with a small dimension action space (e.g., locomotion). Adding dropout mitigates this issue but at the expense of other environments. Furthermore, IDQL does not work well in online finetuning of the Adroit environments. We suspect that excessive pre-training harms online fine-tuning performance in this domain.

Our analysis shows that, although we can generalize IQL to use a variety of loss functions, the original expectile loss ends up performing the best on current tasks. However, we believe that this generalization still has considerable value in informing future research. First, it illuminates how IQL is actually an actor-critic method, and provides for an entire family of implicit actor-critic methods. Future work might investigate new and more effective loss functions. Our work also illustrates how critical the choice of policy extraction method is for implicit Q-learning methods. Finally, our work provides an effective, easy to implement, computationally efficient, and relatively hyperparameter insensitive approach for integrating diffusion models into offline RL.

\section{Acknowledgements}

This research was supported by the Office of Naval Research and AFOSR FA9550-22-1-0273, with compute support from Berkeley Research Computing. We thank Manan Tomar for help with early drafts of the paper.

\bibliographystyle{plainnat}
\bibliography{main_paper.bib}

\begin{thebibliography}{49}
\providecommand{\natexlab}[1]{#1}
\providecommand{\url}[1]{\texttt{#1}}
\expandafter\ifx\csname urlstyle\endcsname\relax
  \providecommand{\doi}[1]{doi: #1}\else
  \providecommand{\doi}{doi: \begingroup \urlstyle{rm}\Url}\fi

\bibitem[Ajay et~al.(2022)Ajay, Du, Gupta, Tenenbaum, Jaakkola, and
  Agrawal]{ajay2022conditional}
Anurag Ajay, Yilun Du, Abhi Gupta, Joshua Tenenbaum, Tommi Jaakkola, and Pulkit
  Agrawal.
\newblock Is conditional generative modeling all you need for decision-making?
\newblock \emph{arXiv preprint arXiv:2211.15657}, 2022.

\bibitem[Ba et~al.(2016)Ba, Kiros, and Hinton]{ba2016layer}
Jimmy~Lei Ba, Jamie~Ryan Kiros, and Geoffrey~E Hinton.
\newblock Layer normalization.
\newblock \emph{arXiv preprint arXiv:1607.06450}, 2016.

\bibitem[Ball et~al.(2023)Ball, Smith, Kostrikov, and
  Levine]{ball2023efficient}
Philip~J Ball, Laura Smith, Ilya Kostrikov, and Sergey Levine.
\newblock Efficient online reinforcement learning with offline data.
\newblock \emph{arXiv preprint arXiv:2302.02948}, 2023.

\bibitem[Bradbury et~al.(2018)Bradbury, Frostig, Hawkins, Johnson, Leary,
  Maclaurin, Necula, Paszke, Vander{P}las, Wanderman-{M}ilne, and
  Zhang]{jax2018github}
James Bradbury, Roy Frostig, Peter Hawkins, Matthew~James Johnson, Chris Leary,
  Dougal Maclaurin, George Necula, Adam Paszke, Jake Vander{P}las, Skye
  Wanderman-{M}ilne, and Qiao Zhang.
\newblock {JAX}: composable transformations of {P}ython+{N}um{P}y programs,
  2018.
\newblock URL \url{http://github.com/google/jax}.

\bibitem[Brandfonbrener et~al.(2021)Brandfonbrener, Whitney, Ranganath, and
  Bruna]{brandfonbrener2021offline}
David Brandfonbrener, Will Whitney, Rajesh Ranganath, and Joan Bruna.
\newblock Offline rl without off-policy evaluation.
\newblock \emph{Advances in neural information processing systems},
  34:\penalty0 4933--4946, 2021.

\bibitem[Byrd and Lipton(2019)]{byrd2019effect}
Jonathon Byrd and Zachary Lipton.
\newblock What is the effect of importance weighting in deep learning?
\newblock In \emph{International conference on machine learning}, pages
  872--881. PMLR, 2019.

\bibitem[Chen et~al.(2022)Chen, Lu, Ying, Su, and Zhu]{chen2022offline}
Huayu Chen, Cheng Lu, Chengyang Ying, Hang Su, and Jun Zhu.
\newblock Offline reinforcement learning via high-fidelity generative behavior
  modeling.
\newblock \emph{arXiv preprint arXiv:2209.14548}, 2022.

\bibitem[Chen et~al.(2021)Chen, Lu, Rajeswaran, Lee, Grover, Laskin, Abbeel,
  Srinivas, and Mordatch]{chen2021decision}
Lili Chen, Kevin Lu, Aravind Rajeswaran, Kimin Lee, Aditya Grover, Misha
  Laskin, Pieter Abbeel, Aravind Srinivas, and Igor Mordatch.
\newblock Decision transformer: Reinforcement learning via sequence modeling.
\newblock \emph{Advances in neural information processing systems},
  34:\penalty0 15084--15097, 2021.

\bibitem[Dabney et~al.(2018)Dabney, Rowland, Bellemare, and
  Munos]{dabney2018distributional}
Will Dabney, Mark Rowland, Marc Bellemare, and R{\'e}mi Munos.
\newblock Distributional reinforcement learning with quantile regression.
\newblock In \emph{Proceedings of the AAAI Conference on Artificial
  Intelligence}, volume~32, 2018.

\bibitem[Florence et~al.(2021)Florence, Lynch, Zeng, Ramirez, Wahid, Downs,
  Wong, Lee, Mordatch, and Tompson]{florence2021implicit}
Pete Florence, Corey Lynch, Andy Zeng, Oscar~A Ramirez, Ayzaan Wahid, Laura
  Downs, Adrian Wong, Johnny Lee, Igor Mordatch, and Jonathan Tompson.
\newblock Implicit behavioral cloning.
\newblock In \emph{5th Annual Conference on Robot Learning}, 2021.
\newblock URL \url{https://openreview.net/forum?id=rif3a5NAxU6}.

\bibitem[Fu et~al.(2020)Fu, Kumar, Nachum, Tucker, and Levine]{fu2020d4rl}
Justin Fu, Aviral Kumar, Ofir Nachum, George Tucker, and Sergey Levine.
\newblock D4rl: Datasets for deep data-driven reinforcement learning.
\newblock \emph{arXiv preprint arXiv:2004.07219}, 2020.

\bibitem[Fujimoto and Gu(2021)]{fujimoto2021minimalist}
Scott Fujimoto and Shixiang~Shane Gu.
\newblock A minimalist approach to offline reinforcement learning.
\newblock \emph{Advances in neural information processing systems},
  34:\penalty0 20132--20145, 2021.

\bibitem[Fujimoto et~al.(2019)Fujimoto, Meger, and Precup]{fujimoto2019off}
Scott Fujimoto, David Meger, and Doina Precup.
\newblock Off-policy deep reinforcement learning without exploration.
\newblock In \emph{International conference on machine learning}, pages
  2052--2062. PMLR, 2019.

\bibitem[Garg et~al.(2023)Garg, Hejna, Geist, and Ermon]{garg2023extreme}
Divyansh Garg, Joey Hejna, Matthieu Geist, and Stefano Ermon.
\newblock Extreme q-learning: Maxent rl without entropy.
\newblock \emph{arXiv preprint arXiv:2301.02328}, 2023.

\bibitem[Ghasemipour et~al.(2021)Ghasemipour, Schuurmans, and
  Gu]{ghasemipour2021emaq}
Seyed Kamyar~Seyed Ghasemipour, Dale Schuurmans, and Shixiang~Shane Gu.
\newblock Emaq: Expected-max q-learning operator for simple yet effective
  offline and online rl.
\newblock In \emph{International Conference on Machine Learning}, pages
  3682--3691. PMLR, 2021.

\bibitem[Goo and Niekum(2022)]{goo2022know}
Wonjoon Goo and Scott Niekum.
\newblock Know your boundaries: The necessity of explicit behavioral cloning in
  offline rl.
\newblock \emph{arXiv preprint arXiv:2206.00695}, 2022.

\bibitem[Haarnoja et~al.(2018)Haarnoja, Zhou, Abbeel, and
  Levine]{haarnoja2018soft}
Tuomas Haarnoja, Aurick Zhou, Pieter Abbeel, and Sergey Levine.
\newblock Soft actor-critic: Off-policy maximum entropy deep reinforcement
  learning with a stochastic actor.
\newblock In \emph{International conference on machine learning}, pages
  1861--1870. PMLR, 2018.

\bibitem[He et~al.(2016)He, Zhang, Ren, and Sun]{he2016deep}
Kaiming He, Xiangyu Zhang, Shaoqing Ren, and Jian Sun.
\newblock Deep residual learning for image recognition.
\newblock In \emph{Proceedings of the IEEE conference on computer vision and
  pattern recognition}, pages 770--778, 2016.

\bibitem[Heek et~al.(2023)Heek, Levskaya, Oliver, Ritter, Rondepierre, Steiner,
  and van {Z}ee]{flax2020github}
Jonathan Heek, Anselm Levskaya, Avital Oliver, Marvin Ritter, Bertrand
  Rondepierre, Andreas Steiner, and Marc van {Z}ee.
\newblock {F}lax: A neural network library and ecosystem for {JAX}, 2023.
\newblock URL \url{http://github.com/google/flax}.

\bibitem[Ho et~al.(2020)Ho, Jain, and Abbeel]{ho2020denoising}
Jonathan Ho, Ajay Jain, and Pieter Abbeel.
\newblock Denoising diffusion probabilistic models.
\newblock \emph{Advances in Neural Information Processing Systems},
  33:\penalty0 6840--6851, 2020.

\bibitem[Janner et~al.(2021)Janner, Li, and Levine]{janner2021offline}
Michael Janner, Qiyang Li, and Sergey Levine.
\newblock Offline reinforcement learning as one big sequence modeling problem.
\newblock \emph{Advances in neural information processing systems},
  34:\penalty0 1273--1286, 2021.

\bibitem[Janner et~al.(2022)Janner, Du, Tenenbaum, and
  Levine]{janner2022planning}
Michael Janner, Yilun Du, Joshua~B Tenenbaum, and Sergey Levine.
\newblock Planning with diffusion for flexible behavior synthesis.
\newblock \emph{arXiv preprint arXiv:2205.09991}, 2022.

\bibitem[Kingma and Ba(2014)]{kingma2014adam}
Diederik~P Kingma and Jimmy Ba.
\newblock Adam: A method for stochastic optimization.
\newblock \emph{arXiv preprint arXiv:1412.6980}, 2014.

\bibitem[Kostrikov(2021)]{jaxrl}
Ilya Kostrikov.
\newblock {JAXRL: Implementations of Reinforcement Learning algorithms in JAX},
  10 2021.
\newblock URL \url{https://github.com/ikostrikov/jaxrl}.

\bibitem[Kostrikov et~al.(2021{\natexlab{a}})Kostrikov, Fergus, Tompson, and
  Nachum]{kostrikovfisher2021offline}
Ilya Kostrikov, Rob Fergus, Jonathan Tompson, and Ofir Nachum.
\newblock Offline reinforcement learning with fisher divergence critic
  regularization.
\newblock In \emph{International Conference on Machine Learning}, pages
  5774--5783. PMLR, 2021{\natexlab{a}}.

\bibitem[Kostrikov et~al.(2021{\natexlab{b}})Kostrikov, Nair, and
  Levine]{kostrikov2021offline}
Ilya Kostrikov, Ashvin Nair, and Sergey Levine.
\newblock Offline reinforcement learning with implicit q-learning.
\newblock \emph{arXiv preprint arXiv:2110.06169}, 2021{\natexlab{b}}.

\bibitem[Kumar et~al.(2019)Kumar, Fu, Soh, Tucker, and
  Levine]{kumar2019stabilizing}
Aviral Kumar, Justin Fu, Matthew Soh, George Tucker, and Sergey Levine.
\newblock Stabilizing off-policy q-learning via bootstrapping error reduction.
\newblock \emph{Advances in Neural Information Processing Systems}, 32, 2019.

\bibitem[Kumar et~al.(2020)Kumar, Zhou, Tucker, and
  Levine]{kumar2020conservative}
Aviral Kumar, Aurick Zhou, George Tucker, and Sergey Levine.
\newblock Conservative q-learning for offline reinforcement learning.
\newblock \emph{Advances in Neural Information Processing Systems},
  33:\penalty0 1179--1191, 2020.

\bibitem[Kuznetsov et~al.(2020)Kuznetsov, Shvechikov, Grishin, and
  Vetrov]{kuznetsov2020controlling}
Arsenii Kuznetsov, Pavel Shvechikov, Alexander Grishin, and Dmitry Vetrov.
\newblock Controlling overestimation bias with truncated mixture of continuous
  distributional quantile critics.
\newblock In \emph{International Conference on Machine Learning}, pages
  5556--5566. PMLR, 2020.

\bibitem[Lange et~al.(2012)Lange, Gabel, and Riedmiller]{lange2012batch}
Sascha Lange, Thomas Gabel, and Martin Riedmiller.
\newblock Batch reinforcement learning.
\newblock \emph{Reinforcement learning: State-of-the-art}, pages 45--73, 2012.

\bibitem[Levine et~al.(2020)Levine, Kumar, Tucker, and Fu]{levine2020offline}
Sergey Levine, Aviral Kumar, George Tucker, and Justin Fu.
\newblock Offline reinforcement learning: Tutorial, review, and perspectives on
  open problems.
\newblock \emph{arXiv preprint arXiv:2005.01643}, 2020.

\bibitem[Loshchilov and Hutter(2016)]{loshchilov2016sgdr}
Ilya Loshchilov and Frank Hutter.
\newblock Sgdr: Stochastic gradient descent with warm restarts.
\newblock \emph{arXiv preprint arXiv:1608.03983}, 2016.

\bibitem[Nair et~al.(2020)Nair, Gupta, Dalal, and Levine]{nair2020awac}
Ashvin Nair, Abhishek Gupta, Murtaza Dalal, and Sergey Levine.
\newblock Awac: Accelerating online reinforcement learning with offline
  datasets.
\newblock \emph{arXiv preprint arXiv:2006.09359}, 2020.

\bibitem[Nakamoto et~al.(2023)Nakamoto, Zhai, Singh, Mark, Ma, Finn, Kumar, and
  Levine]{nakamoto2023cal}
Mitsuhiko Nakamoto, Yuexiang Zhai, Anikait Singh, Max~Sobol Mark, Yi~Ma,
  Chelsea Finn, Aviral Kumar, and Sergey Levine.
\newblock Cal-ql: Calibrated offline rl pre-training for efficient online
  fine-tuning.
\newblock \emph{arXiv preprint arXiv:2303.05479}, 2023.

\bibitem[Nichol and Dhariwal(2021)]{nichol2021improved}
Alexander~Quinn Nichol and Prafulla Dhariwal.
\newblock Improved denoising diffusion probabilistic models.
\newblock In \emph{International Conference on Machine Learning}, pages
  8162--8171. PMLR, 2021.

\bibitem[Pearce et~al.(2023)Pearce, Rashid, Kanervisto, Bignell, Sun,
  Georgescu, Macua, Tan, Momennejad, Hofmann, and Devlin]{pearce2023imitating}
Tim Pearce, Tabish Rashid, Anssi Kanervisto, Dave Bignell, Mingfei Sun, Raluca
  Georgescu, Sergio~Valcarcel Macua, Shan~Zheng Tan, Ida Momennejad, Katja
  Hofmann, and Sam Devlin.
\newblock Imitating human behaviour with diffusion models.
\newblock In \emph{International Conference on Learning Representations}, 2023.
\newblock URL \url{https://openreview.net/forum?id=Pv1GPQzRrC8}.

\bibitem[Peng et~al.(2019)Peng, Kumar, Zhang, and Levine]{peng2019advantage}
Xue~Bin Peng, Aviral Kumar, Grace Zhang, and Sergey Levine.
\newblock Advantage-weighted regression: Simple and scalable off-policy
  reinforcement learning.
\newblock \emph{arXiv preprint arXiv:1910.00177}, 2019.

\bibitem[Peters and Schaal(2007)]{peters2007}
Jan Peters and Stefan Schaal.
\newblock {Reinforcement learning by reward-weighted regression for operational
  space control}.
\newblock In \emph{{Proceedings of the 24th International Conference on Machine
  Learning}}, volume 227 of \emph{{ACM International Conference Proceeding
  Series}}, pages 745--750. {ACM}, 2007.
\newblock ISBN 978-1-59593-793-3.
\newblock \doi{10.1145/1273496.1273590}.

\bibitem[Radford et~al.(2019)Radford, Wu, Child, Luan, Amodei, and
  Sutskever]{radford2019language}
Alec Radford, Jeff Wu, Rewon Child, David Luan, Dario Amodei, and Ilya
  Sutskever.
\newblock Language models are unsupervised multitask learners.
\newblock 2019.

\bibitem[Reuss et~al.(2023)Reuss, Li, Jia, and Lioutikov]{reuss2023goal}
Moritz Reuss, Maximilian Li, Xiaogang Jia, and Rudolf Lioutikov.
\newblock Goal-conditioned imitation learning using score-based diffusion
  policies.
\newblock \emph{arXiv preprint arXiv:2304.02532}, 2023.

\bibitem[Sohl-Dickstein et~al.(2015)Sohl-Dickstein, Weiss, Maheswaranathan, and
  Ganguli]{sohl2015deep}
Jascha Sohl-Dickstein, Eric Weiss, Niru Maheswaranathan, and Surya Ganguli.
\newblock Deep unsupervised learning using nonequilibrium thermodynamics.
\newblock In \emph{International Conference on Machine Learning}, pages
  2256--2265. PMLR, 2015.

\bibitem[Sohn et~al.(2015)Sohn, Lee, and Yan]{NIPS2015_8d55a249}
Kihyuk Sohn, Honglak Lee, and Xinchen Yan.
\newblock Learning structured output representation using deep conditional
  generative models.
\newblock In C.~Cortes, N.~Lawrence, D.~Lee, M.~Sugiyama, and R.~Garnett,
  editors, \emph{Advances in Neural Information Processing Systems}, volume~28.
  Curran Associates, Inc., 2015.
\newblock URL
  \url{https://proceedings.neurips.cc/paper/2015/file/8d55a249e6baa5c06772297520da2051-Paper.pdf}.

\bibitem[Song et~al.(2020)Song, Sohl-Dickstein, Kingma, Kumar, Ermon, and
  Poole]{song2020score}
Yang Song, Jascha Sohl-Dickstein, Diederik~P Kingma, Abhishek Kumar, Stefano
  Ermon, and Ben Poole.
\newblock Score-based generative modeling through stochastic differential
  equations.
\newblock \emph{arXiv preprint arXiv:2011.13456}, 2020.

\bibitem[Wang et~al.(2022)Wang, Hunt, and Zhou]{wang2022diffusion}
Zhendong Wang, Jonathan~J Hunt, and Mingyuan Zhou.
\newblock Diffusion policies as an expressive policy class for offline
  reinforcement learning.
\newblock \emph{arXiv preprint arXiv:2208.06193}, 2022.

\bibitem[Wang et~al.(2020)Wang, Novikov, Zolna, Merel, Springenberg, Reed,
  Shahriari, Siegel, Gulcehre, Heess, et~al.]{wang2020critic}
Ziyu Wang, Alexander Novikov, Konrad Zolna, Josh~S Merel, Jost~Tobias
  Springenberg, Scott~E Reed, Bobak Shahriari, Noah Siegel, Caglar Gulcehre,
  Nicolas Heess, et~al.
\newblock Critic regularized regression.
\newblock \emph{Advances in Neural Information Processing Systems},
  33:\penalty0 7768--7778, 2020.

\bibitem[Watkins and Dayan(1992)]{watkins1992q}
Christopher~JCH Watkins and Peter Dayan.
\newblock Q-learning.
\newblock \emph{Machine learning}, 8:\penalty0 279--292, 1992.

\bibitem[Wu et~al.(2019)Wu, Tucker, and Nachum]{wu2019behavior}
Yifan Wu, George Tucker, and Ofir Nachum.
\newblock Behavior regularized offline reinforcement learning.
\newblock \emph{arXiv preprint arXiv:1911.11361}, 2019.

\bibitem[Xu et~al.(2021)Xu, Ye, and Ruan]{xu2021understanding}
Da~Xu, Yuting Ye, and Chuanwei Ruan.
\newblock Understanding the role of importance weighting for deep learning.
\newblock \emph{arXiv preprint arXiv:2103.15209}, 2021.

\bibitem[Xu et~al.(2023)Xu, Jiang, Li, Yang, Wang, Chan, and
  Zhan]{xu2023offline}
Haoran Xu, Li~Jiang, Jianxiong Li, Zhuoran Yang, Zhaoran Wang, Victor Wai~Kin
  Chan, and Xianyuan Zhan.
\newblock Offline rl with no ood actions: In-sample learning via implicit value
  regularization.
\newblock \emph{arXiv preprint arXiv:2303.15810}, 2023.

\end{thebibliography}

\medskip

{
\small


\newpage

\appendix
\section{Reinforcement Learning Definitions}
\label{prelim}
RL is formulated in the context of a Markov decision process (MDP), which is defined as a tuple $( \mathcal{S}, \mathcal{A}, p_0(s), p_M(s' |s, a), r(s,a), \gamma )$ with state space $\mathcal{S}$, action space $\mathcal{A}$, initial state distribution $p_0(s)$, transition dynamics $p_M(s' | s, a)$, reward function $r(s, a)$, and discount $\gamma$. The goal is to recover a policy $\pi(a|s)$ that maximizes the discounted sum of rewards or return in the MDP, 
$$\pi^* = \argmax_\pi \mathbb{E}_{\rho(\pi)}[\sum_{t=0}^{\infty}\gamma^t r(s_t, a_t)] ,$$
where $\rho(\pi)$ describes the distribution over trajectories, $\{s_t, a_t \sim \pi(a|s_t), s_t' \sim p_M(s'|s_t, a_t) \}_{t=0}^{\infty}$, induced by a given policy $\pi$.
A widely used framework for off-policy RL is Q-learning \citep{watkins1992q}, which involves fitting a Q-function $Q(s,a)$ to match the discounted returns starting at state $s$ and action $a$ and following the best current policy at the next state $s'$ ($\argmax_{a'} Q(s',a')$). 
\section{Proof of Theoretical Results}
\label{proofs}
\subsection{Proof of Theorem~\ref{thm:Implicit_Policy}}
\begin{proof} 
We write the objective in Equation~\ref{eqn:general_IQL}.
$$\argmin_{V(s)}\mathbb{E}_{a \sim \mu(a|s)}[f(Q(s,a) - V(s))]$$
Note that the objective function is convex with respect to $V(s)$

$$0 = \left.  \frac{\partial}{\partial V(s)}\mathbb{E}_{a \sim \mu(a|s)}[ f(Q(s,a) - V(s))]\right|_{V=V^*}$$ 
$$ = -\mathbb{E}_{a \sim \mu(a|s)}[f'(Q(s,a) - V^*(s))]$$
Due to convexity of $f$ and the assumption $f'(0) = 0$, $f'(x) = |f'(x)| \cdot \sign(x) = |f'(x)| \frac{x}{|x|} $
$$ = \mathbb{E}_{a \sim \mu(a|s)}\bigg[\frac{|f'(Q(s,a) - V^*(s))| (Q(s,a) - V^*(s))}{|Q(s,a) - V^*(s)|}\bigg].$$

We then define the implicit policy to be
 $\pi_{\text{imp}}(a|s) = \frac{\mu(a|s)|f'(Q(s,a) - V^*(s))|}{Z_{\text{imp}}|Q(s,a) - V^*(s)|}$, where $Z_{\text{imp}}$ is a normalization constant, and rewrite the above expression as
$$ = \mathbb{E}_{a \sim \pi_{\text{imp}}(a|s)} [(Q(s,a) - V^*(s))]$$

$$ = \left. \frac{\partial}{\partial V(s)}
 -\frac{1}{2} \cdot \mathbb{E}_{a \sim \pi_{\text{imp}}(a|s)}[(Q(s,a) - V(s))^2] \right|_{V=V^*}= 0$$

This means that $V^*(s)$ is a solution for the optimization problem 

$$\argmin_{V(s)}\mathbb{E}_{a \sim \pi_{\text{imp}}(a|s)}[(Q(s,a) - V(s))^2]$$
\end{proof}

\section{Additional Derivations}
\label{additionalderivations}
\subsection{Proof of Optimal Solution for the Exponential Objective}

\begin{proof} 
We can rewrite the objective to remove irrelevent terms
$$\argmin_{V(s)}\mathbb{E}_{a \sim \mu(a|s)}[\exp(\alpha(Q(s,a) - V(s))) - \alpha (Q(s,a) - V(s))]$$
$$ = \argmin_{V(s)}\mathbb{E}_{a \sim \mu(a|s)}[\exp(\alpha(Q(s,a) - V(s))) + \alpha V(s)]$$
We expand the expectation for the objective. Assume $\mu(a|s) > 0$
$$ = \argmin_{V(s)} \sum_a \mu(a|s) \bigg( \exp(\alpha(Q(s,a) - V(s))) + \alpha V(s)\bigg)$$
Note that the objective function is convex with respect to $V(s)$.
$$0 = \left. \frac{\partial}{\partial V(s)}\right|_{V=V^*} \sum_a \mu(a|s) \bigg( \exp(\alpha(Q(s,a) - V(s))) + \alpha V(s)\bigg)$$ 
$$ = \sum_a \mu(a|s) \bigg( -\alpha \exp(\alpha(Q(s,a) - V^*(s))) + \alpha \bigg)$$

$$
      = \alpha \sum_a \mu(a|s) \bigg( -\exp(\alpha(Q(s,a) - V^*(s))) \bigg) + \alpha \sum_a \mu(a|s) 
$$
$$
      = \alpha \sum_a \mu(a|s) \bigg( -\exp(\alpha(Q(s,a) - V^*(s))) \bigg) + \alpha
$$
$$
      = \alpha \exp(-\alpha V^*(s)) \sum_a \bigg( -\exp(\alpha Q(s,a) + \log \mu(a|s)) \bigg) + \alpha
$$

We can now easily solve for $V^*(s)$. 
$$1 = \exp(-\alpha V^*(s)) \sum_a \bigg(\exp(\alpha Q(s,a) + \log \mu(a|s))\bigg)$$
$$\frac{1}{\sum_a \bigg( \exp(\alpha Q(s,a) + \log \mu(a|s))\bigg)} = \exp(-\alpha V^*(s))$$
$$\frac{1}{\alpha} \log \sum_a \bigg( \exp(\alpha Q(s,a) + \log \mu(a|s))\bigg) = V^*(s) = V_{\exp}(s)$$
\end{proof}

\subsection{Kullback–Leibler Divergence between Exponential Implicit Policy and behavior Distribution}
We look to compute the KL divergence between the implicit policy of the exponential distribution and the behavior policy.
\begin{align*}
D_{KL}(\mu(a|s)\|\pi_{\exp}(a|s)) &= \sum_{a} \mu(a|s) \log \big( \frac{\mu(a|s)}{\pi_{\exp}(a|s)} \big) \\
&= \sum_{a} \mu(a|s) \log \big( \frac{\exp(\log \mu(a|s))}{\exp(\alpha (Q(s,a) - V(s)) + \log \mu(a|s))}\big)\\
&= \sum_{a} \mu(a|s) \log \big( \frac{1}{\exp(\alpha (Q(s,a) - V(s)))}\big)\\
&= \sum_{a} \mu(a|s) \alpha(V(s) - Q(s,a))\\
&= \mathbb{E}_{(s, a) \sim \mathcal{D}}[\alpha(V(s) - Q(s,a))]
\end{align*}

This shows that the divergence of the behavior policy with the implicit policy is related to the advantage as well as the temperature hyperparameter. 

\section{Experimental Details}
\label{experimentdetails}
Our implementation is based from the jaxrl repo~\citep{jaxrl} which uses the JAX~\citep{jax2018github} framework using Flax~\citep{flax2020github}. For computing results, we use Titan X GPUs from a university-provided cluster. The cluster contains 80+ GPUs. 
\subsection{Standard Offline RL}
For the standard offline RL benchmark, we train the critic with 1.5 million gradient updates and the diffusion behavior policy with 3 million gradient updates. We found critic learning to be slightly unstable with more updates than 2 million. All reported results ("-A") for Table~\ref{table:d4rl} are taken from Table 1 in \citet{kostrikov2021offline} except for EQL, SfBC, and DQL which are taken directly from their main offline RL results table. As in IQL, we standardize the rewards for the locomotion tasks and we subtract rewards by one for the antmaze tasks. This is also done for fairness in the one hyperparameter reruns of prior methods ("-1"). For training time, we take numbers also from \citet{kostrikov2021offline} and from \citet{chen2022offline}.

\subsection{One-Hyperparameter Offline RL}
For the one hyper parameter experiment, we re-implemented IQL from the \href{https://github.com\/ikostrikov/implicit_q_learning}{IQL repo.}, we reran CQL from the \href{https://github.com/young-geng/JaxCQL}{CQL repo.}, and we reran DiffusionQL from the \href{https://github.com/Zhendong-Wang/Diffusion-Policies-for-Offline-RL}{DQL repo}. We only sweep over the main-hyper parameter mention in the paper and select the best performing one per domain; all other hyper parameters are left constant. For CQL, we sweep over $\lambda \in \{1.0, 2.0, 5.0, 10.0\}$ or the weight of the CQL term, for DQL we swept over the $\eta \in \{1.0, 2.0, 2.5, 3.0\}$ or weight of the Q maximization objective, and for IQL and IDQL, we swept over the expectile $\tau \in \{0.6, 0.7, 0.8, 0.9\}$. For each domain (locomotion and antmaze), we select the best-performing hyperparameter (e.g. $\tau = 0.7$ for locomotion and $\tau=0.9$ for antmaze).

While this may not fairly represent each algorithm, our main intention was to show how our method can perform well out of the box (i.e. the posted GitHub implementation) without the need for excessive tuning. We find that many non-IQL methods tend to overtune their ant maze results to include many more changes than used in locomotion tasks. For example, CQL requires many changes from their locomotion configuration \href{https://github.com/young-geng/JaxCQL/issues/4}{linked here} to get strong antmaze results.
While not bad per se, it does indicate that these methods require more careful tuning to work well across domains. Furthermore, some methods do not fairly present their results as they either tune their method per environment or take the max of their evaluation curve. As a result, another intention of ours was to clearly state a protocol to fairly compare algorithms side by side.

\subsection{Finetuning}
For fine-tuning, we pretrain the critic for 1 million steps and the diffusion BC actor for 2 million steps. During online finetuning, we take gradient step on the critic for each environment step. If the actor is fine-tuned as well, we take 2 gradient steps per environment step. We found fine-tuning failed to improve from pre-training on the adroit tasks, but this is not unexpected. The adroit fine-tuning tasks require significant exploration to achieve strong returns, so having behavior regularization can be harmful. We leave this as future work.

\section{Table Hyperparameters}

The critic and value network follow the same parameterization as in IQL~\citep{kostrikov2021offline} (2 Layer MLP with hidden size 256 and ReLU activations). Other details mentioned about architecture are for the diffusion model.

\begin{center}
\def\arraystretch{1.35}
\begin{tabular}{|l|c|} 
\hline
\textbf{LR} (For all networks)  & 3e-4 \\
\textbf{Critc Batch Size}  & 256 \\
\textbf{Actor Batch Size}  & 1024 \\
\textbf{$\tau$ Expectiles}  & $0.7$ (locomotion), $0.9$ (antmaze)\\
\textbf{$\tau$ Quantiles}  & $0.6$ (locomotion), $0.8$ (antmaze)\\
\textbf{$\alpha$ Exponential}  & $1.0$ (locomotion), $0.5$ (antmaze)\\
\textbf{Critic Grad Steps} & 1.5e6 ("-A"), 1e6 ("-1") \\
\textbf{Actor Grad Steps}  & 3e6 ("-A"), 2e6 ("-1") \\
\textbf{Critic Pre-Training Steps} & 1e6 (Figure~\ref{table:fullfinetune}) \\
\textbf{Actor Pre-Training Steps}  &  2e6 (Figure~\ref{table:fullfinetune}) \\
\textbf{Target Critic EMA} & 0.005 \\
\textbf{T} & 5 \\
\textbf{N} & 32 (antmaze "-A"), 128 (loco "-A"), 64 ("-1") \\
\textbf{Beta schedule} & \text{Variance Preserving \citep{song2020score}} \\
\textbf{Dropout Rate} \citep{NIPS2015_8d55a249} & 0.1 \\
\textbf{Number Residual Blocks} & 3 \\
\textbf{Actor Cosine Decay} \citep{loshchilov2016sgdr} & Number of Actor Grad Steps\\
\textbf{Optimizer} & Adam \citep{kingma2014adam}\\

\hline
\end{tabular}
\end{center}

\section{Additional Experiments and Results}
\label{additionalablations}
\subsection{Other Prior Offline RL Work Results}
We compare our method to a number of recent offline RL algorithms discussed in the related work section: \%BC, Decision Transformers (DT) \citep{chen2021decision}, TD3+BC (TD3) \citep{fujimoto2021minimalist}, conservative Q-learning (CQL) \citep{kumar2020conservative}, implicit Q-learning (IQL) \citep{kostrikov2021offline}, extreme Q-learning (EQL) \citep{garg2023extreme}, and Selecting from Behavior Candidates(SfBC) \citep{chen2022offline}. We don't include all comparisons in the main paper to save space and since our "-1" method outperforms all the reported methods other than DQL. Some of these prior works tune the hyperparameters coarsely, for example, \citet{wang2022diffusion} uses per-task tuning on both locomotion and antmaze tasks and \citet{garg2023extreme} takes the max evaluation during training. This generally leads to better performance but implies an assumption of being able to do per-task online hyperparameter selection. Results are reported in Table~\ref{table:fulld4rl}. IDQL remains competitive with all prior methods on the locomotion tasks and outperforms all prior methods on the antmaze tasks. 
\begin{table*}
\small
\begin{tabular*}{\textwidth}{@{\extracolsep{\fill}}lrrrrrrrr|r}
\multicolumn{1}{c}{\bf Dataset} & \multicolumn{1}{c}{\bf \%BC} & \multicolumn{1}{c}{\bf DT} & \multicolumn{1}{c}{\bf TD3} & \multicolumn{1}{c}{\bf CQL} & \multicolumn{1}{c}{\bf IQL} & \multicolumn{1}{c}{\bf EQL} & \multicolumn{1}{c}{\bf SfBC} &  \multicolumn{1}{c}{\bf DQL} & \multicolumn{1}{c}{\bf IDQL}  \\ 
\shline
halfcheetah-m & 48.4 & 42.6& 48.3 &44.0 &47.4 & 47.7 & 45.9 & \textbf{51.1} & \textbf{51.0}\\ 
hopper-m &56.9 &67.6 &59.3 &58.5 & 66.3 & 71.1 & 57.1 & \textbf{90.5} & 65.4 \\
walker2d-m &75.0 & 74.0 & \textbf{83.7} &72.5 &78.3 & 77.9 & \textbf{81.5} & \textbf{87.0} & \textbf{82.5}\\
halfcheetah-mr &40.6 & 36.6 & 44.6 & \textbf{45.5} & 44.2 & 44.8 & 37.1 & \textbf{47.8} & \textbf{45.9}\\
hopper-mr  &75.9 & 82.7 &60.9 &\textbf{95.0} &\textbf{94.7}& \textbf{97.3}  & 86.2& \textbf{101.3} & \textbf{92.1}\\
walker2d-mr &62.5 & 66.6 & 81.8 &77.2 &73.9 & 75.9 & 65.1 & \textbf{95.5} & 85.1\\
halfcheetah-me &\textbf{92.9} &86.8 &\textbf{90.7} &\textbf{91.6} &86.7 & 89.8 & \textbf{92.6} & \textbf{96.8} & \textbf{95.9}\\
hopper-me &\textbf{110.9} & \textbf{110.9}& 98.0 &\textbf{105.4} &91.5 & \textbf{107.1} & \textbf{108.6} & \textbf{111.1} & \textbf{108.6} \\
walker2d-me &\textbf{109.0} & \textbf{108.1} &\textbf{110.1} &\textbf{108.8} &\textbf{109.6} &  \textbf{109.6} & \textbf{109.8} & \textbf{110.1} & \textbf{112.7}\\ \shline
locomotion-v2 total & 666.2 & 672.6 & 677.4 & 698.5 & 692.4 &  \textbf{725.3} & 680.4 & \textbf{791.2} & \textbf{739.2}\\ \shline
antmaze-u &62.8 &59.2 &78.6 &74.0 & 87.5& 87.2 & \textbf{92.0} & \textbf{93.4} & \textbf{94.0}\\
antmaze-ud &50.2 &53.0 &71.4 &\textbf{84.0} &62.2 & 69.2 & \textbf{85.3} & 66.2 & \textbf{80.2}\\
antmaze-mp &5.4 &0.0 &10.6 &61.2 & 71.2 & 73.5 & \textbf{81.3} & 76.6 & \textbf{84.2}\\
antmaze-md &9.8 &0.0 &3.0 &53.7 & 70.0& 71.2 & \textbf{82.0} & 78.6 & \textbf{84.8} \\
antmaze-lp &0.0 &0.0 &0.2 &15.8 & 39.6 & 41 & \textbf{59.3} & 46.4 & \textbf{63.5}\\
antmaze-ld &6.0 &0.0 &0.0 &14.9 & 47.5 & 47.3 & 45.5 & 57.3 & \textbf{67.9} \\ 
\shline
antmaze-v0 total &134.2 & 112.2 &163.8 &303.6 & 378.0 & 386 & \textbf{445.2} & 418.5 & \textbf{474.6} \\ 
\shline
total &800.4 & 784.8 &841.2 &1002.1 & 1070.4 & 1111.3 & 1125.6 & \textbf{1209.7} & \textbf{1213.8} \\
\shline
training time &10m &960m & 20m & 80m & 20m & 20m & 785m & 240m & 60m
\end{tabular*}
\vspace{.1cm}
\caption{
\textbf{Standard evaluation for offline RL.} IDQL performs on par or better than other SOTA offline RL methods. Algorithm names are shortened to save space. Results for our method are averaged over 10 seeds.
}
\label{table:fulld4rl}
\end{table*}

\subsection{Maze2D Results}
IQL performs quite poorly in the maze2d environment compare to other offline RL methods. Specifically, Diffuser~\citep{janner2022planning} performs very strongly when compared to IQL and utilizes diffusion. We look to see if our method aids the performance of IQL. Results are in Table~\ref{table:maze2d}. Our method outperforms IQL consistently but comes short against Diffuser. We suspect that model based planning generalizes better in the Maze2d environments.

\begin{table*}
\centering
\small
\begin{tabular*}{\textwidth}{@{\extracolsep{\fill}}lrr|r}
\multicolumn{1}{c}{\bf Dataset}  & \multicolumn{1}{c}{\bf IQL} & \multicolumn{1}{c}{\bf Diffuser}  & \multicolumn{1}{c}{\bf IDQL}\\ 
\shline
    umaze & $47.4$ & $\textbf{113.9}$ & $57.9$\\
    medium  & $34.9$ & $\textbf{121.5}$ & $89.5$ \\
    large & $58.6$ & $\textbf{123.0}$ & $90.1$\\
    \shline
    total  & $94$ & $\textbf{358.5}$ & $237.5$ \\
    
\end{tabular*}
\vspace{.1cm}
\caption{
\textbf{Maze2D results.} We include maze2d results compared to IQL and Diffuser. Our method outperforms IQL and remains close to diffusers performance.
}
\label{table:maze2d}
\end{table*}

\subsection{Diffusion Steps $T$ and Beta Schedule}
\begin{figure}
    \centering
    \includegraphics[height=3.0cm]{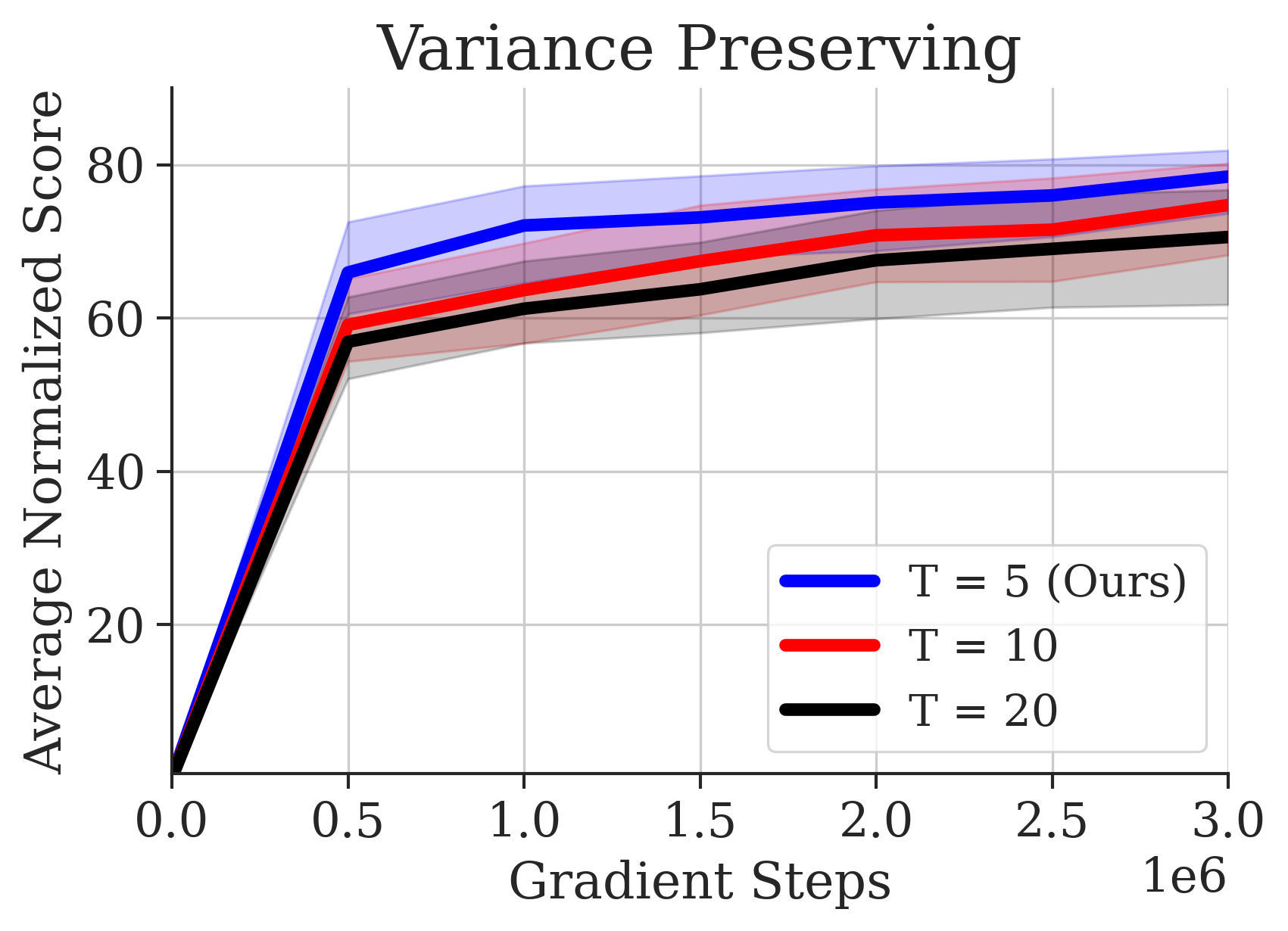} \includegraphics[height=3.0cm]{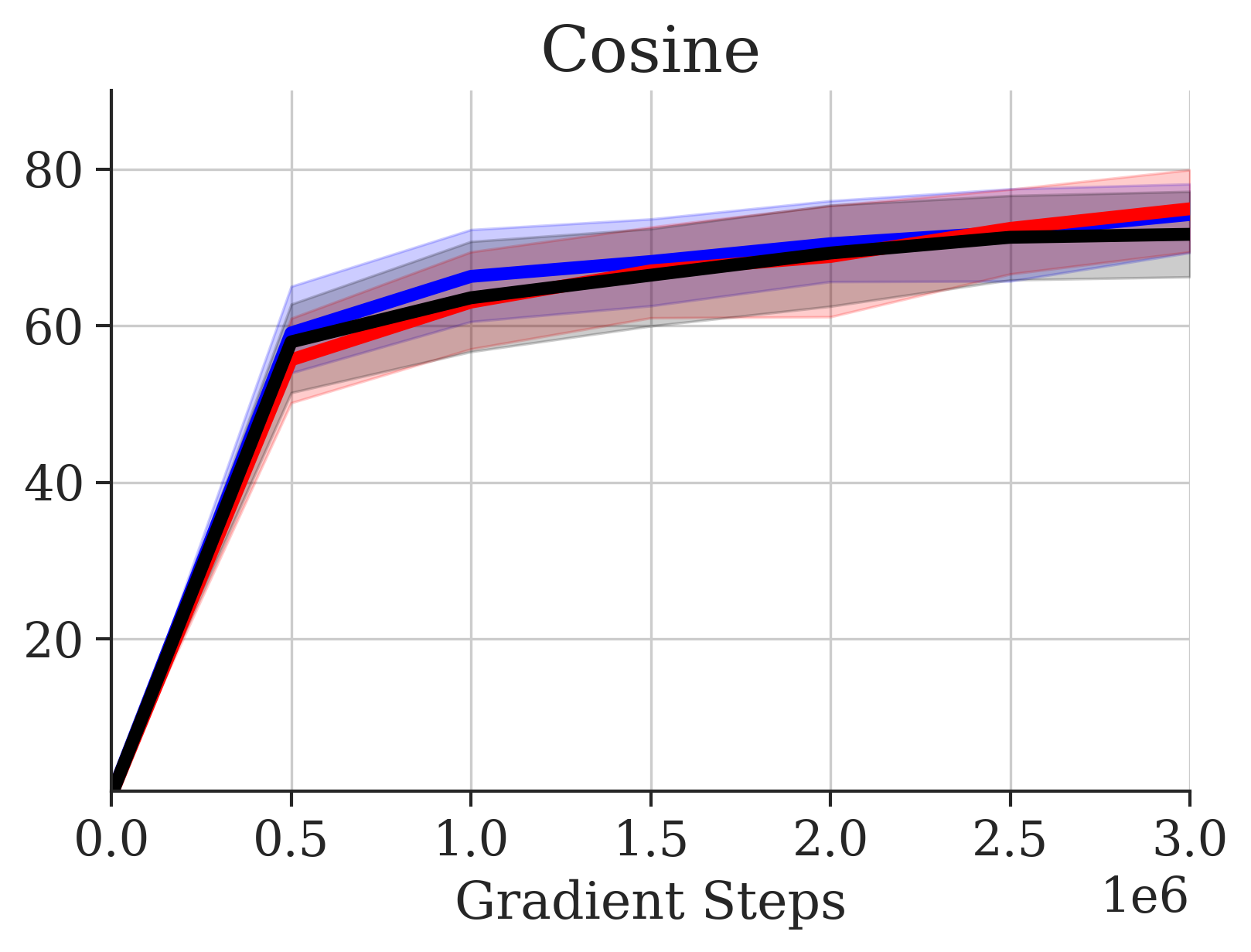}
    \includegraphics[height=3.0cm]{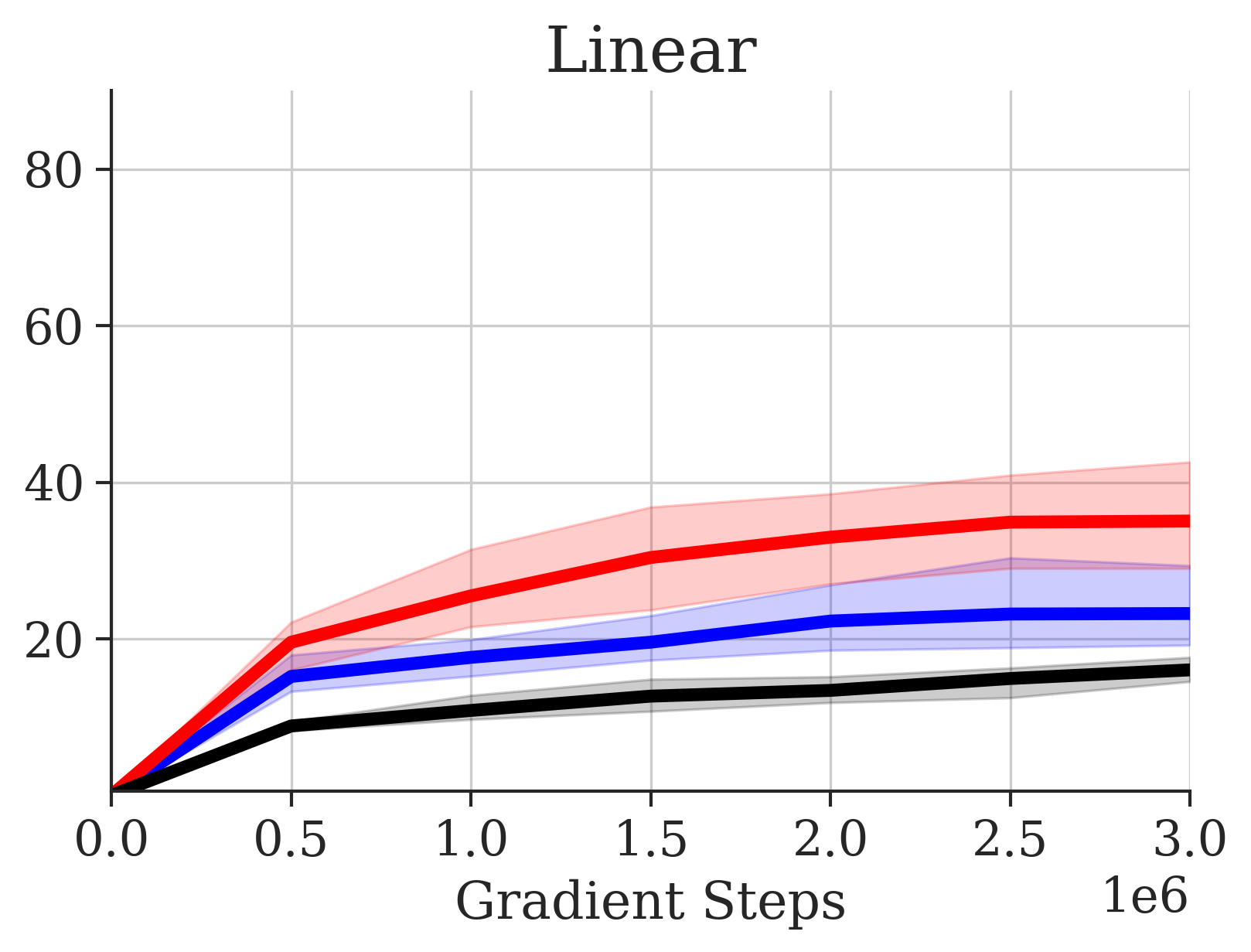}
    \caption{Sweeps over different schedules and $T$. The variance preserving schedule with low $T$ generally works the best.}
    \label{figure:beta_sweep}
\end{figure}
The choice of number of diffusion steps (T) was also pretty crucial to good performance on D4RL. Strangely, the T needed for quality samples of the simple 2D continuous datasets was much larger than required for D4RL. We found T = 50 to be best for the 2D datasets, but T = 5 to be best for D4RL. This might be because control datasets are more deterministic than the continuous datasets presented (which include lots of added noise on top of a simple signal), but this remains an open question. We include ablations of various T on D4RL in Figure~\ref{figure:beta_sweep}. Increasing T has a small but negative effect on performance. One thing to note is that for T > 20, evaluation is quite slow due to needing to resample the entire chain every step of the Markov process.  

Also, the choice of beta schedule has a strong impact in performance. We ablate over a linear \citep{ho2020denoising}, cosine \citep{nichol2021improved}, and variance preserving schedule \citep{song2020score} in Figure~\ref{figure:beta_sweep}. We found the vp to work the best for D4RL, but cosine also worked very well with small T. The linear schedule required far too large of T to perform well and ended up being a poor choice for D4RL tasks. We suspect that the signal to noise ratio induced by the schedules is very important in proper expression. In particular, the first noise step is a crucial part of the noising process. In summary, we recommend to sweep over different T and schedule for the best possible modeling of the action space.

\subsection{Batch Size}
\begin{figure}
    \centering
    \includegraphics[height=3.5cm]{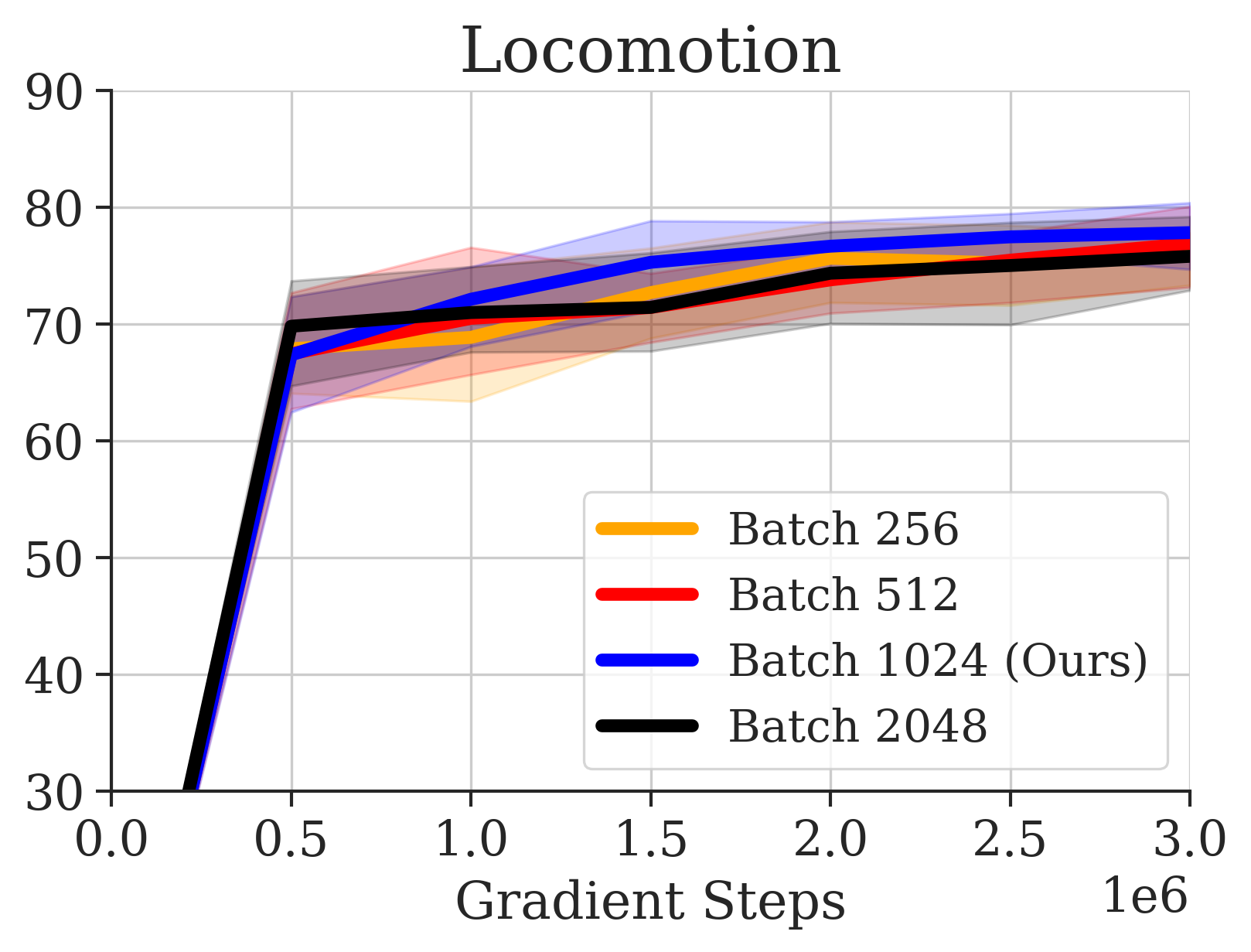} \includegraphics[height=3.5cm]{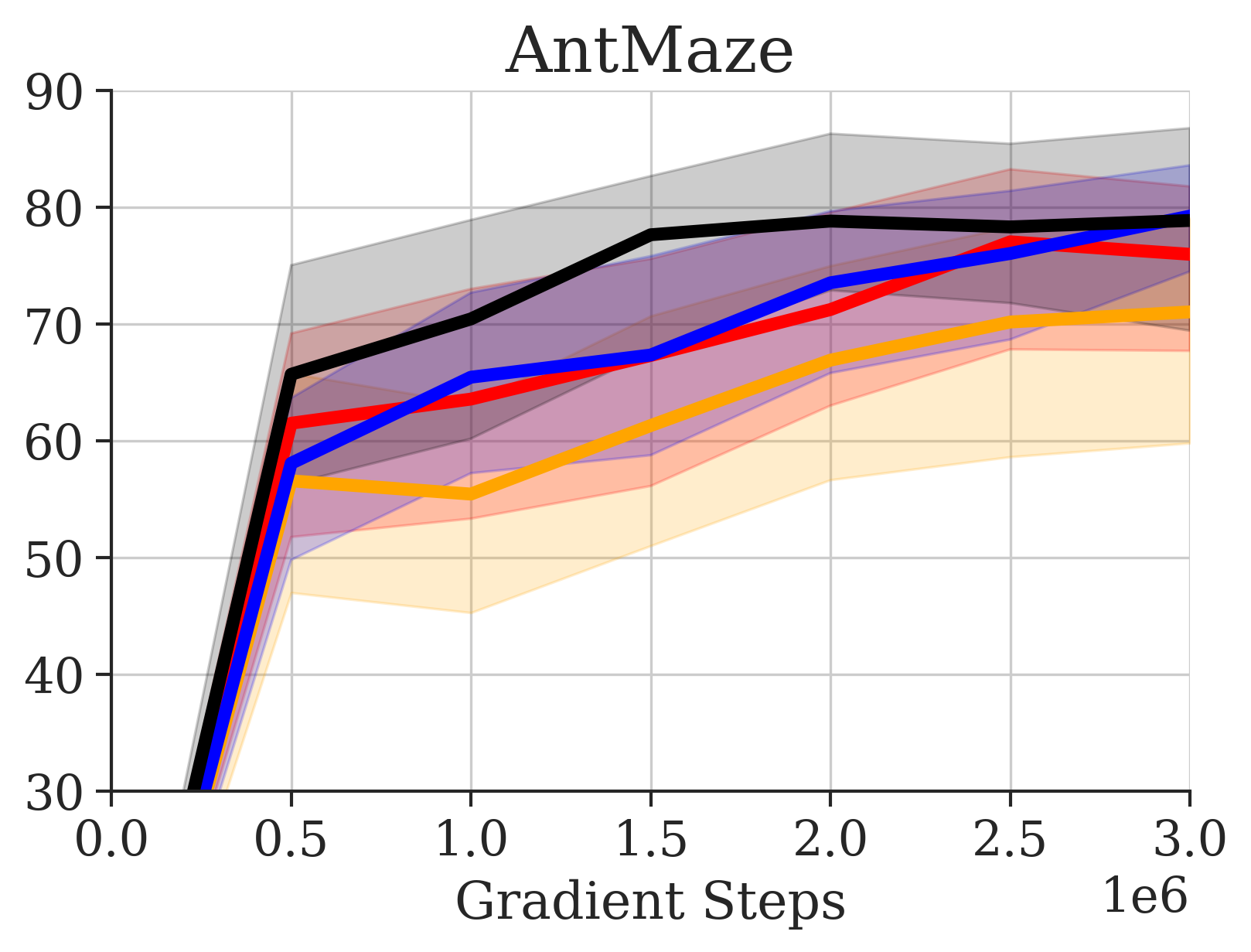} 
    \caption{Batch size ablation. A higher batch size is important for the antmaze tasks, but not important for the locomotion tasks.}
    \label{figure:batch_sweep}
\end{figure}
As discussed in Section~\ref{section:issuesdiff}, batch size was important for reducing outliers samples. We ablate over different batch sizes on the DDPM loss in the D4RL benchmarks. Results are in Figure~\ref{figure:batch_sweep}. While batch size has a small effect on the locomotion tasks, having a larger batch size leads to stronger performance on the antmaze tasks. Though, the batch size is not as crucial as the choice of architecture or capacity of the model.

\subsection{AWR weighted DDPM Ablation.}
To show that our method performs better than directly training a policy, we ablate using AWR to weigh the DDPM loss over a batch. We try two versions: in one we sample the policy only once and in the other we sample $N = 64$ times and take the sample that maximizes the critic. Results are in Figure~\ref{figure:awr_ablation}. In the case of the one sample method, the AWR weighted objective usually performs worse than IDQL overall and at best performs on par with the multi-sample setup. As a result, AWR is not necessary for achieving strong performance and therefore adds an unnecessary training parameter. 
\begin{figure}
    
    \centering
    \includegraphics[height=3.5cm]{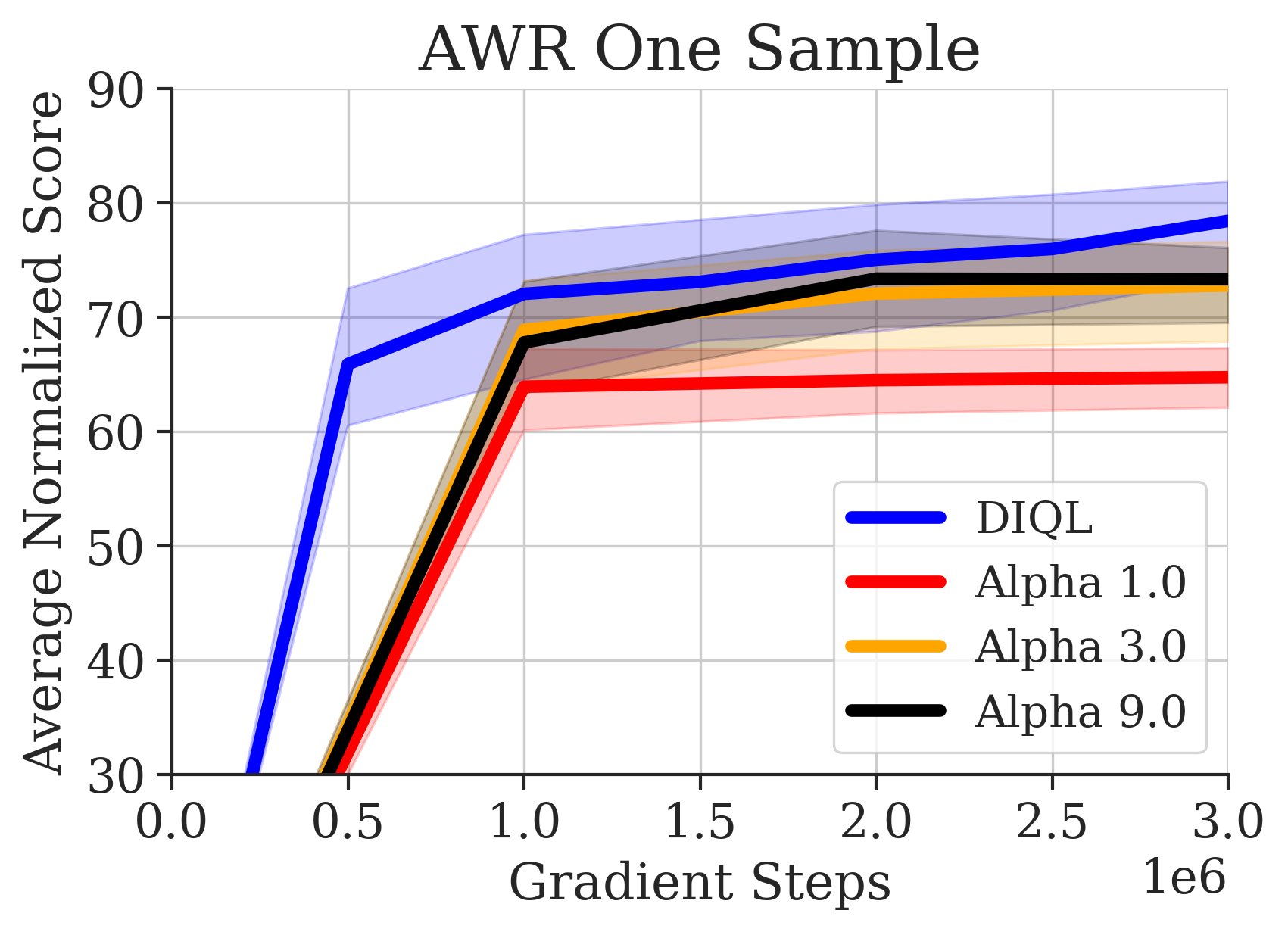} \includegraphics[height=3.5cm]{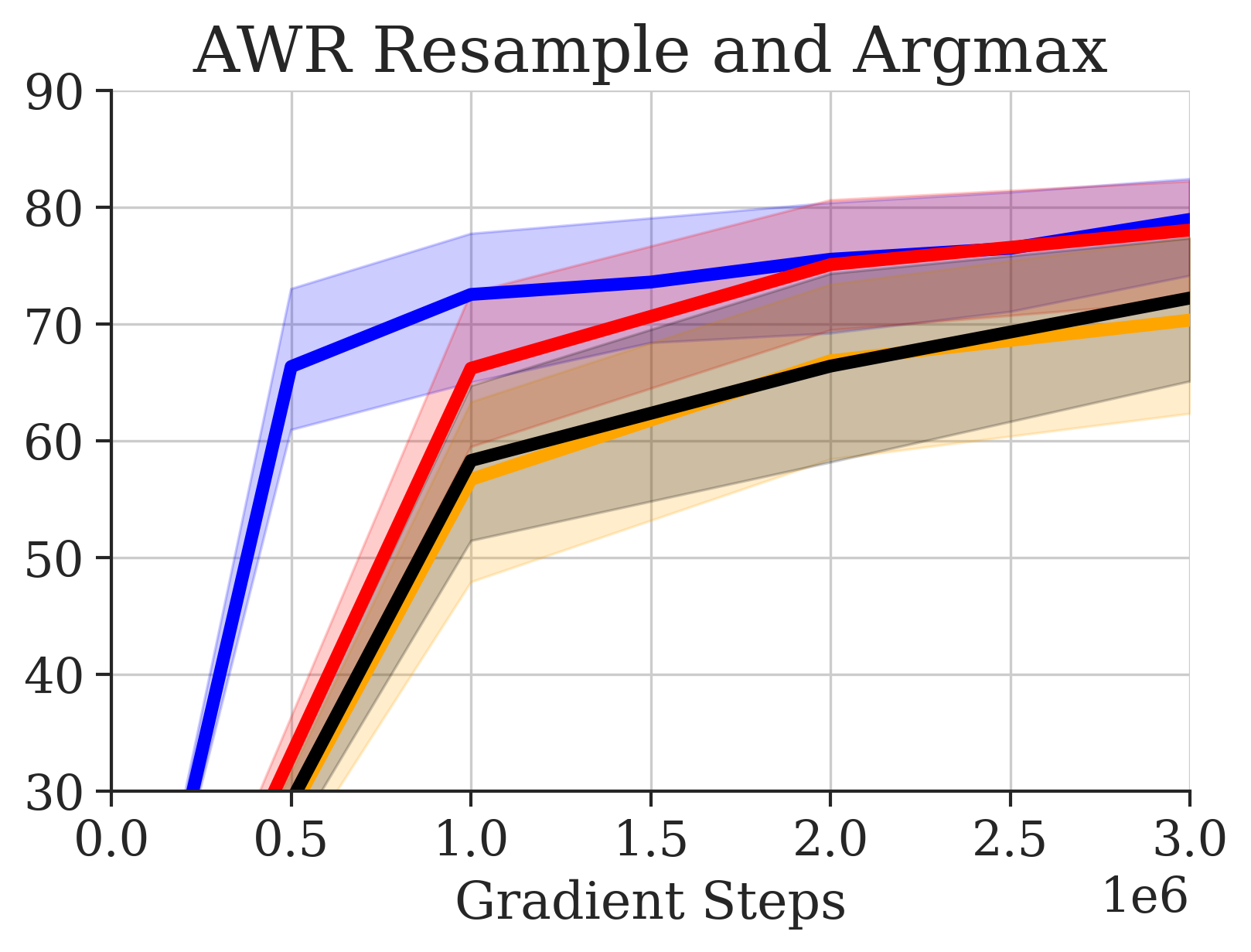}
    \caption{Weighted DDPM loss ablations using AWR. We try both one sample (Left) and multiple samples plus argmax (Right). Results are averaged over 10 seeds. }
    \label{figure:awr_ablation}
\end{figure}

\subsection{Other Important Architecture Details}
\begin{figure}
    
    \centering
    \includegraphics[height=4.0cm]{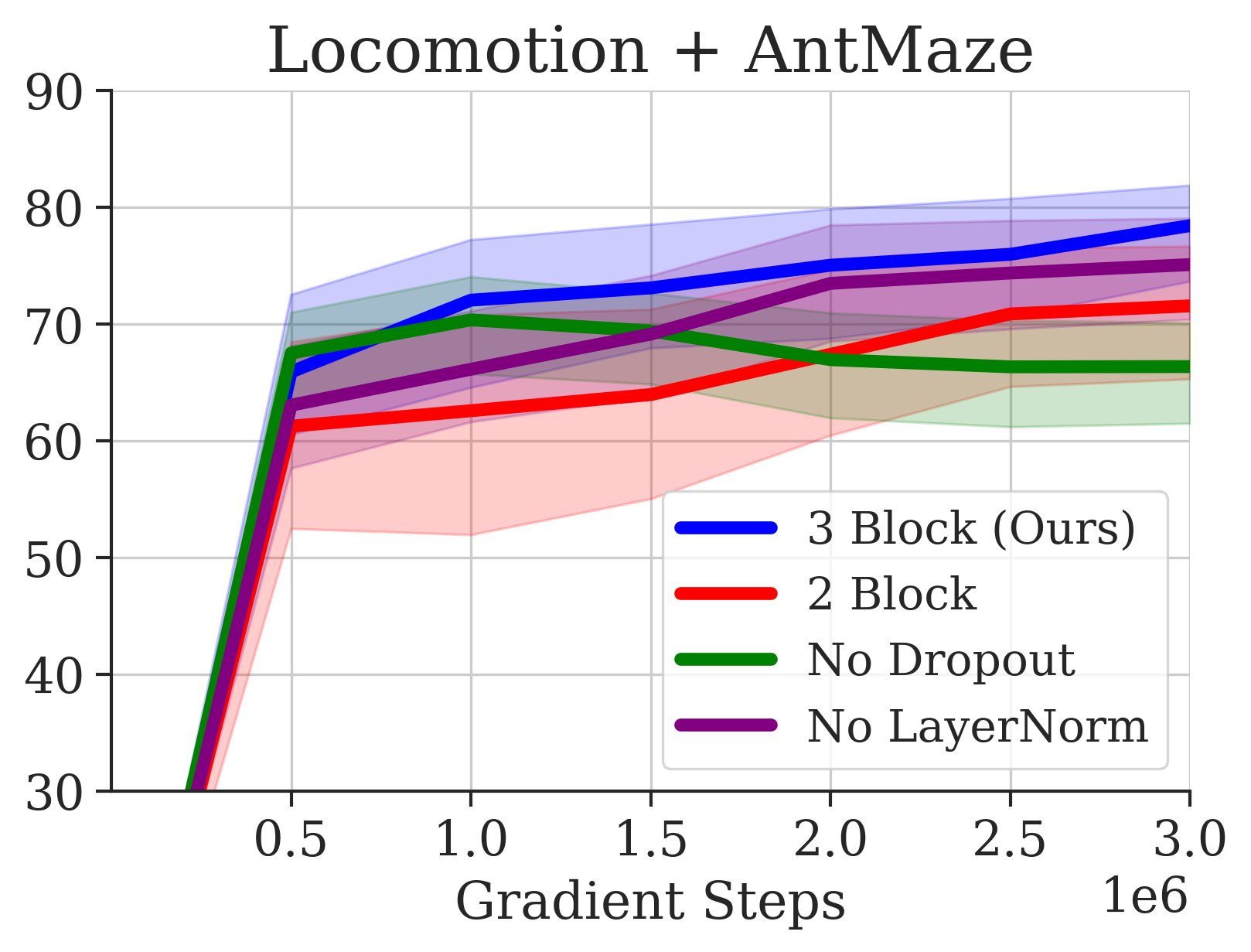}
    \caption{Ablation of capacity, dropout, and layer norm on offline results. }
    \label{figure:details}
\end{figure}

There are other important architecture details for strong performance on D4RL. We ablate over capacity (number of resent blocks), layer norm, and dropout to see the effect in Figure~\ref{figure:details}. As shown in Section~\ref{section:issuesdiff}, having a larger capacity and well regularized network is crucial for strong performance. Dropout has a large benefit in reducing overfitting and layer norm has a small benefit in overall performance.

\subsection{Training Curves for Online Finetuning}
To compare the sample efficiency and effectiveness of our policy extraction method, we compare IDQL to IQL in antmaze-large environments in Figure~\ref{figure:large_fine}. The training curves show a performance gain and sample efficiency gain over IQL, indicating the strength of IDQL in finetuning.

\begin{figure}
    \centering
    \includegraphics[height=3.5cm]{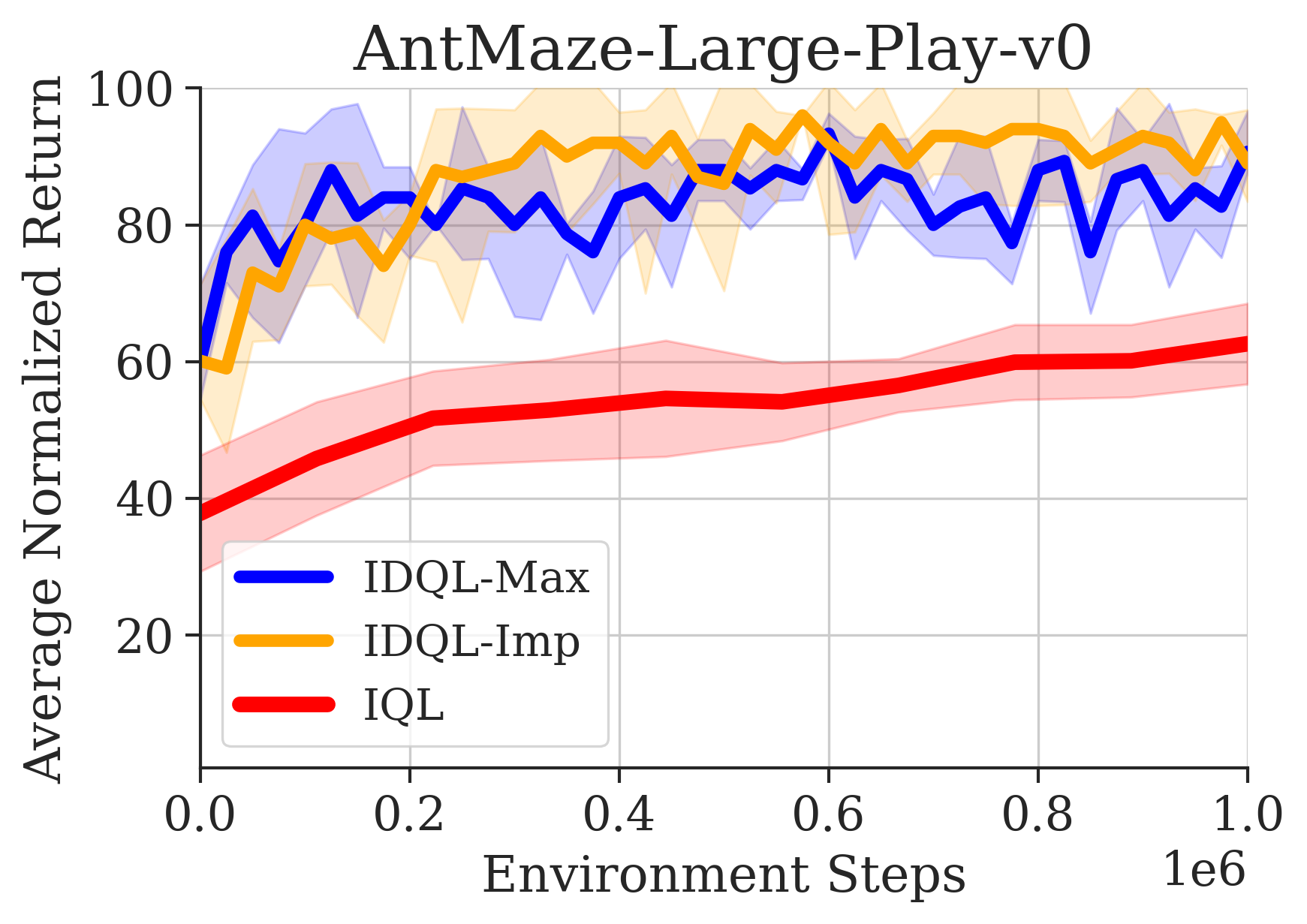} \includegraphics[height=3.5cm]{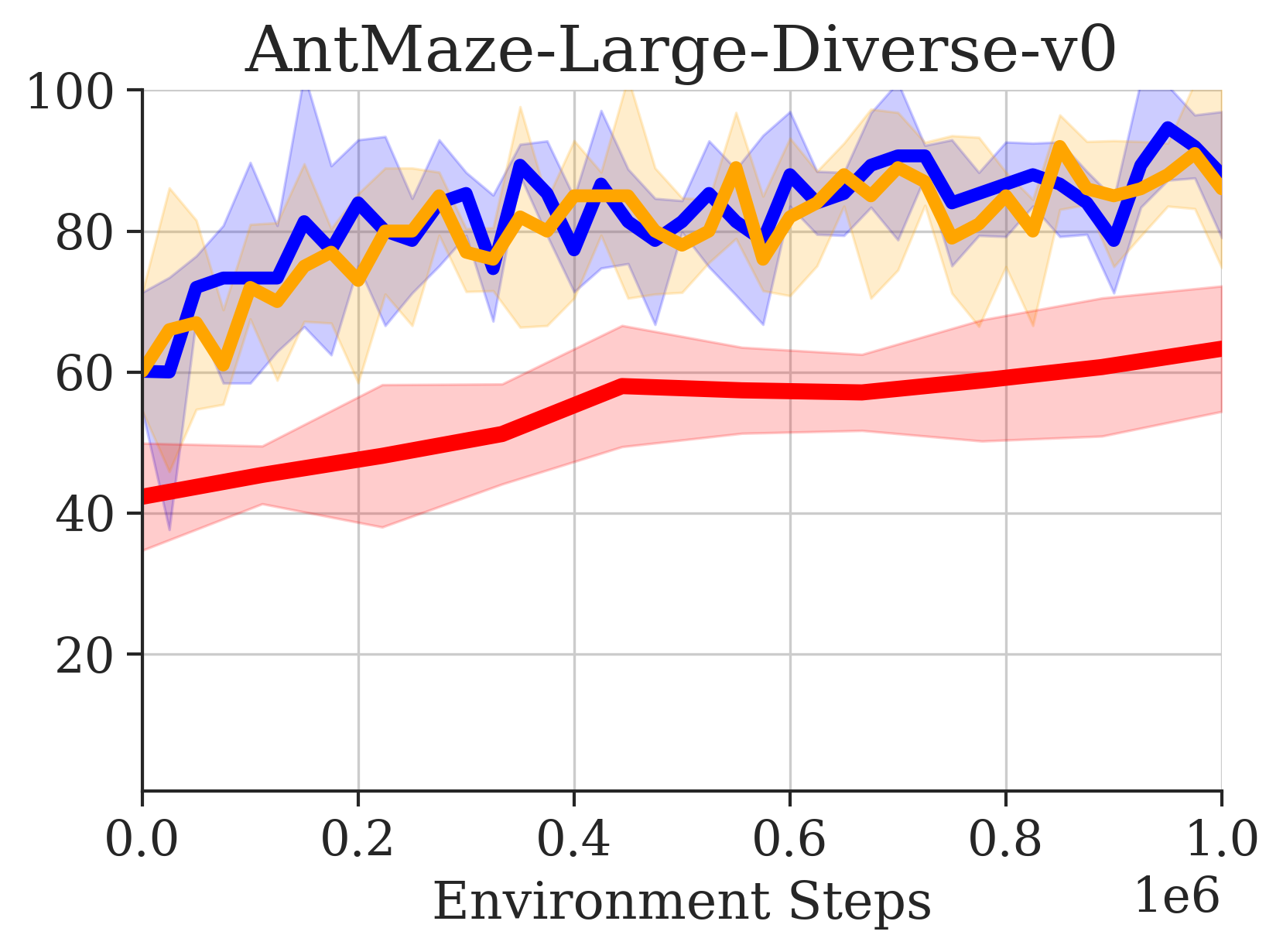}
    \caption{\textbf{Online finetuning results for antmaze-large tasks.} Finetuning training curves for various sampling and finetuning strategies for our method on antmaze-large tasks. The frozen actor requires only 100k samples to reach peak performance.}
    \label{figure:large_fine}
\end{figure}

\subsection{Full Learning Curves}
Figure~\ref{figure:locotrainingcurves} and Figure~\ref{figure:antmazetrainingcurves} are the learning curves for our one hyperparameter ("-1") variant of IDQL.

\begin{figure}
    \centering
    \includegraphics[height=3.3cm]{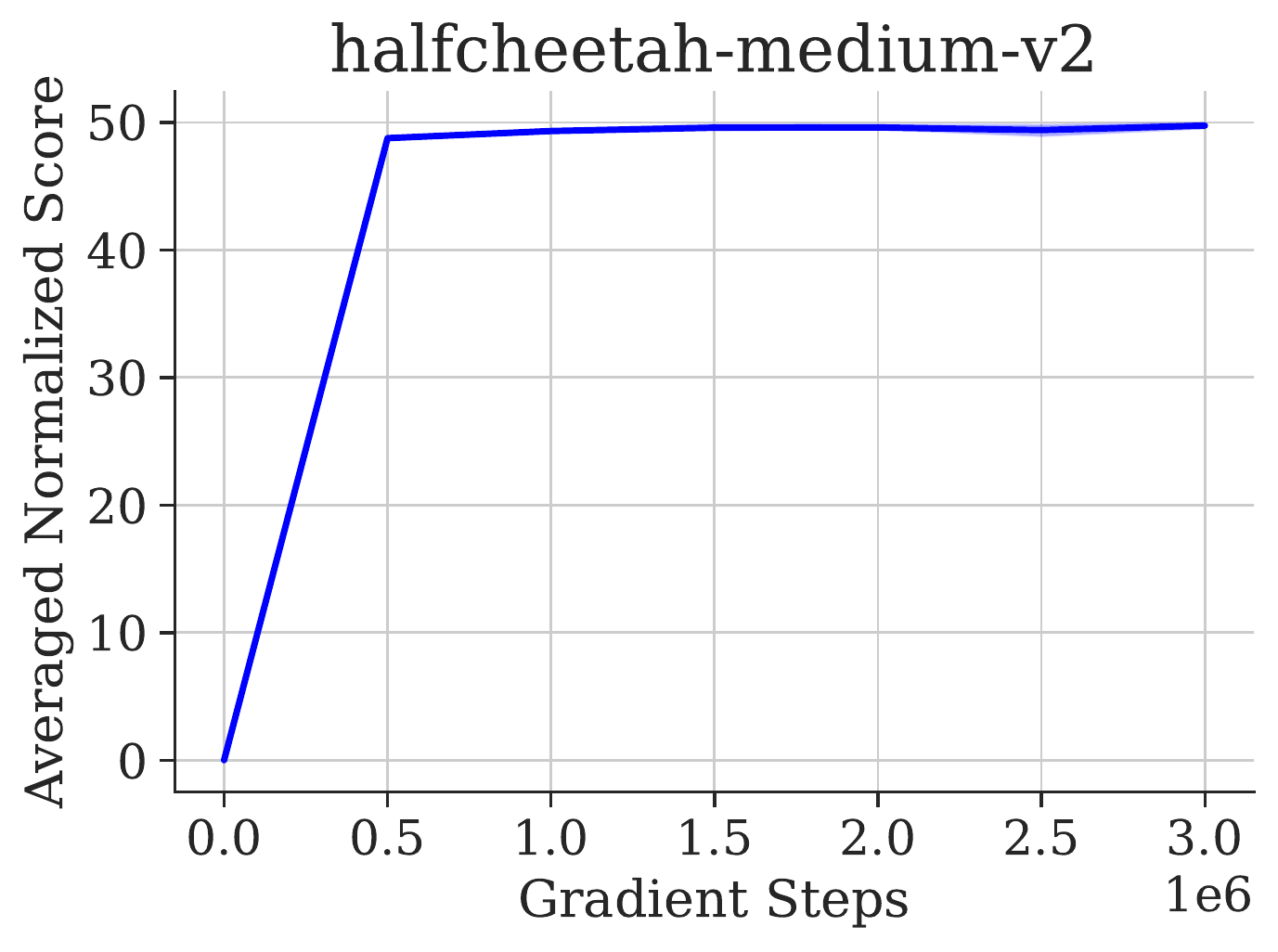} \includegraphics[height=3.3cm]{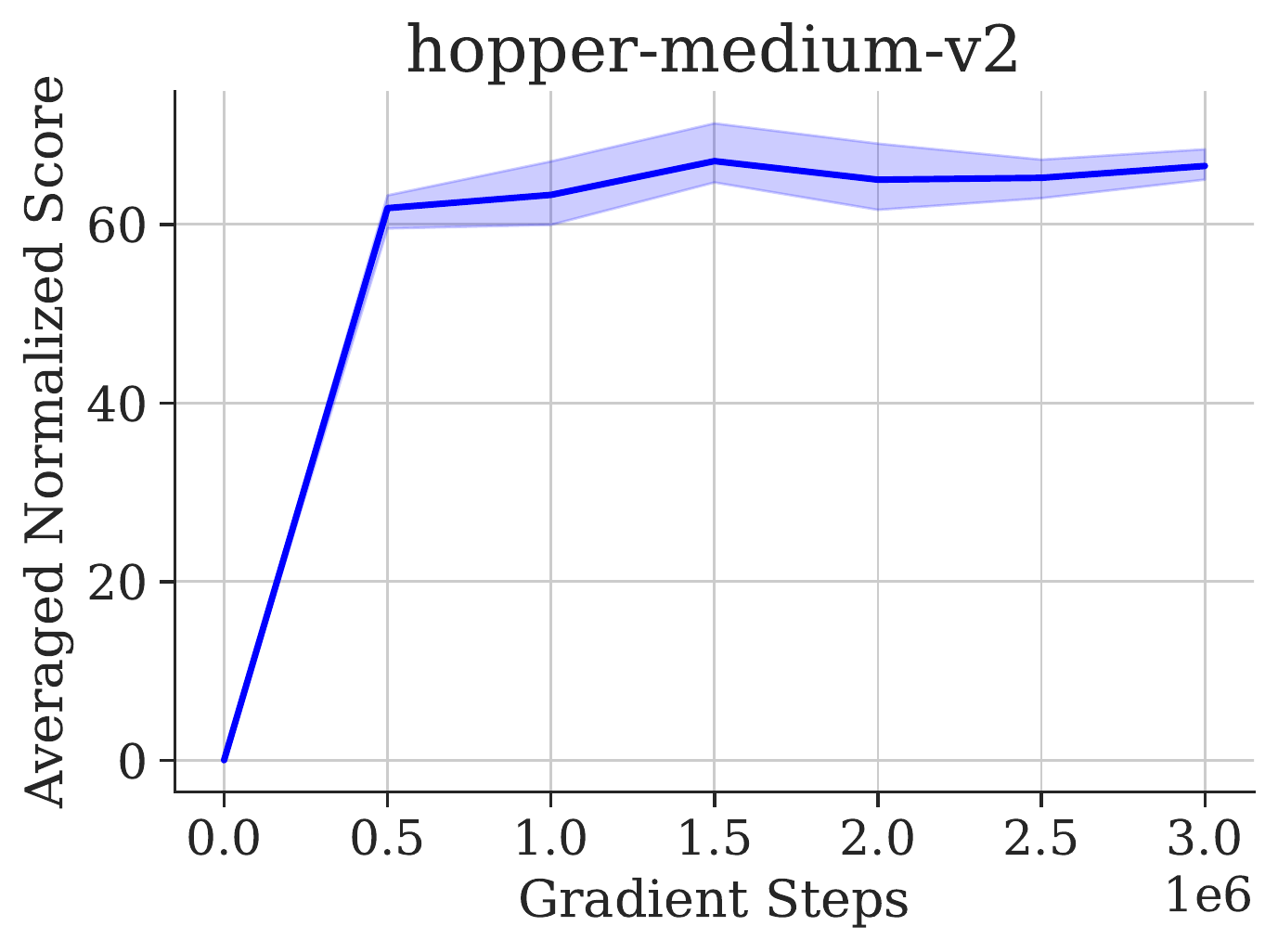}
    \includegraphics[height=3.3cm]{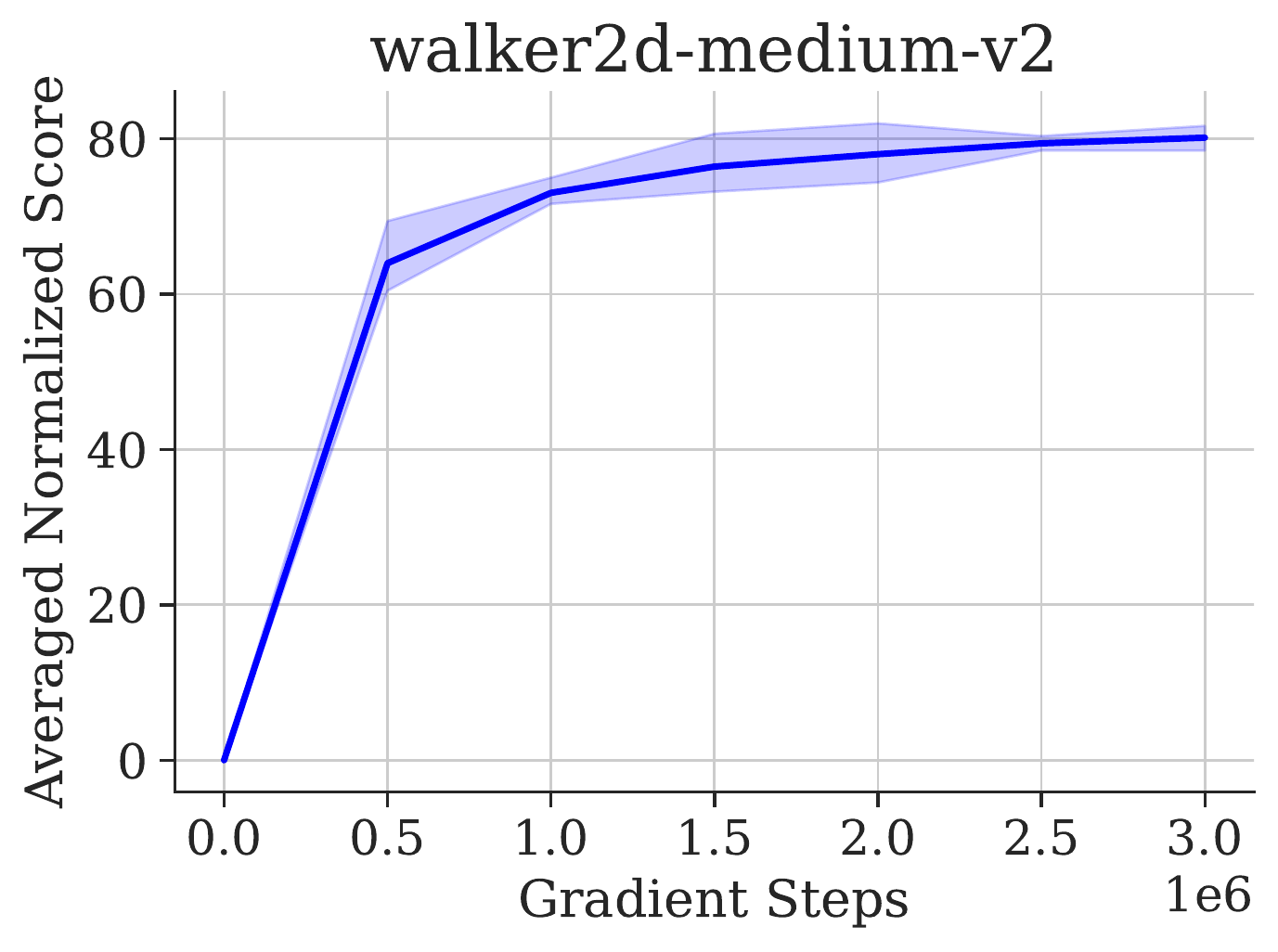} 

    \includegraphics[height=3.3cm]{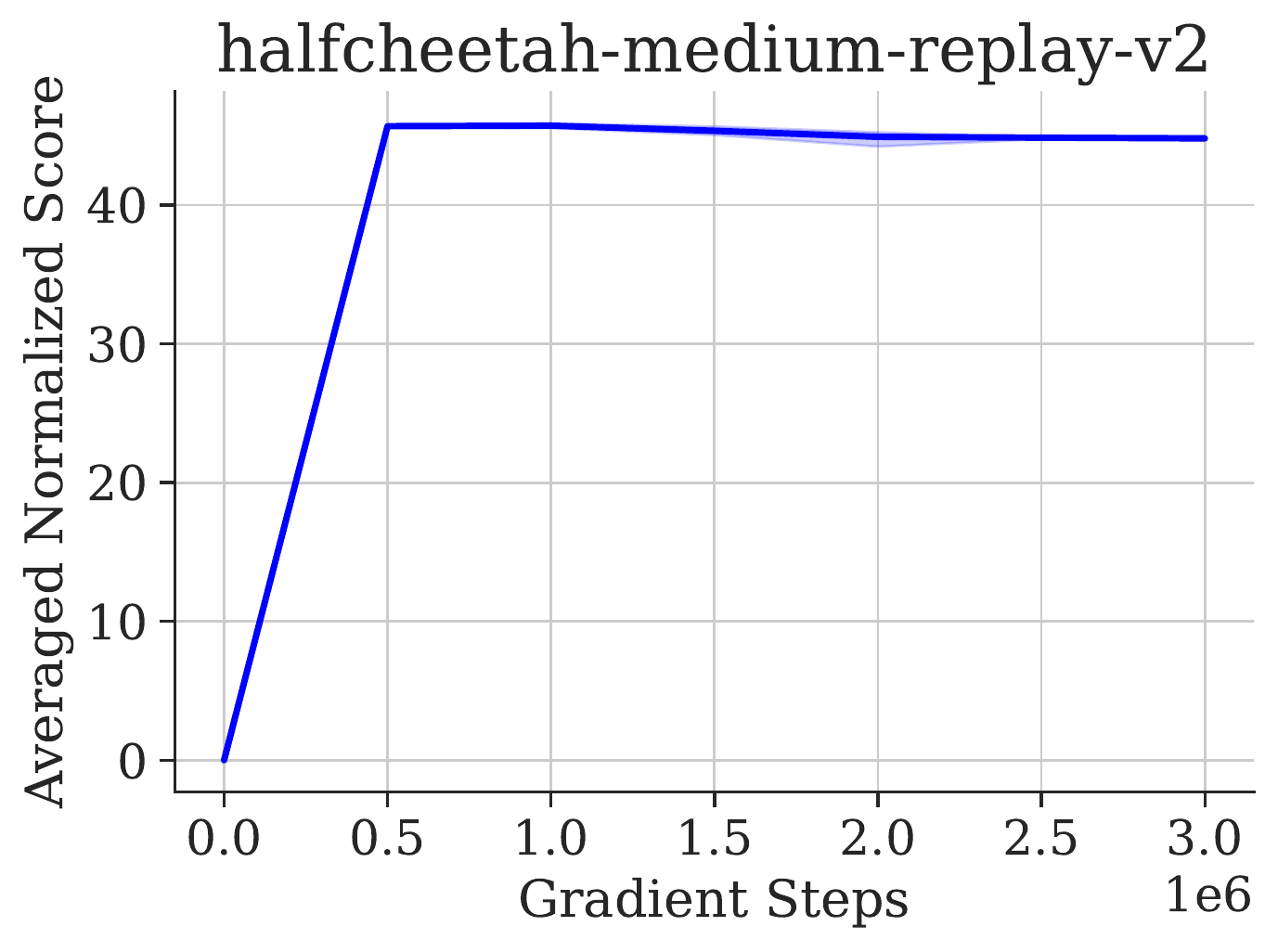} \includegraphics[height=3.3cm]{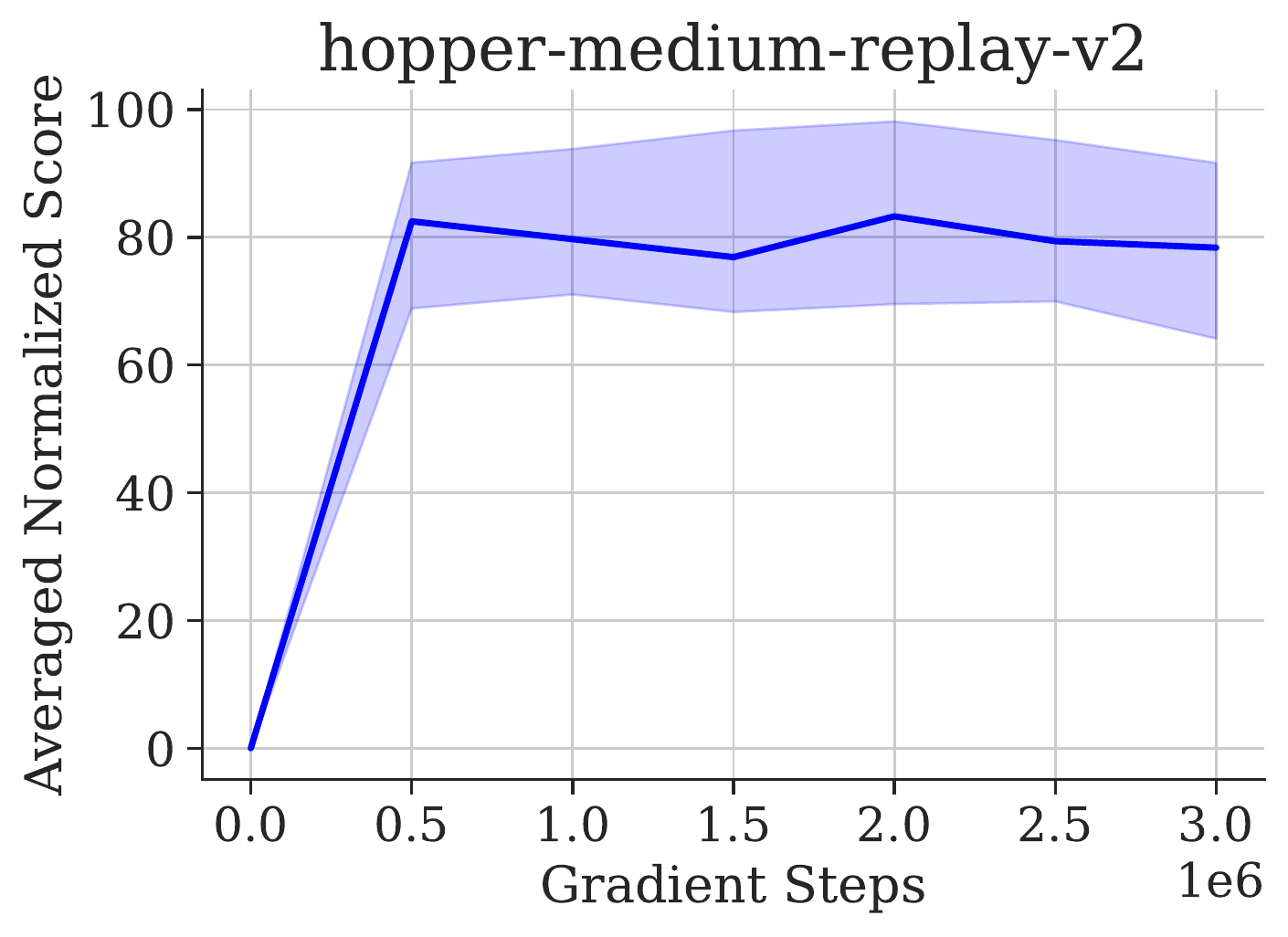}
    \includegraphics[height=3.3cm]{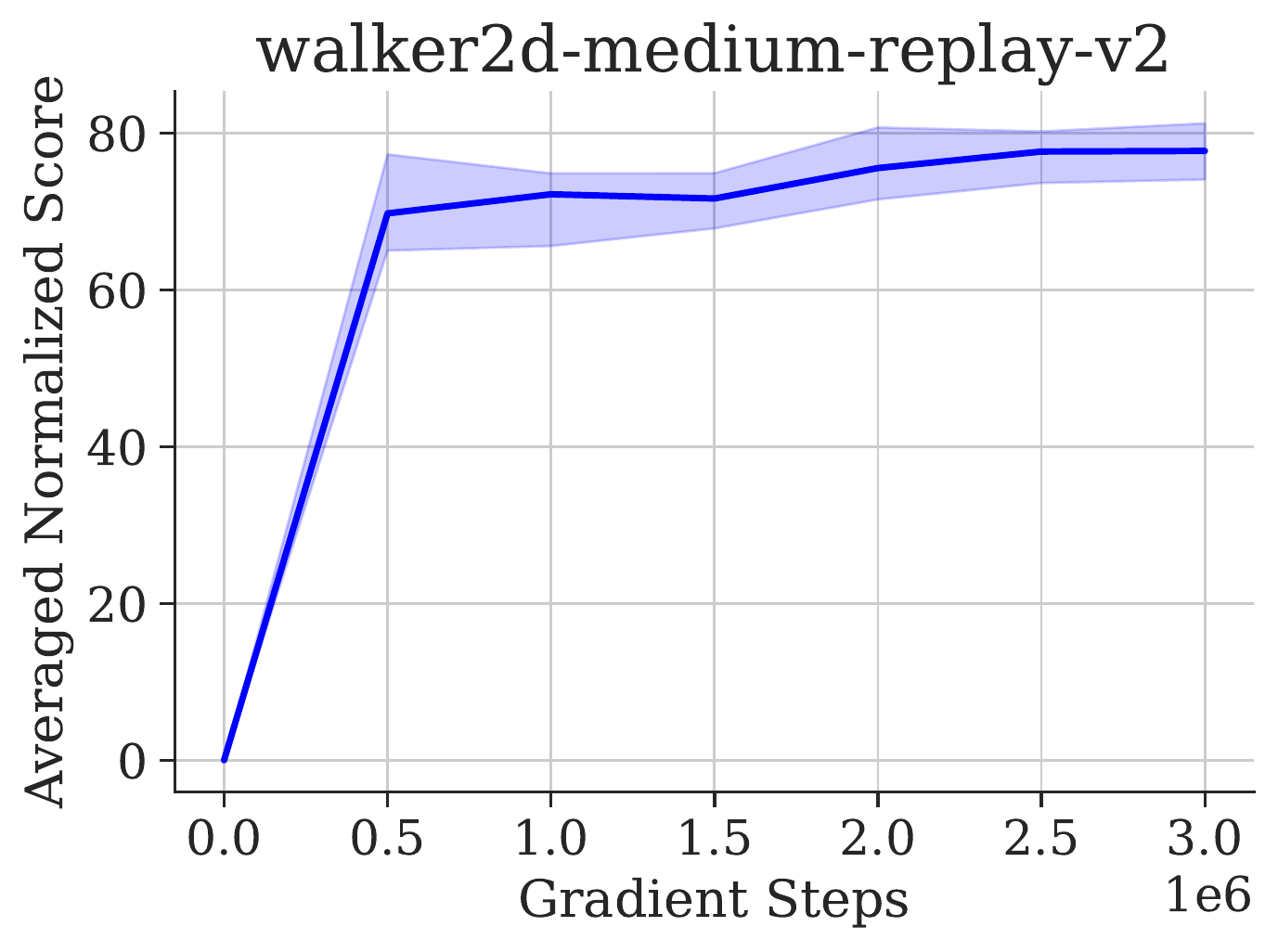} 

    \includegraphics[height=3.3cm]{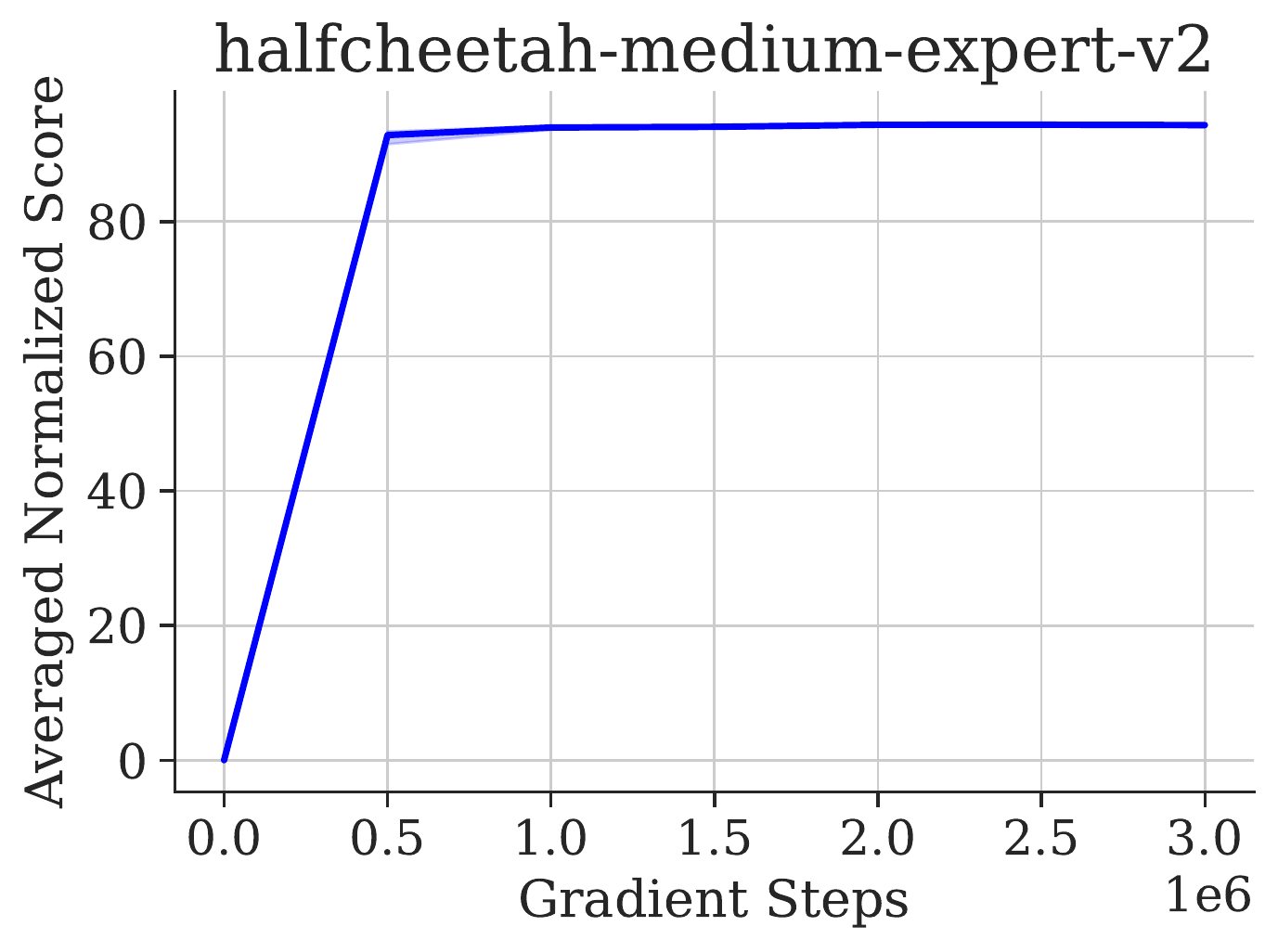} \includegraphics[height=3.3cm]{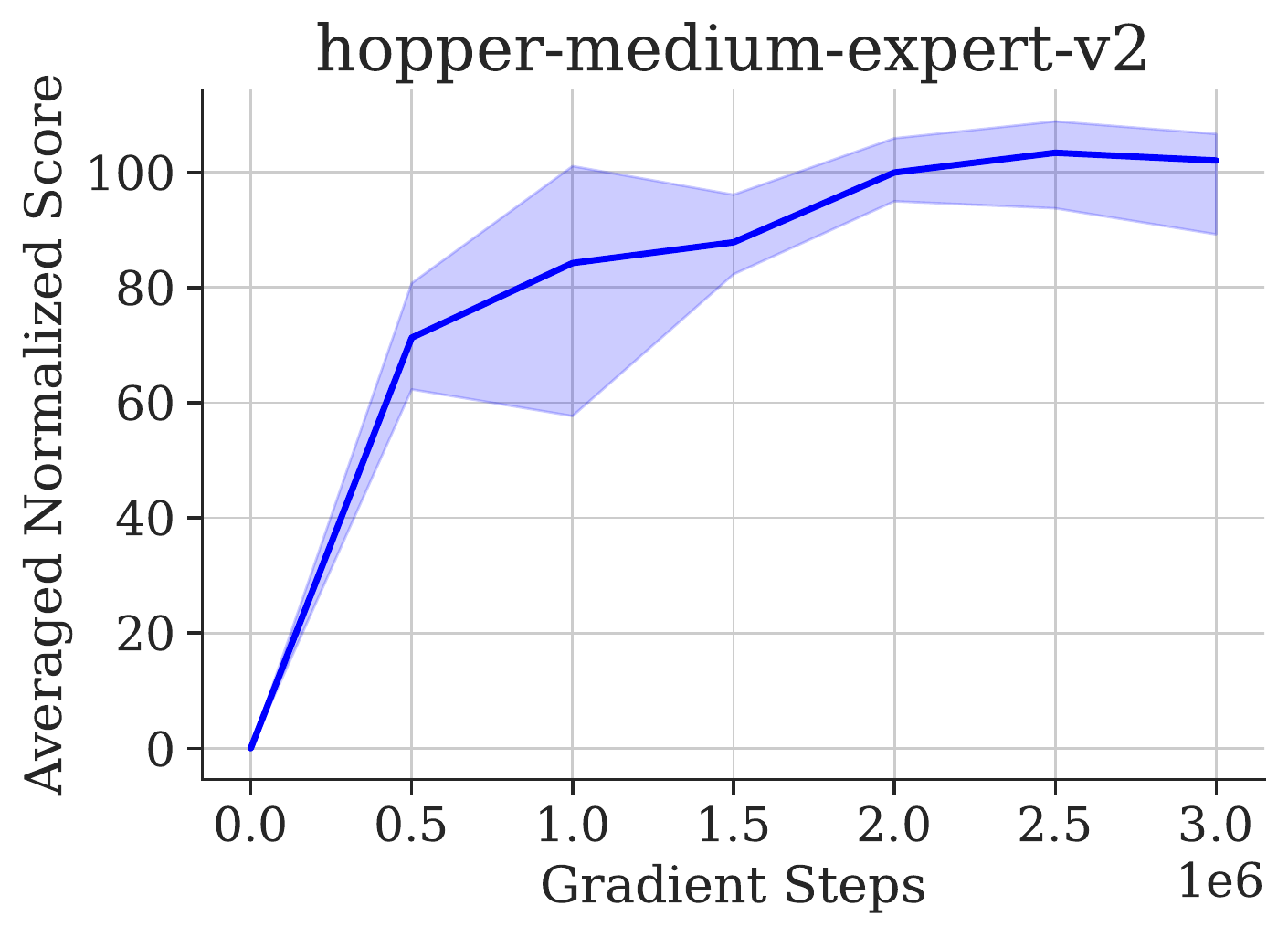}
    \includegraphics[height=3.3cm]{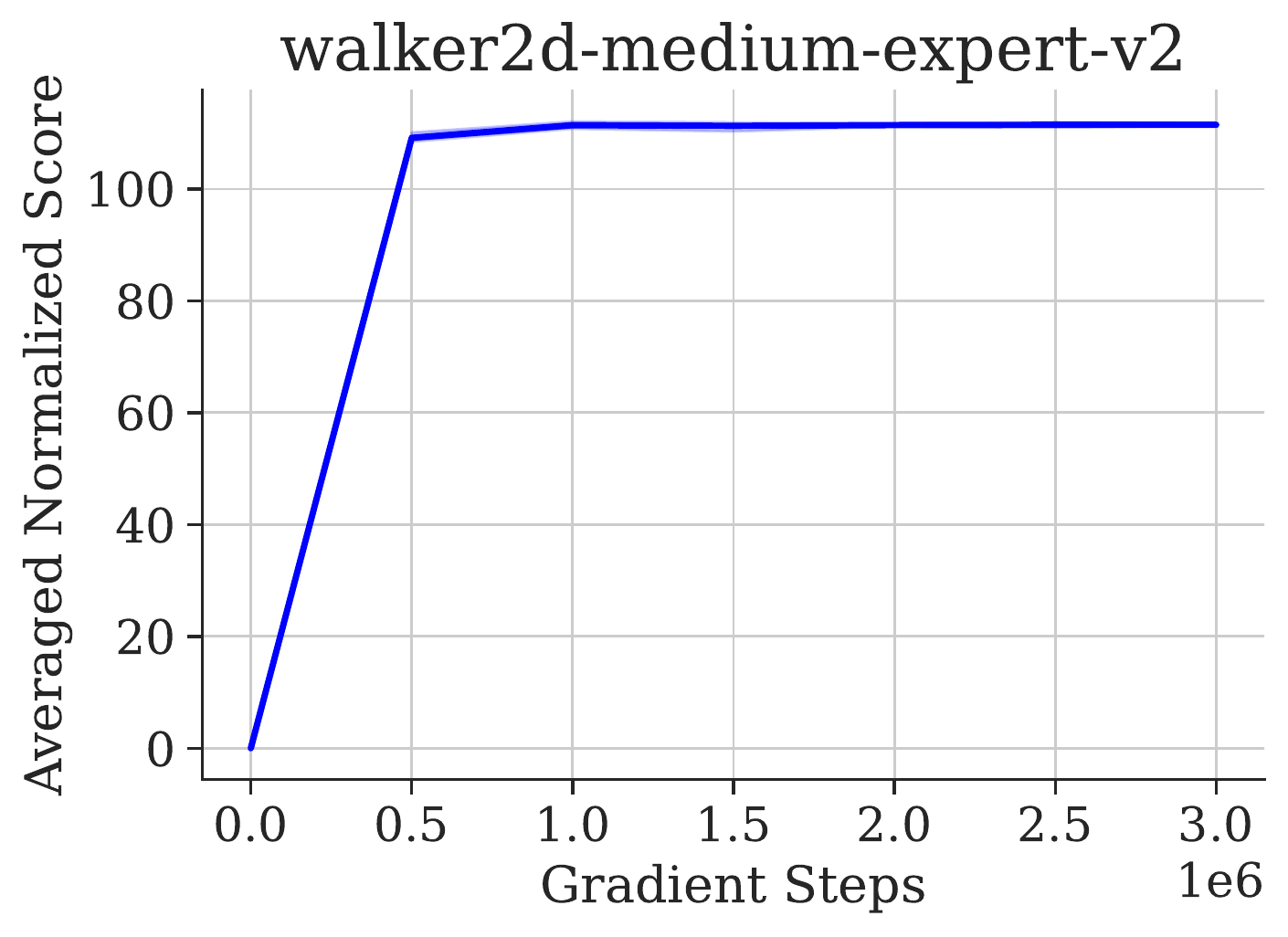}
    \caption{Training curves for locomotion tasks for one-hyperparameter variant of IDQL.}
    \label{figure:locotrainingcurves}
\end{figure}

\begin{figure}
    \centering
    \includegraphics[height=3.3cm]{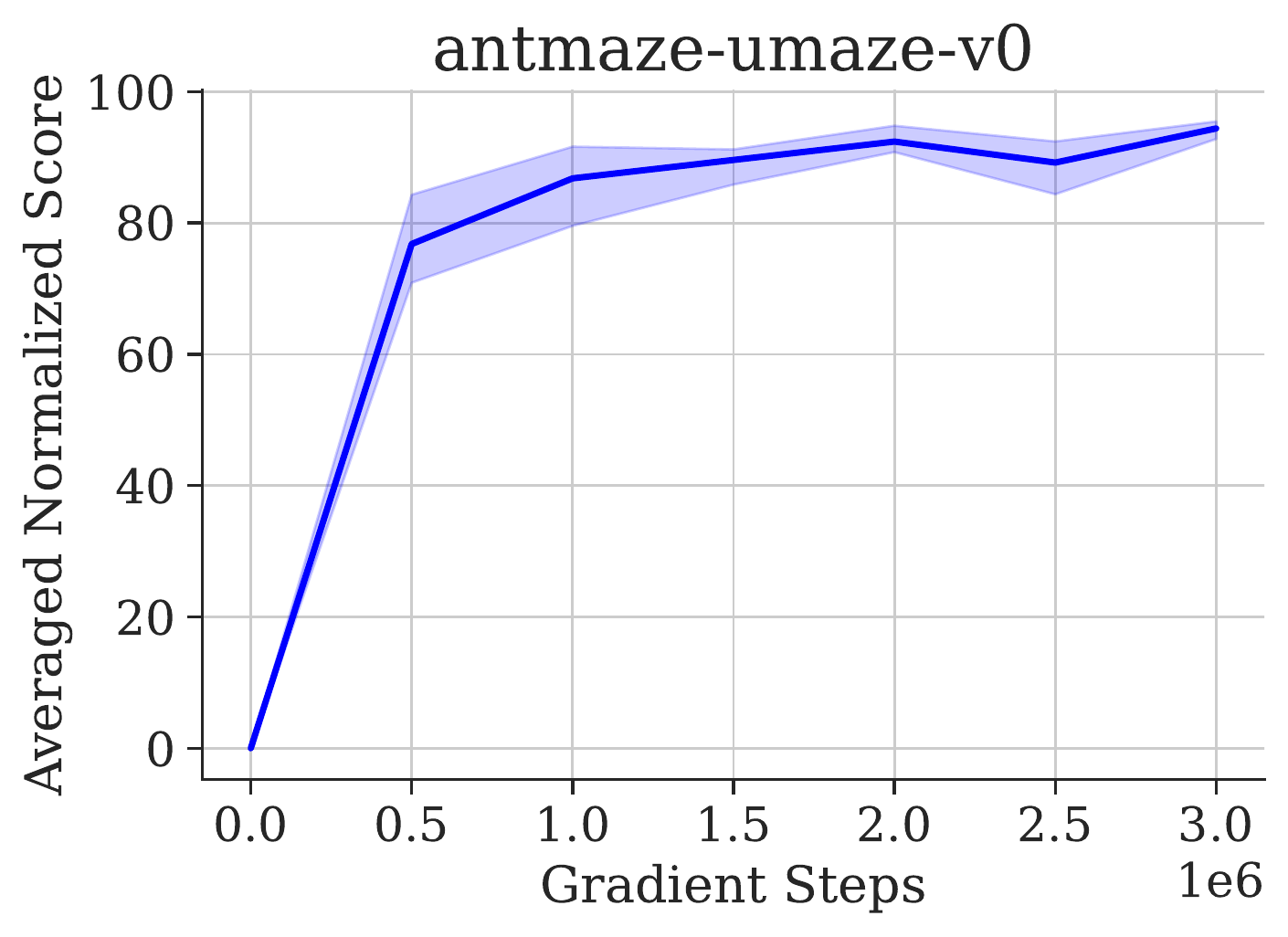} \includegraphics[height=3.3cm]{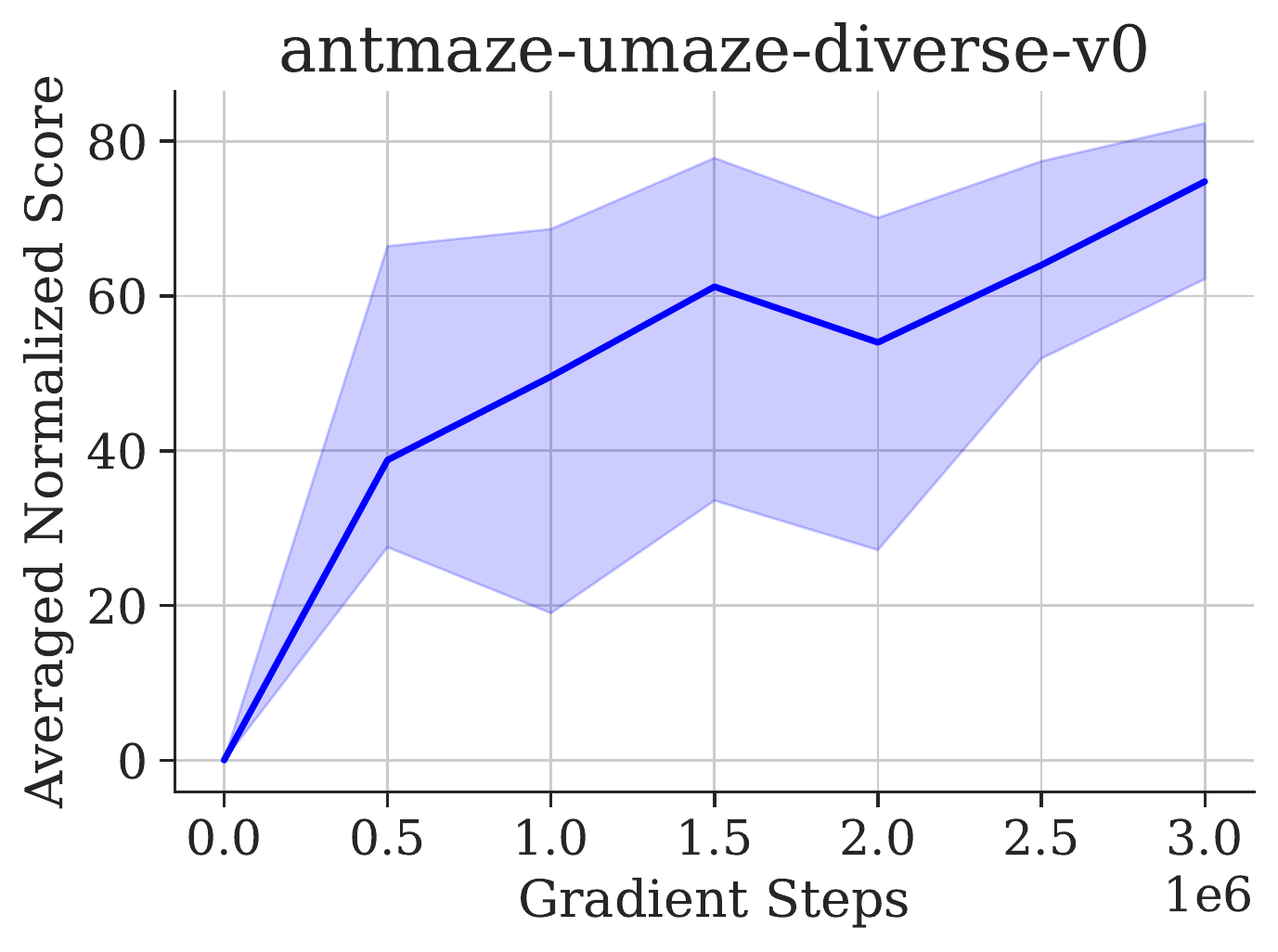}
    \includegraphics[height=3.3cm]{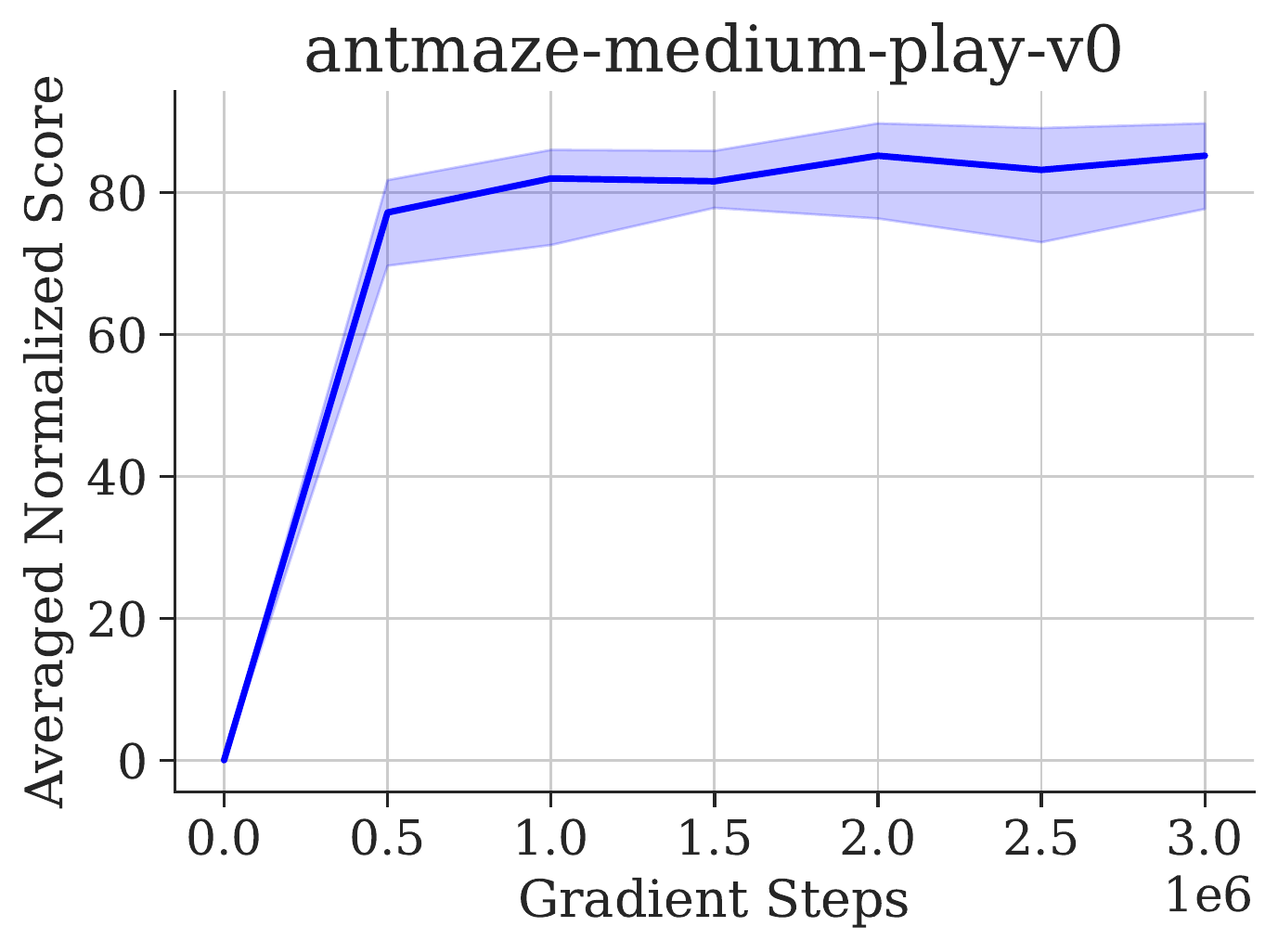} 

    \includegraphics[height=3.3cm]{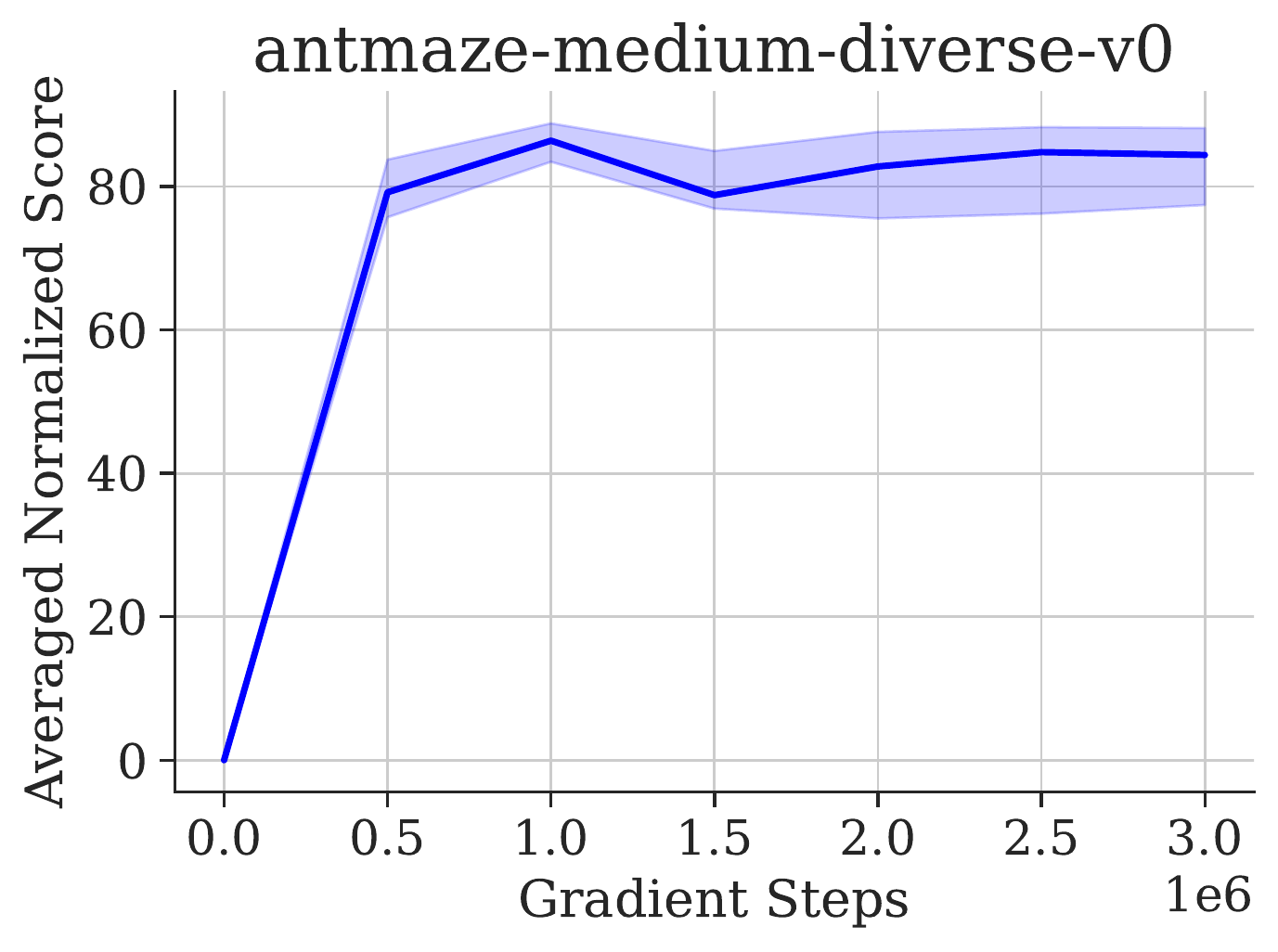} \includegraphics[height=3.3cm]{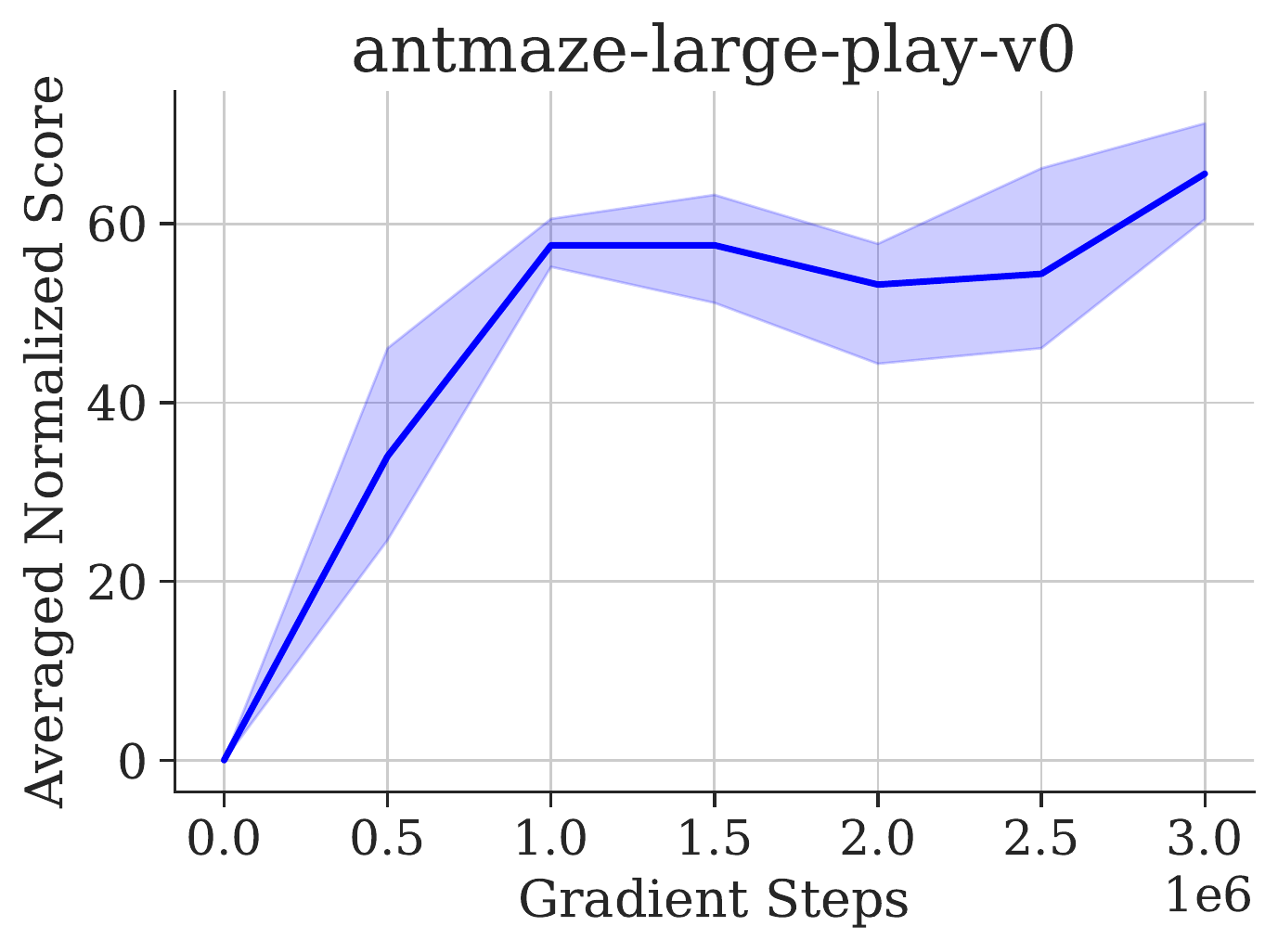}
    \includegraphics[height=3.3cm]{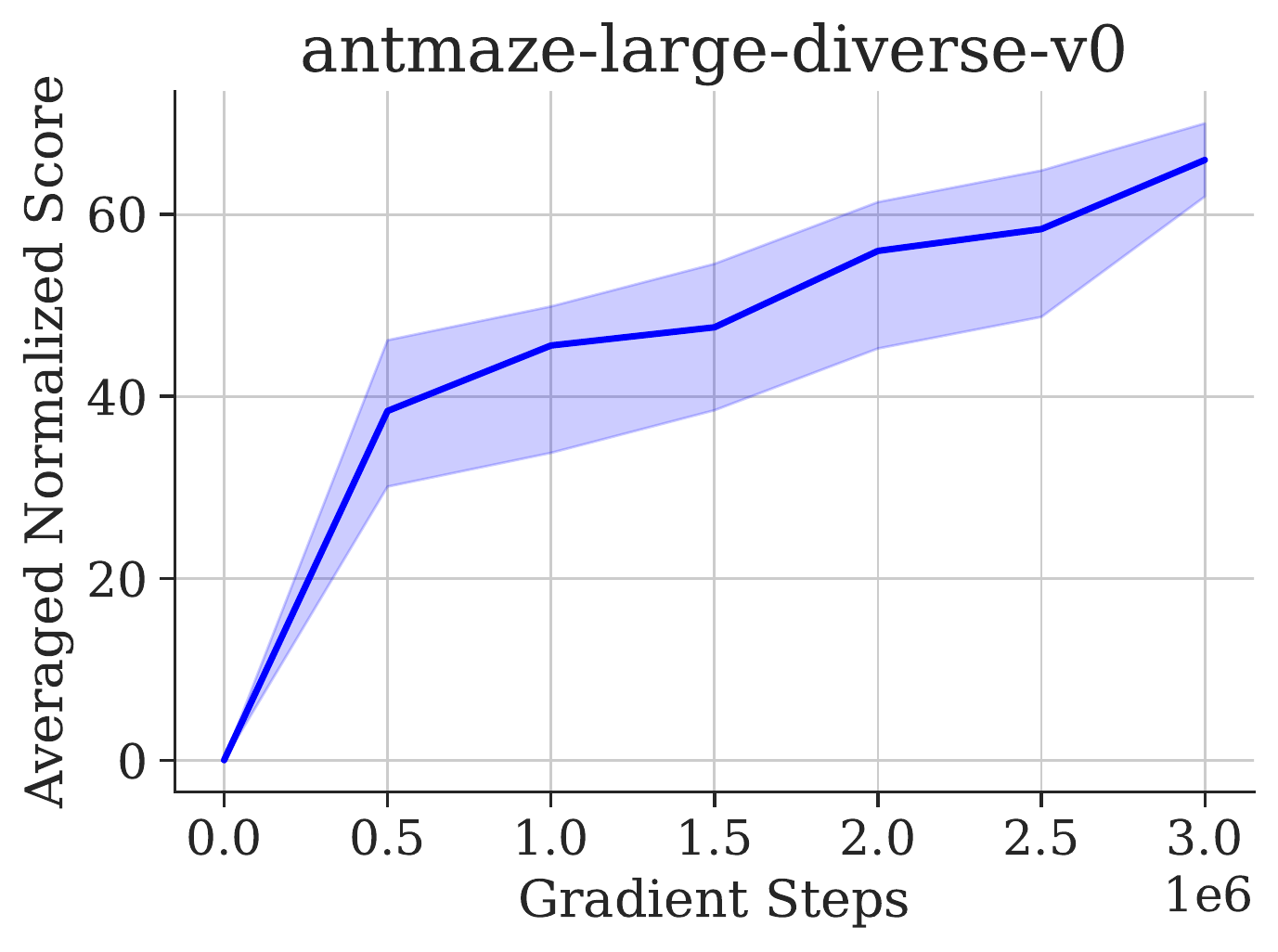} 
    
    \caption{Training curves for antmaze tasks for one-hyperparameter variant of IDQL.}
    \label{figure:antmazetrainingcurves}
\end{figure}

\newpage
\section{LN\_Resnet Architecture}
\label{archdepict}
We depict the architecture used in the LN\_Resnet described in Section~\ref{section:issuesdiff}. We use hidden dim size $256$ and $n = 3$ blocks for our practical implementation.

\usetikzlibrary{shapes, positioning, arrows.meta, decorations.pathreplacing, calc}

\begin{tikzpicture}[
    node distance=1cm and 0.75cm,
    font=\footnotesize,
    block/.style={rectangle, draw, rounded corners, text centered, minimum height=1cm, node distance=1cm and 0.5cm},
    arrow/.style={-Stealth, thick}
]

    \node[block] (input) {Input};
    \node[block, below=of input] (dense1) {Dense (hidden\_dim)};
    \node[block, below=of dense1, dashed] (mlpblockdots) {$n$ \ (MLPResnet Blocks)};
    \node[block, below=of mlpblockdots] (activation) {Activation};
    \node[block, below=of activation] (dense2) {Dense (action\_dim)};
    \node[block, below=of dense2] (output) {Output};

    \draw[arrow] (input) -- (dense1);
    \draw[arrow] (dense1) -- (mlpblockdots);
    \draw[arrow] (mlpblockdots) -- (activation);
    \draw[arrow] (activation) -- (dense2);
    \draw[arrow] (dense2) -- (output);

    \node[block, right=3.5cm of dense1] (input1) {Input};
    \node[block, below=of input1] (dropout) {Dropout};
    \node[block, below=of dropout] (layernorm) {LayerNorm};
    \node[block, below=of layernorm] (dense3) {Dense (hidden\_dim * 4)};
    \node[block, below=of dense3] (act) {Activation};
    \node[block, below=of act] (dense4) {Dense (hidden\_dim)};
    \node[block, below=of dense4] (add) {Add};
    \node[block, right=1cm of add] (output1) {Output};

    \draw[arrow] (input1) -- (dropout);
    \draw[arrow] (dropout) -- (layernorm);
    \draw[arrow] (layernorm) -- (dense3);
    \draw[arrow] (dense3) -- (act);
    \draw[arrow] (act) -- (dense4);
    \draw[arrow] (dense4) -- (add);
    \draw[arrow] (add) -- (output1);
    \draw[arrow] (input1) to[out=-140, in=140] (add);

    \node[above=0.25cm of input1] (mlpblocklabel) {MLPResNet Block};
\end{tikzpicture}

\section{Extra Related Work}
\paragraph{Other Offline RL Methods.}
\citet{brandfonbrener2021offline} use a SARSA critic objective (equivalent to expectile $\tau = 0.5$) for critic learning and then a greedy extraction via AWR.

\paragraph{Measuring Statistic for RL.}
In the family of algorithms, we present quantiles as a potential statistic to measure that induces an implicit policy distribution. Quantile statistics have also been used in RL prior for estimating distributions \citep{dabney2018distributional, kuznetsov2020controlling} over Q-functions. Though, just like IQL, we avoid querying the Q-function on out of distribution actions.

\end{document}